\documentclass{article}

\usepackage{microtype}
\usepackage{graphicx}
\usepackage{subcaption}
\usepackage{booktabs} % for professional tables

\usepackage{hyperref}

\usepackage[accepted]{icml2026}

\usepackage{amsmath}
\usepackage{amssymb}
\usepackage{mathtools}
\usepackage{amsthm}
\usepackage{dsfont}
\usepackage{thm-restate}
\usepackage{cases}
\usepackage[most]{tcolorbox}
\newtcolorbox{nbox}[1][]{
	enhanced,
	fonttitle=\scshape,
	#1
}
\usepackage[nameinlink,capitalize,noabbrev]{cleveref}
\usepackage{enumitem}
\usepackage{placeins}
\setlist{nolistsep}

\theoremstyle{plain}
\newtheorem{theorem}{Theorem}[section]
\newtheorem{proposition}[theorem]{Proposition}
\newtheorem{lemma}[theorem]{Lemma}

\theoremstyle{definition}

\newtheorem{assumption}[theorem]{Assumption}
\theoremstyle{remark}
\newtheorem{remark}[theorem]{Remark}

\usepackage[textsize=tiny]{todonotes}

\icmltitlerunning{A Direct Approach for Handling Contextual Bandits with Latent State Dynamics}

\makeatletter
\newcommand{\Iftwocolumn}[2]{\if@twocolumn #1 \else #2 \fi}
\makeatother

\newcommand{\Sdiffp}[1]{S'_{\mbox{\scriptsize \rm diff},{#1}}}
\newcommand{\Sbeliefp}[1]{S'_{\mbox{\scriptsize \rm belief},{#1}}}
\newcommand{\Setap}[1]{S'_{\mbox{\scriptsize \rm eta},{#1}}}

\newcommand{\Sdiff}[1]{S_{\mbox{\scriptsize \rm diff},{#1}}}
\newcommand{\Seta}[1]{S_{\mbox{\scriptsize \rm eta},{#1}}}

\newcommand{\hr}{\hat{r}}

\renewcommand{\leq}{\leqslant}
\renewcommand{\geq}{\geqslant}
\renewcommand{\P}{\mathbb{P}}
\renewcommand{\L}{\mathbb{L}}
\renewcommand{\phi}{\varphi}
\renewcommand{\hat}{\widehat}
\newcommand{\defeq}{\stackrel{\mbox{\scriptsize \rm def}}{=}}
\newcommand{\eqdef}{\defeq}
\newcommand{\IndEv}[1]{\mathds{1}_{#1}}
\newcommand{\Ind}[1]{\mathds{1}_{\{#1\}}}
\newcommand{\cA}{\mathcal{A}}
\newcommand{\cB}{\mathcal{B}}
\newcommand{\cU}{\mathcal{U}}

\newcommand{\cF}{\mathcal{F}}
\newcommand{\cG}{\mathcal{G}}

\newcommand{\cH}{[H]}

\newcommand{\cS}{\mathcal{S}}

\newcommand{\cX}{\mathcal{X}}

\newcommand{\E}{\mathbb{E}}

\newcommand{\M}{\mathbb{M}}

\newcommand{\R}{\mathbb{R}}
\newcommand{\btheta}{\boldsymbol{\theta}}

\newcommand{\bA}{\boldsymbol{A}}
\newcommand{\bM}{\boldsymbol{M}}
\newcommand{\bE}{\boldsymbol{E}}
\newcommand{\be}{\boldsymbol{e}}

\newcommand{\bx}{\boldsymbol{x}}

\newcommand{\bv}{\boldsymbol{v}}
\newcommand{\bu}{\boldsymbol{u}}

\newcommand{\bb}{\boldsymbol{b}}
\newcommand{\bbbnk}{\boldsymbol{b}^{\mbox{\rm \tiny bnk \!\!}}}

\newcommand{\bp}{\boldsymbol{p}}
\newcommand{\bw}{\boldsymbol{w}}

\newcommand{\by}{\boldsymbol{y}}
\newcommand{\bz}{\boldsymbol{z}}
\newcommand{\bpi}{\boldsymbol{\pi}}

\newcommand{\bone}{\boldsymbol{1}}

\newcommand{\hbb}{\hat{\boldsymbol{b}}}
\newcommand{\hbtheta}{\hat{\btheta}}
\newcommand{\hbM}{\hat{\boldsymbol{M}}}
\newcommand{\hM}{\hat{M}}

\newcommand{\hnu}{\hat{\nu}}
\newcommand{\eps}{\varepsilon}

\newcommand{\bphi}{\boldsymbol{\phi}}
\newcommand{\transp}{{\! \top}}
\newcommand{\e}{\mathrm{e}}

\renewcommand{\O}{\mathcal{O}}

\renewcommand{\d}{\,\mathrm{d}}

\newcommand{\diag}{\mathop{\mathrm{diag}}}
\newcommand{\var}{\textit}

\newcommand{\id}[1]{\boldsymbol{I}_{\! {#1}}}
\DeclareMathOperator*{\argmax}{argmax}
\DeclareMathOperator*{\argmin}{argmin}
\newcommand{\midl}{\,\middle|\,}
\newcommand{\midB}{\,\Big|\,}
\newcommand{\an}{\ \,\, \mbox{\small and} \ \,\,}
\newcommand{\Ubelief}{U_{\mbox{\rm \tiny belief}}}
\newcommand{\tO}{\widetilde{\mathcal{O}}}
\newcommand{\sX}{\arrowvert\cX\arrowvert}
\renewcommand{\epsilon}{\varepsilon}
\newcommand{\Gsum}{G_{{\mbox{\rm \tiny sum}}}}
\newcommand{\cFall}{\mathcal{F}^{\mbox{\rm \tiny all}}}
\newcommand{\cFallp}{\mathcal{F}^{'\mbox{\rm \tiny all}}}
\newcommand{\cFobs}{\mathcal{F}^{\mbox{\rm \tiny obs}}}
\newcommand{\cFobsp}{\mathcal{F}^{'\mbox{\rm \tiny obs}}}
\newcommand{\cFbench}{\mathcal{F}^{\mbox{\rm \tiny bnk}}}

\newcommand{\cplx}{\mbox{\tiny cplx}}
\newcommand{\simpl}{\mbox{\tiny simpl}} 

\begin{document}

\twocolumn[
  \icmltitle{A Direct Approach for Handling Contextual Bandits with Latent State Dynamics}



  \icmlsetsymbol{equal}{*}

  \begin{icmlauthorlist}
    \icmlauthor{Zhen Li}{equal,bnp}
    \icmlauthor{Gilles Stoltz}{equal,lmo,hec}
  \end{icmlauthorlist}

  \icmlaffiliation{bnp}{BNP Paribas Corporate and Institutional Banking, Paris, France}
  \icmlaffiliation{lmo}{Université Paris-Saclay, CNRS, Inria, Laboratoire de mathématiques d'Orsay, 91405, Orsay, France}
  \icmlaffiliation{hec}{HEC Paris, Jouy-en-Josas, France}

  \icmlcorrespondingauthor{Zhen Li}{zhen.li@bnpparibas.com}
  \icmlcorrespondingauthor{Gilles Stoltz}{gilles.stoltz@universite-paris-saclay.fr}

  \icmlkeywords{Machine Learning, ICML}

  \vskip 0.3in
]



\printAffiliationsAndNotice{}  % no special notice (required even if empty)

\begin{abstract}
We consider a linear contextual bandit model where contexts and rewards are governed by a finite hidden Markov chain.
We first revisit the simplified model by \citet{nelson2022ContextLatent},
in which rewards are linear functions of the posterior probabilities over the hidden states given the observed contexts (called beliefs), rather than functions of the hidden states themselves. This simplified model may be handled through a direct reduction to standard linear contextual bandits.
We extend the theoretical analysis of this reduction to take into account the estimation of the parameters of the hidden Markov model [HMM] in the regret bound
and to provide high-probability bounds
not depending anymore on the reward functions and only depending on the model through the estimation of the HMM parameters.
Second, and most importantly, we instead study the more natural and more complex model incorporating direct dependencies in the hidden states (on top of dependencies on the observed contexts, as is natural for contextual bandits).
Under a classic HMM forgetting condition,
the main algorithmic tool introduced to cope with the
various statistical dependencies that the reward structure introduces is to only periodically
update reward-model parameters.
\end{abstract}

\section{Introduction and Related Works}
\label{sec:intro}

We consider a linear contextual bandit model where contexts and rewards are governed by a finite hidden Markov chain.
Before we compare in detail our work to the earlier one by \citet{nelson2022ContextLatent},
we position the problem within the broader context of stochastic bandits, and more particularly,
of stochastic bandits in changing environments.

In finitely-armed stochastic bandits (introduced by \citealp{thompson1933MAB} and \citealp{robbins1952};
see also the survey monograph by \citealp{lattimore2020bandit}),
rewards are drawn i.i.d.\ from fixed but unknown distributions indexed by arms,
and the learner must perform some trade-off between exploration (to estimate
the distributions) and exploitation (to pull more often better-performing arms).
A first extension of interest is called linear contextual bandits (see \citealp{Chu2011ContextualBW} and \citealp{abbasi2011improved},
where the celebrated LinUCB strategy was introduced),
where the learner observes a context (possibly chosen adversarially), selects an action, and receives a reward modeled as a linear function of (some function of) the context and action.

A second extension of interest is when this linear contextual model depends on some latent state changing
over time.
Two modelings and approaches were considered: first, some change-point detection approaches,
relying on infrequent changes, with regret bounds typically functions of the root number of changes
(see \citealp{wu2018ContextNonStationary} or \citealp{austin2025ContextNonStationary});
second, a modeling of the latent state as following some partially known dynamic, typically a Markov chain.
Contexts are then assumed to follow a hidden Markov model [HMM]: they are drawn
independently at random given the latent state.
Much of this literature focuses on context-free reward models, where rewards depend on the actions and latent states but not on the observed contexts
(when present, the latter are used only for state inference). For example, \citet{spectral-pomdps-2016} and \citet{RegimeBandits-2021} study such a setting
and combine a LinUCB-style exploration with a spectral method for reward estimation. A common limitation due to the spectral method is to have to consider
finite reward spaces, with \citet{RegimeBandits-2021} focusing, in particular, on binary rewards.
In terms of guarantees, \citet{spectral-pomdps-2016} obtained an $\sqrt{T}$--high-probability regret bound against a memoryless policy benchmark, whereas \citet{RegimeBandits-2021} provided a $T^{2/3}$ regret bound against a stronger oracle that knows the true belief
(the posterior distribution over latent states given past and present contexts)
and the state-dependent expected rewards.

\paragraph{Motivation.}
We are interested in a general setting of
linear contextual bandits in which the latent state governs the context distribution and, jointly with context and action, determines the rewards.
If one is ready to believe that classic linear contextual bandits form a setting of practical interest, then the extension considered
here to some latent-state dynamic wishes to cover the cases where rewards are not functions of actions and contexts only.
For instance, in economic problems (see our case study in Section~\ref{sec:simu}),
underlying economic states correspond to crises or growth periods and directly influence both the contexts and the rewards.

\paragraph{Specific literature review.}
To the best of our knowledge,
only a few works have studied linear contextual bandits with a latent-state dynamic.
\citet{nelson2022ContextLatent} do so with a HMM modeling of contexts, but further assume that rewards are linear in (some function of) the action and in the belief, rather than being linear in (some function of) the action and the actual latent state. This is
a seemingly harmless but actually major simplification of the problem, as we explain throughout this article.
\citet{nelson2022ContextLatent} propose Thompson-sampling and LinUCB-style algorithms and also provide some partial elements for
a theoretical analysis (however, not discussing the estimation of HMM parameters). Finally, in their setting, rewards do not depend directly on the contexts, they do so only indirectly through beliefs.

Two recent works include \citet{hong2020ContextLatent} and \citet{galozy2025stateevobandit}: they evaluate their algorithms against a strong
benchmark that knows the realized latent states and the latent-state-dependent reward-model parameters. However, their regret guarantees
require a sublinear number of latent-state changes and degrade to linear when switches occur at a linear rate. \citet{hong2020latentbandits},
on the contrary, assumes that the latent state remains constant over time (and does not tackle the estimation of the reward-model parameters). 

\subsection*{Main Contributions; Comparison to \citet{nelson2022ContextLatent}}

The contributions of this article are twofold: first, being able to obtain sublinear high-probability regret bounds
in a complex model more challenging than existing models; second,
achieving an elementary, more direct, and more efficient treatment
of the simpler model by \citet{nelson2022ContextLatent}
as a special case of the methodological developments made to tackle the more general problem considered.

\paragraph{Contribution 1: General model.}
First, we introduce around \Cref{eq:reward-our}
a general setting of linear contextual bandits with latent-state dynamics, where expected rewards depend linearly on (functions of) the contexts and actions, as well as on unobserved states (that follow a HMM). The rewards are continuously-valued. We explain why the seemingly similar
dependence of rewards in \citet{nelson2022ContextLatent} on beliefs, rather than directly on states, is actually an important simplification
of the model. Therefore,
prior work either assumed infrequent state changes, or considered reward models depending only on states and actions but not on contexts
(with contexts only used for inference), or depending on beliefs and actions but not directly on states, or
restricted the rewards to a finite space. None provided sublinear regret bounds in our setting where rewards depends jointly on states, contexts, and actions, where states may switch arbitrarily often, and where these rewards are continuously-valued.

Our solution relies on an extremely careful analysis of the statistical dependencies induced by the reward model,
leading to a strategy proceeding in stages to carefully balance decent estimation of the reward-model parameters and of the beliefs 
and limited dependencies in the past.

\paragraph{Contribution 2: Comparison.}
We formally compare the belief-dependent linear reward model of \citet{nelson2022ContextLatent} to the state-dependent model studied here, showing in that the former reduces to linear contextual bandits whereas the latter does not.

We further show in Appendix~\ref{appendix:simplified-analysis} that, under the belief-dependent linear reward model of \citet{nelson2022ContextLatent},
extended to action-context-belief-dependent rewards,
the regret may be handled through a direct LinUCB-like analysis, up to an extra belief error term. Namely, we obtain a high-probability regret bound of order $T^{3/4}$, where the extra $T^{1/4}$ factor is shown to only come from belief estimation.
The analysis proposed in Appendix~\ref{appendix:simplified-analysis} is simple (much simpler
than in the original reference) and directly exploits the reduction to linear contextual bandits.

We also discuss in \Cref{sec:regret-nelson} how the obtained regret bound is sharper and more general
than the one by \citet{nelson2022ContextLatent}: in particular,
it holds with high probability, not only in expectation, and is reward-model free.

\subsection*{Outline}

\Cref{sec:problem-formulation} introduces the setting and the general
model of HMM-generated contextual bandits with rewards being state-dependent linear functions of
the contexts and actions, as well as the simplified version
considered by \citet{nelson2022ContextLatent}.
In particular, this section introduces the corresponding
notions of regret.

Because of technicalities discussed in detail in Appendix~\ref{sec:challenge-overcome},
including intricate dependencies on observed quantities
to the hidden states, we resort to a staged algorithm.
Namely, a staged LinUCB-like strategy is formally stated in \Cref{sec:main-algorithm};
it relies on belief-estimation subroutines, for which
reminders are provided in Appendix~\ref{sec:belief-estimation-error}.
A special case of this strategy, when stages only contain a single round,
is able to handle the simplified model of \citet{nelson2022ContextLatent}.

\Cref{sec:regret-bds} states and provides sketches of proofs
of the regret bounds of the strategies considered:
a $T^{3/4}$ regret bound in the simplified model by \citet{nelson2022ContextLatent},
with full details in Appendix~\ref{appendix:simplified-analysis},
and a $T^{7/8}$ regret bound in the general model,
with full details provided in Appendix~\ref{appendix:complex-analysis}.
The second regret bound relies on forgetting properties
of HMMs, for which reminders are provided in Appendix~\ref{appendix:general}.

Section~\ref{sec:simu} and Appendix~\ref{appendix:simulation} provide some numerical simulations,
with the mere aim to illustrate the theory developed.

\section{Settings, Notation, and Regret Definitions}
\label{sec:problem-formulation}

We first describe the considered finite-armed contextual bandit setting with latent-state dynamic
and then state two versions of the reward functions:
our own, more complex, version and the original, simplified, version
by \citet{nelson2022ContextLatent}.
After highlighting some issues arising from the statistical dependencies at stake,
we discuss two notions of regret: the same notion of pseudo-regret
as in \citet{nelson2022ContextLatent}, and regret in terms
of actual rewards.

\textbf{Notation.}
The short-hand $\bx_{s:t}$ stands for the sequence $\bx_s,\ldots,\bx_t$.
We let $[H] = \{1,\ldots,H\}$.
The vectors $(1,\ldots,1)$ with all elements equal to~$1$
are denoted, independently of the lengths, by $\bone$.
The identity matrix of size $s \times s$ is denoted by $\id{s}$.
We denote the tensor product
of two vectors $\bu = (u_h)_{h \in [H]} \in \R^H$ and $\bv \in \R^d$ by
$\bu \otimes \bv = (u_h \bv)_{h \in [H]} \in \R^{dH}$.
The Euclidean norm is denoted by $\Arrowvert \,\cdot\, \Arrowvert_2$,
and the $\ell^1$--norm by $\Arrowvert \,\cdot\, \Arrowvert_1$.
For matrices $M$, the norm considered is the Frobenius norm, i.e.,
the Euclidean norm of the coefficients all written into a column vector;
this is why we use the same notation $\Arrowvert M \Arrowvert_2$
for this matrix norm.
The norm induced on $\R^s$ by a symmetric definite positive matrix $G$ of
size $s \times s$ is defined by
\[
\forall \bu \in \R^s, \quad \Arrowvert \bu \Arrowvert_{G} = \sqrt{\bu^{\transp} G \bu}\,.
\]
For two symmetric matrices $G,G'$, we write $G \succeq G'$
when $G-G'$ is a symmetric positive semi-definite matrix.

\subsection{Latent Dynamic and Learning Protocol}

We consider a finite-armed contextual bandit problem,
with a finite action set $\cA$ and with a $m$--dimensional context space $\cX \subseteq \R^m$,
equipped with the Borel $\sigma$--algebra. At each round $t \geq 1$, the learner observes some
context $\bx_t \in \cX$, generated by a hidden Markov model [HMM].

\textbf{HMM modeling.}
More formally, there exists an underlying state $h_t \in [H]$,
where $[H]$ denotes the finite latent state space;
this space $[H]$ is known to the learner.
The first state $h_1$ is distributed according to some
initial distribution denoted by $\bpi$.
At each round $t \geq 1$,
the context $\bx_t$ is drawn independently at random given $h_t$, according
to an emission distribution over $\cX$ denoted by $\nu_{h_t}$.
The next latent state $h_{t+1}$ is then drawn according to a Markov
model indexed by the transition matrix $\bM = (M_{h,h'})_{(h,h') \in [H]^2}$,
where $M_{h,h'}$ is the probability from moving from state $h$ to state~$h'$.

The (homogeneous) HMM is thus parameterized by the initial distribution $\bpi$, the transition matrix $\bM$, and
the emission distributions $\nu = (\nu_h)_{h \in [H]}$, all unknown to the learner.

\textbf{Reward model---most complex one.}
We consider the following linear model: there exist a known transfer function $\bphi: \cA \times \cX \to \R^d$ and some
unknown parameters $\btheta^{\star}_{h} \in \R^d$, where $h \in [H]$, such that the reward $r_t(a)$ obtained with action $a \in \cA$ at round $t$
equals
\begin{equation}
\label{eq:reward-our}
r_t(a) = \bphi(a, \bx_t)^{\transp} \btheta^{\star}_{h_t} + \eta_t(a)\,,
\end{equation}
where $\eta_t(a)$ is a noise term, which we discuss below.
It is handy to assume some boundedness.

\begin{assumption}[bounds on the reward model]
\label{assumption:bound-R-linear}
There exists $C_{\btheta^{\star}} \in (0,+\infty)$ such that
for all $a \in \cA$, \,\, $\bx \in \cX$, \,\, $h \in [H]$,
\[
| \bphi(a, \bx)^{\transp} \,  \btheta^{\star}_{h} | \leq 1\,, \quad
\Vert \bphi(a,\bx) \Vert_2 \leq 1\,, \quad
\Vert \btheta^{\star}_{h} \Vert_2 \leq C_{\btheta^{\star}}\,.
\]
\end{assumption}

\textbf{Learning protocol and information available.}
At each round $t \geq 1$, the learner observes the context $\bx_t$
(but not the latent state $h_t$), picks an action $a_t \in \cA$ based
on $\bx_t$ and on the information available from past rounds, and obtains
and observes the reward $r_t(a_t)$, but not the rewards $r_t(a)$
for actions $a \ne a_t$.

The information available when picking $a_t$ consists therefore of
the past and present contexts $(\bx_\tau)_{1 \leq \tau \leq t}$
and of the past rewards $\bigl( r_\tau(a_\tau) \bigr)_{1 \leq \tau \leq t-1}$.
We denote the filtration generated by this information by
\[
\cFobs_t = \sigma \Bigl( \bigl( \bx_\tau, \, r_\tau(a_\tau) \bigr)_{\tau \leq t-1}, \, \bx_t \Bigr)
\]
(where ``obs'' stands for observed): the action $a_t$ is thus $\cFobs_t$--measurable.

\textbf{Assumptions on the noise term.}
A classic assumption in linear bandits (e.g., \citealp{abbasi2011improved}) on the noise terms $\eta_t(a)$ is
that these terms are conditionally sub-Gaussian, see Assumption~\ref{ass:SGnoise}.
It turns out that following \citet{nelson2022ContextLatent}, a milder
assumption on conditional first and second moments may be enough, see
Assumption~\ref{assumption:idio-noise}.
We will consider the second assumption for stating our main results,
though the stronger Assumption~\ref{ass:SGnoise} will be useful for discussions
and comparison to prior results.

For both assumptions, conditionings are taken with respect to all
priori random variables, whether they are observed or not: we
consider the filtration
\[
\cFall_{t} = \sigma\biggl( \Bigl(  h_\tau,\, \bx_\tau, \, \bigl( \eta_\tau(a) \bigr)_{a \in \cA} \Bigr)_{\tau \leq t-1}, \, h_t, \, \bx_t \biggr)\,.
\]

\begin{assumption}[conditionally sub-Gaussian noise]
\label{ass:SGnoise}
There exists $v_\eta$ such that for all $a \in \cA$,
\[
\E \bigl[ \e^{\lambda \eta_t(a)} \mid \cFall_t \bigr] \leq \e^{\lambda^2 v_\eta^2/2} \,.
\]
Note that this entails that $\E \bigl[\eta_t(a) \mid \cFall_t \bigr] = 0$.
\end{assumption}

\begin{assumption}[Bounded conditional second-order moment]
\label{assumption:idio-noise}
There exists $C_{\eta}$ such that for all $a \in \cA$,
\[
\E \bigl[\eta_t(a) \mid \cFall_t \bigr] = 0
\quad \mbox{and} \quad
\E \bigl[\eta_t(a)^2 \mid \cFall_t \bigr] \leq C_{\eta} \,.
\]
\end{assumption}

\subsection{The Simplified Model by \citet{nelson2022ContextLatent}}
\label{sec:simplified}

Consider the beliefs (the posterior probabilities over the hidden states given the observed contexts)
\[
\bb_t : h \in [H] \longmapsto \bb_t(h) = \P(h_t = h \mid \bx_{1:t})\,.
\]
\citet{nelson2022ContextLatent} consider the same latent dynamics and learning protocol
as above but rather study the following reward model:
there exist scalars $(\theta^{\star}_{h,a})_{h \in [H], a \in \cA}$ such that
\begin{equation}
\label{eq:nelson-model-orig}
r'_t(a) = \sum_{h \in [H]} \bb_t(h) \, \theta^{\star}_{h,a} + \eta'_t(a)\,,
\end{equation}
where the noise terms $\eta'_t(a)$ satisfy
Assumption~\ref{assumption:idio-noise}.

We rather consider an immediate generalization
where expected rewards can depend directly also on contexts and where
general transfer functions are considered as in~\eqref{eq:reward-our}:
\begin{equation}
\label{eq:nelson-model-gener}
r'_t(a) = \sum_{h \in [H]} \bb_t(h) \, \bphi(a, \bx_t)^{\transp} \btheta^{\star}_h + \eta'_t(a)\,,
\end{equation}
The original model~\eqref{eq:nelson-model-orig} corresponds to the special case
where $\btheta^{\star}_h = (\theta^{\star}_{h,a})_{a \in \cA}$
and $\bphi(a, \bx_t) \in \{0,1\}^{\cA}$ with the $a$--th component equal to~1
and all other components being null.

The difference between the model~\eqref{eq:reward-our} we study in
this article and
the immediate generalization~\eqref{eq:nelson-model-orig} of
the model by \citet{nelson2022ContextLatent} lies in replacing
$\bphi(a, \bx_t)^{\transp} \btheta^{\star}_{h_t}$ by
\[
\smash{\E\bigl[ \bphi(a, \bx_t)^{\transp} \btheta^{\star}_{h_t} \mid \bx_{1:t} \bigr]
= \bphi(a, \bx_t)^{\transp} \sum_{h \in [H]} \bb_t(h) \, \btheta^{\star}_h \, }.
\]

This substitution looks harmless at first sight but has important consequences:
the problem can be reduced to contextual bandits, as exploited
by \citet{nelson2022ContextLatent}. Without this substitution, and when keeping
the direct dependencies on the hidden states $h_t$, no such reduction holds
and an improved analysis is required. We now detail these claims,
as the technical discussions that follow will clarify
why and how we consider two notions of regret in \Cref{sec:regret-def}.

\textbf{Reduction of model~\eqref{eq:nelson-model-gener} to linear contextual bandits.}
Introduce
\[
\smash{\cFobsp_t = \sigma \Bigl( \bigl( \bx_\tau, \, r'_\tau(a_\tau) \bigr)_{\tau \leq t-1}, \, \bx_t \Bigr)}\,.
\]
The action $a_t$ picked in the model~\eqref{eq:nelson-model-gener} is $\cFobsp_t$--measurable.
Also, the assumptions on the noise entail, by the tower rule,
that for all $a \in \cA$,
\[
\E\bigl[ \eta'_t(a) \mid \cFobsp_t \bigr] = 0\,,
\quad \mbox{thus} \quad \E\bigl[ \eta'_t(a_t) \mid \cFobsp_t \bigr] = 0\,.
\]
Because of the specific form of the reward model~\eqref{eq:nelson-model-gener},
these equalities translate into:
\begin{align*}
\forall a \in \cA, \quad \E\bigl[ r'_t(a) \mid \cFobsp_t \bigr] & = \bphi(a, \bx_t)^{\transp} \sum_{h \in [H]} \bb_t(h) \, \btheta^{\star}_h \\
\mbox{and} \qquad \smash{\E\bigl[ r'_t(a_t) \mid \cFobsp_t \bigr]} & \smash{= \bphi(a_t, \bx_t)^{\transp} \sum_{h \in [H]} \bb_t(h) \, \btheta^{\star}_h\,}\,.
\end{align*}
The vectors $\bigl( \bphi(a, \bx_t) \bb_t(h) \bigr)_{h \in [H]}$ act as contexts in linear contextual bandits.
\citet{nelson2022ContextLatent} only provide an analysis when these contexts are known (because the HMM parameters
are assumed to be known in their theoretical analysis) but with the techniques introduced in this article, these contexts may
be estimated and the reduction to linear bandits can be saved.

See Appendix~\ref{appendix:simplified-analysis} for details and a regret analysis taking care of estimation errors: under \Cref{ass:SGnoise}, we obtain a high probability regret bound of $\tO(T^{3/4})$, where the extra $\tO(T^{1/4})$ term is due to belief estimation (recovering $\tO(T^{1/2})$ if the belief were known).

\textbf{No such reduction for model~\eqref{eq:reward-our}.}
There is no such reduction in the reward model~\eqref{eq:reward-our} primarily
studied in this article, where reward depend directly on the hidden states $h_t$.
For this model, for all $a \in \cA$,
\begin{multline}
\label{eq:core-difficulty-ante}
\E\bigl[ r_t(a) \mid \cFobs_t \bigr] \\ = \bphi(a, \bx_t)^{\transp} \sum_{h \in [H]} \P(h_t = h \mid \cFobs_t) \, \btheta^{\star}_h\,,
\end{multline}
but $\P(h_t = h \mid \cFobs_t)$ is a complex quantity, depending on the strategy implemented (as the actions played
are $\cFobs_t$--measurable),
that cannot be easily estimated,
and that is in general different from the belief $\bb_t(h)$.
Appendix~\ref{sec:challenge-overcome} further details the issues that arise.

\subsection{Two Notions of Regret}
\label{sec:regret-def}

\paragraph{Pseudo-regret based on beliefs.}
The literature of bandits with latent space dynamics
considers benchmarks involving posterior probabilities over the states of the form
\[
\bbbnk_t : h \in [H] \longmapsto \bbbnk_t(h) = \P(h_t = h \mid \cFbench_t)
\]
for filtrations $\sigma(\bx_{1:t}) \subseteq \cFbench_t \subseteq \cFobs_t$ discussed below;
these posterior probabilities rely on the knowledge of the HMM parameters.
The associated benchmarks are of the form of sums of
\[
\max_{a \in \cA} \! \sum_{h \in [H]} \!\! \E\bigl[ r_t(a) \mid \cFbench_t \bigr]
= \max_{a \in \cA} \! \sum_{h \in [H]} \!\! \bbbnk_t(h) \, \bphi(a, \bx_t)^{\transp} \btheta^{\star}_{h}\,,
\]
where the equality holds by the tower rule.

\citet{RegimeBandits-2021} consider a model with $\{0,1\}$--valued rewards
and (only) because of that, may take $\cFbench_t = \cFobs_t$. This choice
however is somewhat unnatural, as the benchmark is not intrinsic and depends
on the strategy used.

\citet{nelson2022ContextLatent} consider a more intrinsic choice, which also does not constrain
rewards to take finitely many values: $\cFbench_t = \sigma(\bx_{1:t})$,
i.e., the posterior probabilities equal the beliefs $\bb_t$ and are based only on contexts. Thus, no additional information
from the complex dependencies of rewards on the hidden states is exploited.
More formally, they consider associated pseudo-regret defined by
\begin{multline}
\label{eq:def-regret}
R_T = \sum_{t=1}^{T} \max_{a \in \cA} \sum_{h \in [H]} \bb_t(h) \, \bphi(a, \bx_t)^{\transp} \btheta^{\star}_{h} - \\
\sum_{t=1}^{T} \sum_{h \in [H]} \bb_t(h) \, \bphi(a_t, \bx_t)^{\transp} \btheta^{\star}_{h} \, .
\end{multline}

\paragraph{Regret based on actual rewards.}
The first sum in the definition~\eqref{eq:def-regret} admits some natural interpretation
as the sum of actual rewards achieved, up to some high-probability $\sqrt{T}$--deviation terms,
by an oracle that would know the HMM parameters and the reward-model parameters $\btheta^{\star}_{h}$, and would pick its actions based on the contexts observed.
Indeed, for all $a \in \cA$,
\[
\E\bigl[ r_t(a) \mid \bx_{1:t} \bigr] = \bphi(a, \bx_t)^{\transp} \sum_{h \in [H]} \bb_t(h) \, \btheta^{\star}_h\,.
\]
However, it is actually difficult to interpret the second sum in~\eqref{eq:def-regret}, because in general,
it is difficult to relate
\[
\bphi(a_t, \bx_t)^{\transp} \sum_{h \in [H]} \bb_t(h) \, \btheta^{\star}_h
\]
to conditional expectations like
$\E\bigl[ r_t(a_t) \mid \bx_{1:t},\,a_t \bigr]$
or $\E\bigl[ r_t(a_t) \mid \bx_{1:t} \bigr]$.
This is due, exactly as in \Cref{eq:core-difficulty-ante},
to the complex dependencies between the actions taken and the hidden states, through the rewards observed.
See Appendix~\ref{sec:challenge-overcome} for details.

However, Appendix~\ref{app:true-regret}
proves, by adapting the proof of \Cref{theorem:total-regret-bound} (and in particular,
the one of~\Cref{lm:true-vs-belief}), that
for the strategy considered in Box~A (which proceeds in stages),
the second sum in~\Cref{eq:def-regret} is close
to the sum $\sum_{t \in [T]} r_t(a_t)$
of actual rewards, with high-probability
and up to an additive term of order $T^{5/8}$ up to poly-logarithmic factors.
Put differently, the regret bounds on $R_T$ stated later in this article
also yield bounds on the actual regret
\begin{multline*}
R_T^{\mbox{\rm \tiny actual}} = \sum_{t=1}^{T} r_t(a^{\star}_t) - r_t(a_t)\,, \\
\mbox{where} \qquad a^{\star}_t  \in \argmax_{a \in \cA} \sum_{h \in [H]} \bb_t(h) \, \bphi(a, \bx_t)^{\transp} \btheta^{\star}_{h}\,.
\end{multline*}

\section{Algorithm(s): \\ ~~~~Staged LinUCB on Estimated Beliefs}
\label{sec:main-algorithm}

In this section, we both present our main algorithm (Box~A)
addressing the most complex reward model of \Cref{eq:reward-our},
as well as a special case thereof addressing
the simplified model of \Cref{eq:nelson-model-gener}
but in a more generic way than in \citet{nelson2022ContextLatent},
as we do not fix a specific belief estimation subroutine
(online expectation-maximization in their case)
but consider any efficient such subroutine (see \Cref{assumption:bound-belief}).
We discuss these subroutines first (in \Cref{sec:belief-estimation})
and then state the strategies (in \Cref{sec:stamt-strat}).

\subsection{Belief Estimation Subroutines}
\label{sec:belief-estimation}

As justified in Appendix~\ref{sec:challenge-overcome}
and as in \citet{nelson2022ContextLatent},
due to the complex dependencies between rewards and hidden states,
we estimate beliefs only based on contexts.
We therefore define a belief estimation subroutine $\cB$ as a sequence of functions
where the $t$--th function
\[
\smash{\bx_1, \ldots, \bx_t \longmapsto \hbb_t = \bigl(\hbb_t(h)\bigr)_{h \in [H]}}
\]
associates
with the contexts $\bx_1, \ldots, \bx_t$ a probability distribution $\hbb_t$ over the hidden state spaces $[H]$.

We provide no methodological development on the estimation of beliefs
and instead resort to known results, up to one addition.
The estimation of HMM parameters, and thus of beliefs, requires knowing the
number $H$ of hidden states but only provides estimates that are correct
up to permutations of the hidden states (as the latter have no specific ordering).
This is why estimation guarantees are only formulated in norms.
However, the strategy considered (see Box~A) must keep track of
specific states, as it will maintain estimators for each parameter $\btheta^\star$.
That the labeling of hidden states is consistent throughout time
will be vital. We achieve this through an additional alignment
step. See details on this issue and on the solution in
Appendix~\ref{sec:belief-estimation-error}

To make our arguments generic,
we consider the following assumption on the belief-estimation subroutine
$\cB$; examples and pointers below
explain why it is a reasonable assumption (and to which large classes of hidden
Markov chains it applies).

\begin{restatable}[belief estimation error]{assumption}{boundbelief}
	\label{assumption:bound-belief}
	The belief estimation procedure $\cB$ is such that
    for all hidden Markov chains $(\bpi,\bM,\nu)$ in a wide class,
    there exist
    \begin{itemize}
    \item a constant $T_{\cB,\bM,\nu}$ not necessarily known to the learner,
    \item a fully known belief error function
	$\Ubelief$ on $\{1,2,\ldots\} \times (0,1)$,
    where $\Ubelief(t, \delta)$ depends logarithmically on $\delta$
	and, up to poly-log factors, for each $\delta \in (0,1)$,
	\[
	\sum_{t \in [T]} \Ubelief(t, \delta) = \tO\bigl(T^{1/2}\bigr)\,,
	\]
    \end{itemize}
    such that
    for all $\delta \in (0,1)$, with probability at least $1-\delta$,
    the following statements hold for all $t \geq T_{\cB,\bM,\nu}\bigl(1 + \ln(1/\delta) \bigr)$:
    \begin{itemize}
    \item first, the labeling of hidden states is consistent over the rounds considered;
    \item second, $\bigl\Arrowvert \hbb_t - \bb_t \bigr\Arrowvert_1 \leq \Ubelief(t, \delta)$.
    \end{itemize}
\end{restatable}

In the HMM literature, belief estimation is more commonly referred to as the estimation of the filtering distributions. The hidden state space is typically assumed
to be finite, while the context space may be finite or continuous.
For the sake of exposition, and since the belief estimation is used here only as an independent subroutine, we will mostly focus on the case of a finite context set.

\paragraph{Example 1: Spectral method for finite context sets $\cX$.}
The so-called spectral method was proposed by \citet{spectral-2012} and further developed by \citet{anandkumar-2012-spectral-hmm} and \citet{anandkumar-2014-tensor-mixture}. It provides estimates of the HMM transition matrix $\bM$ and of the emission distributions $\nu_h$.
\citet{de2017consistent} show how the performance of these estimates, combined with the Bayes' update rule, transfers
into a performance bound on estimated beliefs of the form of \Cref{assumption:bound-belief}.
This is formally stated in \Cref{lm:beliefestimationerror} below.

Assume that the context set $\cX$ is finite, so that each emission distribution $\nu_h$
on $\cX$ may be seen as a column vector, and let $\bE$ denote the emission matrix, indexed by $\cX \times [H]$, obtained by concatenating the vectors $\nu_h$ as $h \in [H]$.
We assume below that $\bE$ has full column rank: this imposes, in particular,
that $H$ is smaller than the cardinality $X$ of $\cX$.

Recall that $\sigma$ is a singular value of $\bE$ if $\sigma^2$ is an eigenvalue of the square matrix $\bE^{\transp}\bE$.

\begin{restatable}{assumption}{regularityhmm}
	\label{assumption:regularity-hmm}
	The context set $\cX$ is finite, with cardinality denoted by $\sX = X$. \smallskip \\
	The emission matrix $\bE$ has full column rank with smallest singular value $\sigma_{\min}(\bE) > 0$,
	and its smallest element
	satisfies
	\[
    e_{\nu, \min} \eqdef \min_{h \in [H]} \min_{x \in \cX} \nu_h(x) > 0\,.
    \]
	The transition matrix $\bM$ is invertible, with smallest eigenvalue denoted by $\sigma_{\min}(\bM) > 0$, and the smallest element of $\bM$
	is positive: $\epsilon_{\bM} = \displaystyle{\min_{h,h'} M_{h,h'} > 0}$. \smallskip \\
    Finally, the initial distribution $\bpi$ is the (unique)
    stationary distribution of $\bM$.
\end{restatable}

Appendix~\ref{sec:belief-estimation-error} reviews the literature necessary
to obtain the guarantee stated in \Cref{lm:beliefestimationerror} (whose proof may be found in Appendices~\ref{sec:proof-beliefestimationerror} and~\ref{sec:align}), and also provides
more details on the underlying belief estimation procedure
(namely, the spectral method combined with a Bayes' update rule).

\begin{restatable}{lemma}{beliefestimationerror}
	\label{lm:beliefestimationerror}
	\Cref{assumption:bound-belief} is satisfied for all hidden Markov chains of \Cref{assumption:regularity-hmm},
	for the spectral method (followed by an alignment step) combined with the Bayes' update rule,
    \Iftwocolumn{with the known belief error function $\Ubelief(t, \delta) =$
	\[
	 \ln(t) \Biggl( H \sqrt{X} \sqrt{\frac{2 \ln\bigl(6 X t(t+1)/\delta\bigr)}{t}}
    + \e^{-\sqrt{t-1}} \Biggr)
	\]}{with the known belief error function
    \[
	\Ubelief(t, \delta) = \ln(t) \Biggl( H \sqrt{X} \sqrt{\frac{2 \ln\bigl(6 X t(t+1)/\delta\bigr)}{t}}
    + \e^{-\sqrt{t-1}} \Biggr)
	\]}
	and the unknown threshold $T_{\cB,\bM,\nu}$ whose closed-form expression is
    provided in \Cref{eq:closedformT}.
\end{restatable}

\paragraph{Example 2: More general context sets $\cX$.}
\citet{de2017consistent} extended the spectral method and its analysis
to the case of continuously-valued contexts, under an assumption that contexts are continuously projectable into a finite-dimensional feature space via basis functions such as splines, trigonometric functions, or wavelets.

\subsection{LinUCB Strategies on Estimated Beliefs}
\label{sec:stamt-strat}

For any probability distribution $\bb$ over $[H]$,
we use the short-hand notation, for all $a \in \cA$ and $\bx \in \cX$,
\begin{align*}
	\bb \otimes \bphi(a,\bx) = \bigl( \bb(h) \bphi(a, \bx) \bigr)_{h \in [H]} & \in \R^{dH}\,, \\
\mbox{so that} \qquad
		\sum_{h \in \cH} \bb(h) \, \bphi(a, \bx)^{\transp} \btheta^{\star}_{h} & =
			\bigl( \bb \otimes \bphi(a, \bx) \bigr)^{\transp} \btheta^{\star}\,.
\end{align*}
Note that by \Cref{assumption:bound-R-linear}, which considers the Euclidean norm in $\R^d$,
and the fact that $\bb$ is a probability distribution, we also have, for the Euclidean norm in $\R^{dH}$,
\begin{equation}
	\label{eq:bbbphi-norm}
	\forall a \in \cA, \ \ \forall \bx \in \cX, \qquad
	\bigl\Arrowvert \bb \otimes \bphi(a, \bx) \bigr\Arrowvert_2 \leq 1\,.
\end{equation}
We estimate the stacked vector
$\btheta^{\star} = (\btheta^{\star}_h)_{h \in [H]} \in \R^{dH}$ through a LinUCB-style (\citealp{abbasi2011improved}) approach: let $\lambda > 0$
and introduce, for $t \geq 1$, the (symmetric definite positive thus invertible) Gram matrix
\[
G_t \defeq \smash{\sum_{\tau=1}^{t}} \bigl(\hbb_\tau \otimes \bphi(a_\tau, \bx_\tau)\bigr)
\bigl(\hbb_\tau \otimes \bphi(a_\tau, \bx_\tau)\bigr)^{\!\transp} + \lambda \id{dH}\,,
\]
based on which we define the estimates
\begin{equation}
	\label{eq:estimate-theta}
	\hbtheta_t  = G_t^{-1} \smash{\sum_{\tau=1}^{t}} \bigl(\hbb_\tau \otimes \bphi(a_\tau, \bx_\tau)\bigr) \, r_\tau(a_\tau)\,.
\end{equation}

\paragraph{Strategy in the most complex reward model~\eqref{eq:reward-our}.}
As justified in Appendix~\ref{sec:challenge-overcome}, we consider
a strategy that works in stages of lengths $\ell \geq 1$ and only performs the estimations~\eqref{eq:estimate-theta}
periodically, at rounds $t$ multiple of $\ell$. This defines stages, where stage $s \geq 1$ gather rounds
$(s-1)\ell+1$ to $s\ell$. Within a stage, rewards are estimated by estimates of their conditional means
$\bphi(a, \bx_t)^{\transp} \btheta^{\star}_{h_t}$, of the form
\[
\sum_{h \in \cH} \hbb_{t}(h) \bphi(a, \bx_t)^{\transp} \hbtheta_{(s-1) \ell, h} + \epsilon_{t,a}\,,
\]
where the $\epsilon_{t,a}$ are confidence bonuses.
The strategy considered is optimistic and plays arms that maximize the upper confidence
estimates defined above.

The resulting strategy, called \emph{staged LinUCB on estimated beliefs}, is formally stated in
Box~A.
\begin{figure}[!ht]
	\begin{nbox}[title={\small Box A: Staged LinUCB on estimated beliefs}]
		\label{nbox:main-algorithm}
		\textbf{Known parameters:} finite action set $\cA$; context set $\cX$; transfer function $\bphi: \cA \times \cX \to \R^d$;
        finite state space $\cH$

		\textbf{Unknown parameters:}
		HMM parameters, given by a transition matrix $\bM = (M_{h,h'})_{(h,h') \in [H]}$
		and emission distributions $(\nu_h)_{h \in [H]}$ over $\cX$;
        reward parameters $\btheta^{\star}_{h} \in \R^d$, for $h \in [H]$ \smallskip
		
		\textbf{Inputs:} risk $\delta \in (0,1)$; belief estimation subroutine~$\cB$;
		stage length $\ell \geq 1$;
        regularization parameter $\lambda > 0$; closed-form
		expression for the confidence bonuses $\eps_{t,a}$, possibly depending on $\delta$, $\lambda$, and $\ell$ \smallskip
		
		\textbf{Initialization:} set $\hbtheta_0 = (1/\lambda) \, \bone \in \R^{dH}$ \medskip

		\textbf{For} stages $s = 1, 2, \dots$\textbf{:} \smallskip \newline
		\textbf{For} rounds $t = (s - 1) \ell + 1, \ldots,  s \ell$, \textbf{the learner:}
		\begin{enumerate}
			\item Observes the context $\bx_{t}$, drawn independently by the environment from $\nu_{h_{t}}$;
			\item Obtains the belief estimate $\hbb_{t}$ by feeding $\bx_1,\ldots,\bx_{t}$ to the subroutine $\cB$;
			\item Computes estimated rewards: for all $a \in \cA$,
			\[
			\hat{r}_{t}(a) = \smash{\sum_{h \in \cH} \hbb_{t}(h) \bphi(a, \bx_t)^{\transp} \hbtheta_{(s-1) \ell, h}}\,; \smallskip
			\]
			\item Picks an action $\displaystyle{a_{t} \in \argmax_{a \in \cA} \bigl\{ \hat{r}_{t}(a) + \eps_{t, a}} \bigr\}$\,; \smallskip
			\item \label{item:5} Obtains and observes the reward \vspace{-.25cm} \\
			\[
			\smash{r_t(a_t) = \bphi(a_t, \bx_t)^{\transp} \btheta^{\star}_{h_t} + \eta_t(a_t)\,}; \vspace{-.2cm}
			\]
		\end{enumerate}
		\quad \textbf{end}

		\quad Computes $\hbtheta_{s\ell}$ as in \Cref{eq:estimate-theta}.
		
		\textbf{end}
	\end{nbox}
\end{figure}

\paragraph{Strategy in the simplified model~\eqref{eq:nelson-model-gener}.}
Our generic version of the strategy by \citet{nelson2022ContextLatent}
is given by the Box-A strategy run with $\ell = 1$, i.e.,
updating the LinUCB estimates of $\btheta^\star$ at each round,
and with the reward-obtention
step (numbered~\ref{item:5} in Box~A) of course replaced by \Cref{eq:nelson-model-gener}.
For the sake of clarity, we state separately this strategy
in Box B of Appendix~\ref{appendix:simplified-analysis}.

\section{Regret Bounds}
\label{sec:regret-bds}

In this section, we present regret analyses
both for the main strategy of Box~A
addressing the most complex reward model~\eqref{eq:reward-our},
as well as its special case addressing
the simplified model of \Cref{eq:nelson-model-gener}
(see the paragraph above).
We start with the latter as it can be performed with no
additional assumption.

\subsection{Regret Bound for the Simplified Model~\eqref{eq:nelson-model-gener}}
\label{sec:regret-nelson}

In Appendix~\ref{appendix:simplified-analysis},
we state (Theorem~\ref{theorem:total-regret-bound-simplified})
and show that under \Cref{ass:SGnoise} (sub-Gaussian noise),
\Cref{assumption:bound-belief} (on the belief estimation subroutine),
and \Cref{assumption:bound-R-linear} (boundedness of rewards),
with proper inputs,
the strategy in the simplified model~\eqref{eq:nelson-model-gener}
described above satisfies, with probability at least $1-\delta$,
up to poly-log factors,
\[
R_T = \tO \bigl(T^{3/4} \bigr)\,,
\]
where a closed-form expression of the regret bound may be found in
the proof, see \Cref{eq:regret-bound-final-simplified}.

\paragraph{Comparison to \citet[Theorem~2]{nelson2022ContextLatent}.}
First, \citet[Theorem~2]{nelson2022ContextLatent} do not take
into account the belief estimation error into account in
their regret bound, which, in addition, only holds in expectation;
they obtain a $\sqrt{T}$ rate and the proof of
Theorem~\ref{theorem:total-regret-bound-simplified} shows that the worsened
rate $T^{3/4}$ is only due to the belief estimation error.

Second, \citet[Theorem~2]{nelson2022ContextLatent}
consider a milder noise condition (\Cref{assumption:idio-noise}
instead of \Cref{ass:SGnoise}) but to do so,
require a forgetting condition (as \Cref{assumption:strong-mixing-hmm}
below). We instead provide a more direct analysis, close to the standard LinUCB analysis
and not requiring this forgetting condition; see Appendix~\ref{appendix:simplified-analysis}.

Third, the bound of \citet[Theorem~2]{nelson2022ContextLatent}
is an expected bound, and not a bound in high probability; it
involves constants that heavily depend on the problem, in particular, on
the reward gaps, while the bound achieved in
Theorem~\ref{theorem:total-regret-bound-simplified} is model-free
for the part not linked to the estimation of HMM parameters, see
\Cref{eq:regret-bound-final-simplified}.

Fourth, \citet[Theorem~1]{nelson2022ContextLatent} also impose a non-degeneracy assumption on its population design matrix, which can
be stated as follows in our extended setting, denoting by $\lambda_{\min}$ the smallest eigenvalue:
for all actions $a \in \cA$,
\begin{multline*}
\liminf_{T \to \infty} \ \lambda_{\min}\bigl( \tilde{G}_t^{(a)} \bigr) > 0\,, \qquad \mbox{where} \\
\tilde{G}_t^{(a)} =
\frac{1}{T} \sum_{t=1}^{T} \E \Bigl[\Ind{a=a^{\star}_t}\bigl(\hbb_t \otimes \bphi(a, \bx_t)\bigr)
\bigl(\hbb_t \otimes \bphi(a, \bx_t)\bigr)^{\!\transp}\Bigr] \,.
\end{multline*}
Note that the sum in the definition of $\tilde{G}_t^{(a)}$ is restricted to rounds $t$ such that $a^{\star}_t = a$, where
\[
a^{\star}_t \in \argmax_{a \in \cA} \sum_{h \in [H]} \bb_t(h) \, \bphi(a, \bx_t)^{\transp} \btheta^{\star}_{h}\,.
\]
This implies that the population design matrix grows linearly in all directions. We do not impose such a coverage assumption; instead, we only use $1/\lambda_{\min}(G_t) \leq 1/\lambda$, which leads to a larger regret rate but avoids this additional condition.

In a nutshell, we leverage the reduction to linear contextual bandits
proposed by~\citet{nelson2022ContextLatent}
a in a more direct and more efficient way.

\subsection{Regret Bound for the Most Complex Model~\eqref{eq:reward-our}}
\label{eq:epsilonta-mostcomplex}

We require a final, classic (see \citealp{olivier-2005-inf-hmm}),
assumption on the HMM: that it satisfies some fast forgetting property.
Details, exemples, and further references (including
the alternative forgetting condition assumed by \citealp{nelson2022ContextLatent})
are provided in Appendix~\ref{appendix:HMM-forgetting}.

\begin{restatable}[exponentially fast forgetting of initial condition]{assumption}{assexpfastmixing}
	\label{assumption:strong-mixing-hmm}
There exists a constant $\gamma \in [0,1)$ so that, for all $s \leq t$, for all pairs $h, h' \in [H]$ of	hidden states,
	\begin{equation*}
		\begin{aligned}
			\Iftwocolumn{\smash{\sum_{j \in [H]}} \Bigl| \P_{\{h_s=h\}}&(h_t=j \mid \bx_{s+1:t}) \\
			& - \P_{\{h_s=h'\}}(h_t=j \mid \bx_{s+1:t}) \Bigr| \leq 2 \gamma^{t-s} \,.}{\sum_{j
            \in [H]} \Bigl| \P_{\{h_s=h\}}(h_t=j \mid \bx_{s+1:t})
            - \P_{\{h_s=h'\}}(h_t=j \mid \bx_{s+1:t}) \Bigr| \leq 2 \gamma^{t-s} \,.}
		\end{aligned}
	\end{equation*}
\end{restatable}

We may now state our main result. The $T^{7/8}$ rate achieved therein
must be contrasted with the $T^{3/4}$ rate discussed in \Cref{sec:regret-nelson}
above: the price to pay for facing the actual latent model
(and not an overly simplified version thereof) is a $T^{1/8}$ factor
with our method, mostly due to of proceeding in stages.
While Appendix~\ref{sec:challenge-overcome}
explains how handy it is to proceed in stages,
this might be avoidable and the regret bound might be improvable.
In particular, we do not provide any matching regret lower bound.

\begin{restatable}{theorem}{tregretbound}
	\label{theorem:total-regret-bound}
	Assume the horizon $T$ is known to the learner and fix $\delta \in (0, 1)$.
	Consider the strategy of Box~A with a belief
	estimation subroutine satisfying \Cref{assumption:bound-belief},
	with parameters $\lambda = T^{3/4}$ and $\ell = \lceil T^{3/4} \rceil$,
	as well as the confidence bonuses $\eps_{t,a} = 1 + \sqrt{d}/\lambda$
    for $t \in [1,\ell]$
	\Iftwocolumn{and for $t \geq \ell+1$,
	\begin{equation*}
		\begin{aligned}
			& \eps_{t,a} = \Ubelief(t, \delta/2)
            + f_t \biggl\Arrowvert G_{(s_t -1) \ell}^{-1} \Bigl(\hbb_t \otimes \bphi(a, \bx_t)\Bigr) \! \biggr\Arrowvert_2 \\
			& \mbox{where} \quad f_t = \lambda \sqrt{H} \, C_{\btheta^{\star}}
            + 4 \sqrt{ \frac{s_T(s_t -1) (1+s_t \gamma) \ell}{\delta(1 - \gamma)} } \ + \\
			& \sqrt{\frac{4 s_T}{\delta} C_{\eta} (s_t-1) \ell } + \frac{2 (s_t-1) \gamma}{1 - \gamma}
			+ \!\!\! \sum_{\tau=1}^{(s_t-1)\ell} \!\! \Ubelief(\tau,\delta/2)
		\end{aligned}
	\end{equation*}
    and where $s_t = \lceil t/ \ell \rceil$ denotes the stage to which round $t$ belongs.}{and for $t \geq \ell+1$,
	\begin{multline*}
			\eps_{t,a} = \Ubelief(t, \delta/2)
            + \biggl\Arrowvert G_{(s_t -1) \ell}^{-1} \Bigl(\hbb_t \otimes \bphi(a, \bx_t)\Bigr) \! \biggr\Arrowvert_2
			\smash{\overbrace{\Biggl( \lambda \sqrt{H} \, C_{\btheta^{\star}}
            + 4 \sqrt{ \frac{s_T(s_t -1) (1+s_t \gamma) \ell}{\delta(1 - \gamma)} }
			+ \sqrt{\frac{4 s_T}{\delta} C_{\eta} (s_t-1) \ell }}^{= f_t}} \\ + \frac{2 (s_t-1) \gamma}{1 - \gamma}
			+ \sum_{\tau=1}^{(s_t-1)\ell} \!\! \Ubelief(\tau,\delta/2)\Biggr) \,,
	\end{multline*}
    where $s_t = \lceil t/ \ell \rceil$ denotes the stage to which round $t$ belongs.}
	Then, under \Cref{assumption:bound-R-linear} (boundedness of rewards),
    \Cref{assumption:idio-noise} (noise with bounded conditional second-order moments), \Cref{assumption:bound-belief} (controlled belief estimation error),
    and \Cref{assumption:strong-mixing-hmm} (exponentially fast forgetting),
	with probability at least $1-\delta$,
	up to poly-log factors,
	\[
	R_T = \tO \bigl( T^{7/8} \bigr)\,.
	\]
	A closed-form expression of the regret bound may be found in
	the proof, see \Cref{eq:3terms,eq:Gsum}.
\end{restatable}

\subsection{Proof Sketch for \Cref{theorem:total-regret-bound}}

The full proof of \Cref{theorem:total-regret-bound} may be found in Appendix~\ref{appendix:complex-analysis}.

We introduce a filtration augmented by the algorithmic updates, by considering the estimates $\hbtheta_{s\ell}$
computed at the end of past complete stages $s \leq s_t -1$ on top of contexts $\bx_{1:t}$:
\[
\cU_t = \sigma\Bigl( \bx_{1:t}, \, \bigl( \hbtheta_{s \ell} \bigr)_{s \leq s_t -1} \Bigr)\,.
\]
The key of the proof, as discussed in Appendix~\ref{sec:challenge-overcome},
is that this filtration is such that $a_t$ is $\cU_t$--measurable
(by design, thanks to staging in Box A)
while
\begin{equation}
\label{eq:Ut-close-enough}
\overline{\bb}_t(h) = \P(h_t = h \mid \cU_t)
\quad \mbox{and} \quad \bb_t(h)
\end{equation}
are close enough;
this may be guaranteed by \Cref{assumption:strong-mixing-hmm} (exponentially fast forgetting
condition).

\paragraph{Summing confidence bounds.}
The core of the proof is to show that the confidence bonuses $\epsilon_{t,a}$
in \Cref{theorem:total-regret-bound} satisfy, with high probability,
uniformly over $T_0 \leq t \leq T$ and $a \in \cA$, that
\begin{multline}
\label{eq:core-proof-main}
\smash{\biggl|} \sum_{h \in \cH} \bb_t(h) \bphi(a, \bx_t)^{\transp} \btheta^{\star}_{h} \\
	- \sum_{h \in \cH} \hbb_t(h) \bphi(a, \bx_t)^{\transp} \hbtheta_{(s_t-1) \ell,h} \smash{\biggr|} \leq \eps_{t,a} + 2 T_0/\lambda\,,
\end{multline}
where $T_0$ is essentially the unknown constant threshold of \Cref{assumption:bound-belief}.
Based on that, classic manipulations entail that the pseudo-regret $R_T$ is
essentially bounded by
\[
2 \sum_{t \in [T]} \epsilon_{t,a_t} \quad \mbox{which is seen} \quad = \tO \bigl( T^{7/8} \bigr)
\]
by substituting classic linear-algebra bounds (the so-called elliptic potential lemma,
adapted to stages, see \citealp[Section~C]{abbasi2011improved})
and by carefully picking $\lambda$ and $\ell$ to optimize the bound.

Thus, the core of the proof is to show~\eqref{eq:core-proof-main}.

\paragraph{Three sums, including a difficult one.}
The left-hand side of~\eqref{eq:core-proof-main} is bounded by
the sum of two terms; first, $\Arrowvert \bb_t - \hbb_t \Arrowvert_1$, which is manageable
thanks to the estimation \Cref{assumption:bound-belief};
and second,
\begin{align*}
& \Bigl|\bigl(\hbb_t \otimes \bphi(a, \bx_t)\bigr)^{\transp} \bigl(\btheta^{\star} - \hbtheta_{(s_t-1) \ell} \bigr) \Bigr| \\
= & \ \Bigl(\hbb_t \otimes \bphi(a, \bx_t)\Bigr)^{\transp} \, G_{(s_t -1) \ell}^{-1} \, \times \, \\
& \quad  \bigl( \Sdiffp{(s_t-1) \ell} + \Sbeliefp{(s_t-1) \ell} + \Setap{(s_t-1) \ell} - \lambda \id{dH} \btheta^{\star} \bigr)
\end{align*}
where the equality follows
by substituting the very definition of $\hbtheta_{(s_t-1) \ell}$ and where
the exact definitions of the three $S$ terms are in Appendix~\ref{appendix:complex-analysis}.
We handle the term in the display above by a Cauchy-Schwarz inequality and the
boundedness \Cref{assumption:bound-R-linear}, together with
the fact the Euclidean norms
\[
\Arrowvert \Sdiffp{(s_t-1) \ell} \Arrowvert, \quad
\Arrowvert\Sbeliefp{(s_t-1) \ell}\Arrowvert, \quad
\Arrowvert\Setap{(s_t-1) \ell}\Arrowvert
\]
behave respectively as the absolute values of
\begin{align*}
S^\Delta = & \sum_{\tau=1}^{(s-1) \ell} \bphi(a_\tau,\bx_\tau)^{\transp} \biggl( \btheta^\star_{h_\tau}
- \sum_{h' \in [H]} \bar{\bb}_\tau(h') \btheta^\star_{h'} \biggr) \,, \\
S^b = & \sum_{\tau=1}^{(s-1)\ell} \sum_{h' \in [H]}  \bphi(a_\tau,\bx_\tau)^{\transp} \btheta^\star_{h'} \bigl( \bar{\bb}_\tau(h')
	- \hbb_\tau(h')  \bigr) \,, \\
S^\eta = & \sum_{\tau=1}^{(s-1)\ell} \eta_{\tau}(a_\tau) \,.
\end{align*}
Now, the term $S^b$ may be bounded by
\[
\sum_{\tau=1}^{(s-1)\ell} \bigl\Arrowvert \bar{\bb}_\tau
- \bb_\tau \bigr\Arrowvert_1
+ \sum_{\tau=1}^{(s-1)\ell} \bigl\Arrowvert \bb_\tau
- \hbb_\tau \bigr\Arrowvert_1\,,
\]
where \Cref{eq:Ut-close-enough} and \Cref{assumption:bound-belief}
take respective care of each sum.
The term $S^\eta$ could be bounded by resorting to martingale arguments (like
in the LinUCB analysis, though we rather mimic for it
in Appendix~\ref{appendix:complex-analysis} the proof scheme
used for $S^\Delta$ and described next).

The term $S^\Delta$ is the term that is difficult to control, see
the discussions in Appendix~\ref{sec:challenge-overcome}:
LinUCB-type analyses are not applicable.
Indeed, denoting
\[
z_\tau = \bphi(a_\tau,\bx_\tau)^{\transp} \biggl( \btheta^\star_{h_\tau}
- \sum_{h \in [H]} \bar{\bb}_\tau(h) \btheta^\star_h \biggr),
\]
where $|z_\tau| \leq 2$ by \Cref{assumption:bound-R-linear};
we have that
\[
\E \bigl[ z_\tau \mid U_{\tau} \bigr] = 0
\]
but $z_{\tau}$ is not $U_{\tau}$--measurable (as it explicitly depends on $h_\tau$).
However, we follow instead an approach by \citet{nelson2022ContextLatent}, which consists of
controlling $S^\Delta$ in $\L^2$--norm and applying Markov's inequality:
thanks to \Cref{assumption:strong-mixing-hmm} (exponentially fast forgetting
condition), $\E [ z_\tau z_{\tau'}]$ is exponentially small when
$\tau$ and $\tau'$ are separated, so that
\[
\E \bigl[ (S^\Delta)^2 \bigr] \quad \mbox{is of order} \quad s^2\ell\,,
\]
hence, $|S^\Delta|$ is smaller than $s\sqrt{\ell/\delta_s}$ with probability
at least $1-\delta_s$.

Collecting all elements, together with careful union bounds (taking $\delta_s = \delta/s_T$, where
$s_T$ denotes the stage of $T$), concludes the proof.
Again, the complete proof of \Cref{theorem:total-regret-bound} may be found in Appendix~\ref{appendix:complex-analysis}.

\section{Numerical Simulations}
\label{sec:simu}

We consider a partially simulated but realistic data set derived from the UCI ``Default of Credit Card Clients'' dataset (\citealp{UCI2016DefaultCR,Yeh2009creditcard}), in a banking marketing setup with three actions (calling a client;
emailing a client; not reaching out). Two latent states, inflation and recession, affect both context distributions and the rewards.

\Cref{fig:simulation-final} reports empirically estimated pseudo-regrets of the Box~A strategy (with stages of length $\ell=37$
or without stages, i.e., for $\ell =1$) versus a baseline formed by the LinUCB strategy by \citet{abbasi2011improved} in its standard form
(referred to as \emph{Plain LinUCB} in the picture). This baseline ignores the latent-state
dynamics altogether and therefore does not exploit either the HMM structure or the belief estimates;
as a result, it suffers linear pseudo-regret.
By contrast, the strategies developed achieve sublinear pseudo-regrets. 

Full simulation details, hyperparameter definitions, and robustness checks with respect to hyperparameter grids are provided in \Cref{appendix:simulation}.

\begin{figure}[hbt!]
\center \includegraphics[width=0.47\textwidth]{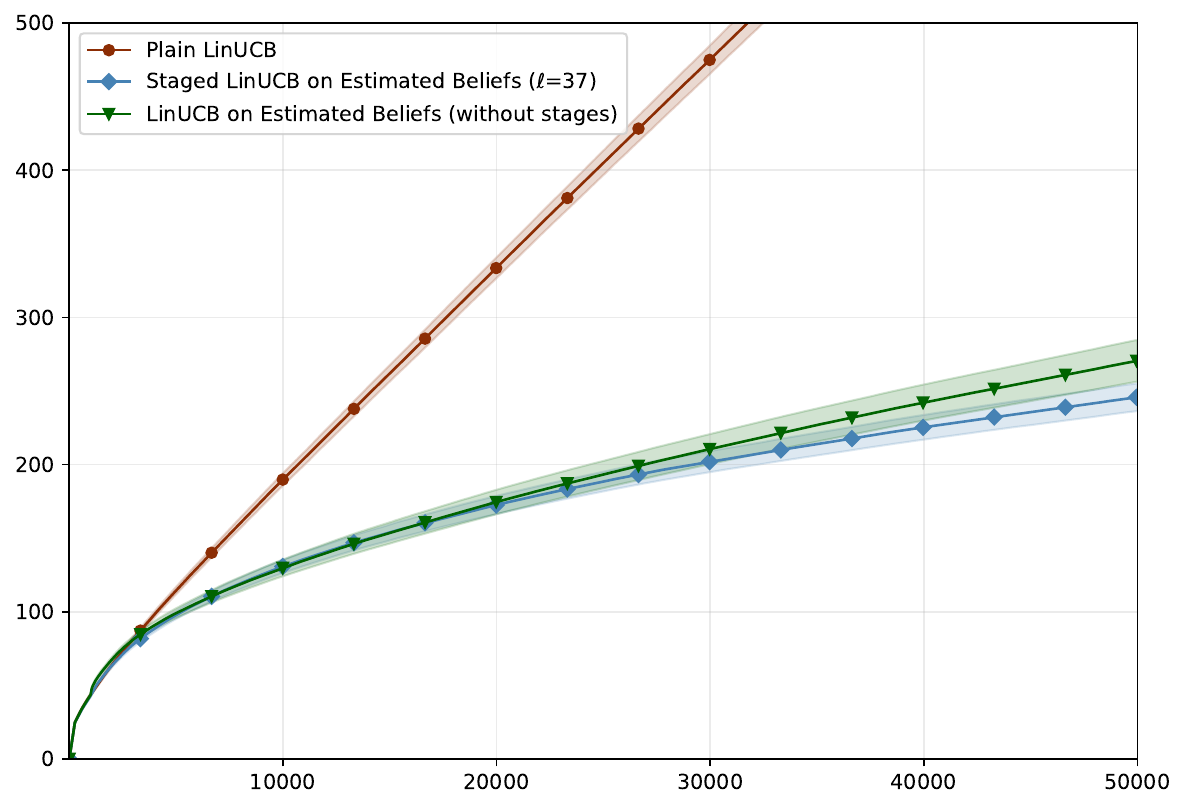}
\caption{\label{fig:simulation-final} Pseudo-regrets averaged over $100$ runs. Solid lines correspond to averages and 
shaded areas to $\pm 2$ standard errors.}
\end{figure}

\section{Limitations and Future Work}

The main open questions are around optimality:
first, showing that a $T^{3/4}$ rate on the pseudo-regret is inevitable,
even in the simplified reward model, due to belief estimation;
second, possibly improving the $T^{7/8}$ rate in the most complex reward model
into a $T^{3/4}$ rate by finding a more efficient theoretical argument than the $\mathbb{L}^2$--Markov
exhibited or, on the algorithmic front, by avoiding proceeding in stages. Indeed, the $\mathbb{L}^2$--Markov
argument entails dependencies on the probabilities of failure
as $1/\sqrt{\delta_s}$, instead of typical $\sqrt{\ln(1/\delta_s)}$ dependencies under
exponential-martingale arguments, and this worsened dependency
comes at a polynomial cost in the final regret bound.

\clearpage

\section*{Impact Statement}

This paper presents work whose goal is to advance the field of Machine
Learning. There are many potential societal consequences of our work, none
which we feel must be specifically highlighted here.

\bibliography{Li-Stoltz--HMM-Bandits--Bib}

\begin{thebibliography}{28}
\providecommand{\natexlab}[1]{#1}
\providecommand{\url}[1]{\texttt{#1}}
\expandafter\ifx\csname urlstyle\endcsname\relax
  \providecommand{\doi}[1]{doi: #1}\else
  \providecommand{\doi}{doi: \begingroup \urlstyle{rm}\Url}\fi

\bibitem[Abbasi-Yadkori et~al.(2011)Abbasi-Yadkori, P{\'a}l, and
  Szepesv{\'a}ri]{abbasi2011improved}
Abbasi-Yadkori, Y., P{\'a}l, D., and Szepesv{\'a}ri, C.
\newblock Improved algorithms for linear stochastic bandits.
\newblock \emph{Advances in Neural Information Processing Systems
  (NeurIPS'11)}, 24, 2011.

\bibitem[Anandkumar et~al.(2012)Anandkumar, Hsu, and
  Kakade]{anandkumar-2012-spectral-hmm}
Anandkumar, A., Hsu, D., and Kakade, S.~M.
\newblock A method of moments for mixture models and hidden {M}arkov models.
\newblock In \emph{Proceedings of the 25th Annual Conference on Learning Theory
  (COLT'2012)}, volume 23 of PMLR, pp.\  33.1--33.34, 2012.

\bibitem[Anandkumar et~al.(2014)Anandkumar, Ge, Hsu, Kakade, and
  Telgarsky]{anandkumar-2014-tensor-mixture}
Anandkumar, A., Ge, R., Hsu, D.~J., Kakade, S.~M., and Telgarsky, M.
\newblock Tensor decompositions for learning latent variable models.
\newblock \emph{Journal of Machine Learning Research}, 15\penalty0
  (1):\penalty0 2773--2832, 2014.

\bibitem[Austin \& Morgan(2025)Austin and
  Morgan]{austin2025ContextNonStationary}
Austin, E. and Morgan, L.~E.
\newblock Detecting changes and anomalies in nonstationary contextual bandits
  with an application to task categorisation.
\newblock \emph{Information Sciences}, 717:\penalty0 122270, 2025.

\bibitem[Azizzadenesheli et~al.(2016)Azizzadenesheli, Lazaric, and
  Anandkumar]{spectral-pomdps-2016}
Azizzadenesheli, K., Lazaric, A., and Anandkumar, A.
\newblock Reinforcement learning of {POMDPs} using spectral methods.
\newblock In \emph{Proceedings of the 29th Annual Conference on Learning Theory
  (COLT'2016)}, volume 49 of PMLR, pp.\  193--256, 2016.

\bibitem[Boyen \& Koller(1998)Boyen and Koller]{Boyen1998Stochastic}
Boyen, X. and Koller, D.
\newblock Tractable inference for complex stochastic processes.
\newblock In \emph{Proceedings of the Fourteenth Conference on Uncertainty in
  Artificial Intelligence (UAI'98)}, pp.\  33--42, 1998.

\bibitem[Br{\'e}g{\`e}re et~al.(2019)Br{\'e}g{\`e}re, Gaillard, Goude, and
  Stoltz]{bandtis-target-2019}
Br{\'e}g{\`e}re, M., Gaillard, P., Goude, Y., and Stoltz, G.
\newblock Target tracking for contextual bandits: Application to demand side
  management.
\newblock In \emph{Proceedings of the 36th International Conference on Machine
  Learning (ICML'20)}, volume 97 of PMLR, pp.\  754--763, 2019.

\bibitem[Capp{\'e} et~al.(2005)Capp{\'e}, Moulines, and
  Ryd{\'e}n]{olivier-2005-inf-hmm}
Capp{\'e}, O., Moulines, E., and Ryd{\'e}n, T.
\newblock \emph{Inference in Hidden Markov Models}.
\newblock Springer Series in Statistics. Springer, 2005.

\bibitem[Carpentier et~al.(2020)Carpentier, Vernade, and
  Abbasi-Yadkori]{Elliptical2020}
Carpentier, A., Vernade, C., and Abbasi-Yadkori, Y.
\newblock The elliptical potential lemma revisited, 2020.
\newblock Preprint, arXiv:2010.10182.

\bibitem[Chen \& Guestrin(2016)Chen and Guestrin]{Chen2016XGBoost}
Chen, T. and Guestrin, C.
\newblock {XGBoost}: A scalable tree boosting system.
\newblock In \emph{Proceedings of the 22nd ACM SIGKDD International Conference
  on Knowledge Discovery and Data Mining (KDD'16)}, pp.\  785--794, 2016.

\bibitem[Chu et~al.(2011)Chu, Li, Reyzin, and Schapire]{Chu2011ContextualBW}
Chu, W., Li, L., Reyzin, L., and Schapire, R.
\newblock Contextual bandits with linear payoff functions.
\newblock In \emph{Proceedings of the 14th International Conference on
  Artificial Intelligence and Statistics (AIStats'11)}, volume 15 of PMLR, pp.\
   208--214, 2011.

\bibitem[De~Castro et~al.(2017)De~Castro, Gassiat, and
  Le~Corff]{de2017consistent}
De~Castro, Y., Gassiat, E., and Le~Corff, S.
\newblock Consistent estimation of the filtering and marginal smoothing
  distributions in nonparametric hidden {M}arkov models.
\newblock \emph{IEEE Transactions on Information Theory}, 63\penalty0
  (8):\penalty0 4758--4777, 2017.

\bibitem[Ding \& Zhou(2007)Ding and Zhou]{ding2007eigenvalues}
Ding, J. and Zhou, A.
\newblock Eigenvalues of rank-one updated matrices with some applications.
\newblock \emph{Applied Mathematics Letters}, 20\penalty0 (12):\penalty0
  1223--1226, 2007.

\bibitem[Galozy et~al.(2025)Galozy, Nowaczyk, and
  Ohlsson]{galozy2025stateevobandit}
Galozy, A., Nowaczyk, S., and Ohlsson, M.
\newblock A new bandit setting balancing information from state evolution and
  corrupted context.
\newblock \emph{Data Mining and Knowledge Discovery}, 39\penalty0 (9), 2025.

\bibitem[Hong et~al.(2020{\natexlab{a}})Hong, Kveton, Zaheer, Chow, Ahmed, and
  Boutilier]{hong2020latentbandits}
Hong, J., Kveton, B., Zaheer, M., Chow, Y., Ahmed, A., and Boutilier, C.
\newblock Latent bandits revisited.
\newblock \emph{Advances in Neural Information Processing Systems
  (NeurIPS'20)}, 33, 2020{\natexlab{a}}.

\bibitem[Hong et~al.(2020{\natexlab{b}})Hong, Kveton, Zaheer, Chow, Ahmed,
  Ghavamzadeh, and Boutilier]{hong2020ContextLatent}
Hong, J., Kveton, B., Zaheer, M., Chow, Y., Ahmed, A., Ghavamzadeh, M., and
  Boutilier, C.
\newblock Non-stationary latent bandits, 2020{\natexlab{b}}.
\newblock Preprint, arXiv:2012.00386.

\bibitem[Hsu et~al.(2012)Hsu, Kakade, and Zhang]{spectral-2012}
Hsu, D., Kakade, S.~M., and Zhang, T.
\newblock A spectral algorithm for learning hidden {M}arkov models.
\newblock \emph{Journal of Computer and System Sciences}, 78\penalty0
  (5):\penalty0 1460--1480, 2012.

\bibitem[Kontorovich \& Weiss(2014)Kontorovich and
  Weiss]{kontorovich2014uniform}
Kontorovich, A. and Weiss, R.
\newblock Uniform {C}hernoff and {Dvoretzky-Kiefer-Wolfowitz}-type inequalities
  for {M}arkov chains and related processes.
\newblock \emph{Journal of Applied Probability}, 51\penalty0 (4):\penalty0
  1100--1113, 2014.

\bibitem[Krishnamurthy(2016)]{pomdp-krishnamurthy-2016}
Krishnamurthy, V.
\newblock \emph{Partially Observed Markov Decision Processes}.
\newblock Cambridge University Press, 2016.

\bibitem[Lattimore \& Szepesv{\'a}ri(2020)Lattimore and
  Szepesv{\'a}ri]{lattimore2020bandit}
Lattimore, T. and Szepesv{\'a}ri, C.
\newblock \emph{Bandit Algorithms}.
\newblock Cambridge University Press, 2020.

\bibitem[Li \& Stoltz(2022)Li and Stoltz]{CBwK-LP-2022}
Li, Z. and Stoltz, G.
\newblock Contextual bandits with knapsacks for a conversion model.
\newblock In \emph{Advances in Neural Information Processing Systems
  (NeurIPS'22)}, volume~35, 2022.

\bibitem[Nelson et~al.(2022)Nelson, Bhattacharjya, Gao, Liu, Bouneffouf, and
  Poupart]{nelson2022ContextLatent}
Nelson, E., Bhattacharjya, D., Gao, T., Liu, M., Bouneffouf, D., and Poupart,
  P.
\newblock Linearizing contextual bandits with latent state dynamics.
\newblock In \emph{Proceedings of the Thirty-Eighth Conference on Uncertainty
  in Artificial Intelligence (UAI'22)}, volume 180 of PMLR, pp.\  1477--1487,
  2022.

\bibitem[Robbins(1952)]{robbins1952}
Robbins, H.
\newblock Some aspects of the sequential design of experiments.
\newblock \emph{Bulletin of the American Mathematical Society}, 58\penalty0
  (5):\penalty0 527--535, 1952.

\bibitem[Thompson(1933)]{thompson1933MAB}
Thompson, W.
\newblock On the likelihood that one unknown probability exceeds another in
  view of the evidence of two samples.
\newblock \emph{Biometrika}, 25\penalty0 (3-4):\penalty0 285--294, 1933.

\bibitem[Wu et~al.(2018)Wu, Iyer, and Wang]{wu2018ContextNonStationary}
Wu, Q., Iyer, N., and Wang, H.
\newblock Learning contextual bandits in a non-stationary environment.
\newblock In \emph{Proceedings of the 41st International ACM SIGIR Conference
  on Research \& Development in Information Retrieval}, pp.\  495--504, 2018.

\bibitem[Yeh(2009)]{UCI2016DefaultCR}
Yeh, I.-C.
\newblock {Default of Credit Card Clients}.
\newblock UCI Machine Learning Repository, 2009.
\newblock {DOI}: https://doi.org/10.24432/C55S3H.

\bibitem[Yeh \& Lien(2009)Yeh and Lien]{Yeh2009creditcard}
Yeh, I.-C. and Lien, C.
\newblock The comparisons of data mining techniques for the predictive accuracy
  of probability of default of credit card clients.
\newblock \emph{Expert Systems with Applications}, 36\penalty0 (2):\penalty0
  2473--2480, 2009.

\bibitem[Zhou et~al.(2021)Zhou, Xiong, Chen, and Gao]{RegimeBandits-2021}
Zhou, X., Xiong, Y., Chen, N., and Gao, X.
\newblock Regime switching bandits.
\newblock In \emph{Advances in Neural Information Processing Systems
  (NeurIPS'21)}, volume~34, 2021.

\end{thebibliography}
\bibliographystyle{icml2026}

\newpage

\appendix

\onecolumn
\section{Algorithm and Analysis for the Simplified Reward Model of \citet{nelson2022ContextLatent}}
\label{appendix:simplified-analysis}

\Cref{sec:simplified} indicated that (a generalized version of) the simplified reward model by \citet{nelson2022ContextLatent}
may be stated as
\[
r'_t(a) = \sum_{h \in [H]} \bb_t(h) \, \bphi(a, \bx_t)^{\transp} \btheta^{\star}_h + \eta'_t(a)\,,
\qquad \mbox{where} \qquad \bb_t(h) = \P(h_t = h \mid \bx_{1:t})
\]
and where here, we assume that the noise terms satisfy a sub-Gaussian assumption as in \Cref{ass:SGnoise}:
denoting by
\[
\cFallp_{t} = \sigma\biggl( \Bigl(  h_\tau,\, \bx_\tau, \, \bigl( \eta'_\tau(a) \bigr)_{a \in \cA} \Bigr)_{\tau \leq t-1}, \, h_t, \, \bx_t \biggr)
\]
the filtration with respect to all random variables anterior to the $\eta'_t(a)$,
there exists $v_\eta$ such that for all $a \in \cA$,
\begin{equation}
\label{eq:ass:SG}
\E \bigl[\eta'_t(a) \mid \cFallp_t \bigr] = 0 \qquad \mbox{and} \qquad
\E \bigl[ \e^{\lambda \eta'_t(a)} \mid \cFallp_t \bigr] \leq \e^{\lambda^2 v_\eta^2/2} \,.
\end{equation}

\paragraph{Aim of this appendix.}
This appendix recalls the main claim by \citet{nelson2022ContextLatent},
namely, how
a reduction to standard linear contextual bandits may be performed for the reward model above.
Unlike \citet{nelson2022ContextLatent},
we also provide a straightforward analysis based on the LinUCB analysis,
taking into account the belief estimation error
(see Appendix~\ref{sec:belief-estimation-error} for a description
of a belief estimation routine and its associated guarantees), and yielding high-probability
bounds (not only bounds in expectation); no HMM forgetting
properties are required to that end.

Actually, the more complex analysis by \citet{nelson2022ContextLatent}, which, in particular,
relies on HMM forgetting properties
(as in \Cref{assumption:strong-mixing-hmm}, see more generally Appendix~\ref{appendix:HMM-forgetting}),
is only required because of the relaxation considered on the noise terms:
\citet{nelson2022ContextLatent} only assume that conditional second-order moments are bounded,
as in \Cref{assumption:idio-noise}. We see this relaxation as unimportant.

\paragraph{Algorithm.}
We use the reduction to linear contextual bandits pointed out by \citet{nelson2022ContextLatent},
discussed in \Cref{sec:simplified}, and relying on the rewriting
\[
r'_t(a) = \sum_{h \in [H]} \bb_t(h) \, \bphi(a, \bx_t)^{\transp} \btheta^{\star}_h + \eta'_t(a)
= \bigl( \bb_t \otimes \bphi(a, \bx_t) \bigr)^{\transp} \btheta^{\star} + \eta'_t(a)\,,
\qquad \mbox{where} \qquad \btheta^{\star} = \bigl( \btheta^{\star}_{h} \bigr)_{h \in [H]}\,;
\]
the quantities $\bb_t \otimes \bphi(a, \bx_t)$ act as (unknown) contexts and mean rewards
depend linearly on them, via the $dH$--dimensional parameter $\btheta^{\star}$.
We also consider a belief estimation subroutine $\cB$,
as discussed in \Cref{sec:belief-estimation} (see also Appendix~\ref{sec:belief-estimation-error}),
so as to replace the unknown contexts by known estimated contexts.
Fix a regularization parameter $\lambda > 0$.
At the end of each round $t \geq 1$, the algorithm computes
\begin{multline}
\label{eq:estimate-theta-simplified}
\hbtheta_t = \bigl(\hbtheta_{t,h}\bigr)_{h \in [H]}
\eqdef G_t^{-1} \sum_{\tau=1}^{t} \bigl(\hbb_\tau \otimes \bphi(a_\tau, \bx_\tau)\bigr) \, r'_\tau(a_\tau)\,, \\
\mbox{where} \qquad
G_t \eqdef \sum_{\tau=1}^{t} \bigl(\hbb_\tau \otimes \bphi(a_\tau, \bx_\tau)\bigr)
\bigl(\hbb_\tau \otimes \bphi(a_\tau, \bx_\tau)\bigr)^{\!\transp} + \lambda \id{dH}\,,
\end{multline}
based, in particular, on the estimated belief $\hbb_t$ obtained from $\cB$ at the beginning of round $t$.
Then, in the next round $t+1$,
\[
\mbox{mean rewards} \ \ \sum_{h \in [H]} \bb_{t+1}(h) \, \bphi(a, \bx_{t+1})^{\transp} \btheta^{\star}_h
\qquad \mbox{are estimated by} \qquad
\hr_{t+1}(a) \eqdef \sum_{h \in [H]} \hbb_{t+1}(h) \, \bphi(a, \bx_{t+1})^{\transp} \hbtheta_{t,h}\,.
\]
The action $a_{t+1}$ to be played at round $t+1$ is picked in an optimistic way
as the action maximizing $\hr_{t+1}(a)$ plus some confidence bonus over $a \in \cA$.
The corresponding algorithm is formally stated in Box~B.
It corresponds to a LinUCB approach (\citealp{abbasi2011improved})
with contexts computed based on estimated beliefs; the mere difference
to the main algorithm of Box~A is that it does not proceed in stages.

\begin{figure}[h]
	\begin{nbox}[title={Box B: LinUCB on estimated beliefs (without stages)}]
		\textbf{Known parameters:}
		finite action set $\cA$; finite state space $[H]$; context space $\cX$;
		transfer function $\bphi: \cA \times \cX \to \R^d$ \hspace{-1cm} \, \smallskip
		
		\textbf{Unknown parameters:}
		HMM parameters, given by a transition matrix $\bM = (M_{h,h'})_{(h,h') \in [H]}$
		and emission distributions $(\nu_h)_{h \in [H]}$ over $\cX$;
        reward parameters $\btheta^{\star}_{h} \in \R^d$, for $h \in [H]$ \smallskip
		
		\textbf{Inputs:} risk $\delta \in (0,1)$; belief estimation subroutine $\cB$ (see Section~\ref{sec:belief-estimation});
		regularization parameter $\lambda > 0$; closed-form
		expression for the confidence bonuses $\eps_{t,a}$, possibly depending on $\delta$ and $\lambda$ \smallskip
		
		\textbf{Initialization:} the learner sets $\hbtheta_0 = (1/\lambda) \, \bone \in \R^{dH}$ \medskip
		
		\textbf{For} rounds $t \geq 1$ \textbf{the learner:} \smallskip
		\begin{enumerate}
			\item Observes the context $\bx_{t}$, drawn independently by the environment from $\nu_{h_{t}}$; \smallskip
			\item Obtains the belief estimate $\hbb_{t}$ by feeding $\bx_1,\ldots,\bx_{t}$ to the subroutine $\cB$; \smallskip
			\item Computes the estimated mean rewards \quad $\displaystyle{\hr_{t}(a) = \sum_{h \in \cH} \hbb_{t}(h) \, \bphi(a, \bx_t)^{\transp} \hbtheta_{t-1,h}}$ \quad for all $a \in \cA$;
			\item Picks an action $\displaystyle{a_{t} \in \argmax_{a \in \cA} \bigl\{ \hat{r}_{t}(a) + \eps_{t, a}} \bigr\}$\,;
			\item Obtains and observes the reward \quad $\displaystyle{r'_t(a_t) = \sum_{h \in [H]} \bb_{t}(h) \, \bphi(a_t, \bx_t)^{\transp} \btheta^\star_{h} + \eta'_t(a_t)}$\,;
			\item Computes $\hbtheta_t$ as in \Cref{eq:estimate-theta-simplified}.
		\end{enumerate}
	\end{nbox}
\vspace{-.5cm}
\end{figure}

\paragraph{Analysis.}
At a high level, the analysis adapts the LinUCB proof to handle the substitution of the
true contexts $\bb_t \otimes \bphi(a, \bx_t)$ by estimations thereof.
We follow closely classic analyses of LinUCB (the original reference by \citealp{abbasi2011improved},
the monograph by \citealp[Chapters~19 and~20]{lattimore2020bandit}, as well
as the extension by \citealp{bandtis-target-2019}), with occasional simplifications
or shortcuts---e.g., we avoid stating confidence ellipsoids on the $\btheta^{\star}_{h}$
and rather focus on confidence intervals on the mean payoffs, as studied
in \Cref{lm:est-simple} below. The bounds obtained in the sequel
corresponds to the classic bound when $\hbb_t = \bb_t$, i.e., if there was no estimation error for the beliefs.
The formal aim is to prove the following theorem.

\begin{restatable}{theorem}{totalregretboundsimplified}
	\label{theorem:total-regret-bound-simplified}
	Assume that the horizon $T$ is known to the learner and fix $\delta \in (0, 1)$.
	Consider the strategy of Box~B with a belief
	estimation subroutine satisfying \Cref{assumption:bound-belief},
	with $\lambda = T^{1/2}$ and with the confidence bonuses~\eqref{eq:def-esp-simple}.
	Then, under the boundedness stated in \Cref{assumption:bound-R-linear}
    and under the sub-Gaussian noise assumption~\eqref{eq:ass:SG},
	with probability at least $1-\delta$,
	up to poly-log factors,
	\[
	R_T = \tO \bigl(T^{3/4} \bigr)\,,
	\]
	where a closed-form expression of the regret bound may be found in
	the proof, see \Cref{eq:regret-bound-final-simplified}.
\end{restatable}

The total regret bound is basically given by $2$ times the sum of the upper confidence bounds
of \Cref{lm:est-simple} below, which we prove first.

\begin{lemma}
\label{lm:est-simple}
Under \Cref{assumption:bound-R-linear} and for sub-Gaussian noise terms as in \Cref{eq:ass:SG},
with probability at least $1-\delta$, for all $t \geq 2$,
\begin{multline*}
\forall a \in \cA, \qquad\qquad \left| \sum_{h \in \cH} \bb_t(h) \, \bphi(a, \bx_t)^{\transp} \btheta^{\star}_{h}
- \sum_{h \in \cH} \hbb_t(h) \, \bphi(a, \bx_t)^{\transp} \hbtheta_{t-1,h} \right| \\ \leq
\bigl\Arrowvert \bb_t - \hbb_t \bigr\Arrowvert_1 + \Bigl\Arrowvert \hbb_t \otimes \bphi(a, \bx_t) \Bigr\Arrowvert_{G_{t-1}^{-1}} \, \Biggl( \frac{1}{\sqrt{\lambda}}\sum_{\tau=1}^{t-1} \Arrowvert \bb_{\tau} - \hbb_{\tau} \Arrowvert_1  + \sqrt{\lambda H} \, C_{\btheta^{\star}}
+ v_\eta \sqrt{2\ln(1/\delta) + dH \ln\bigl(1 + t/(\lambda d H)\bigr)} \Biggr)\,.
\end{multline*}
\end{lemma}

In the case $t = 1$, by several triangle inequalities, \Cref{assumption:bound-R-linear}, the fact that $\bb_1$ and $\hbb_1$ are probability vectors, and a Cauchy-Schwarz inequality, we have, with probability~$1$: for all $a \in \cA$,
\[
\left| \sum_{h \in \cH} \bb_1(h) \, \bphi(a, \bx_1)^{\transp} \btheta^{\star}_{h}
- \sum_{h \in \cH} \hbb_1(h) \, \bphi(a, \bx_t)^{\transp} \hbtheta_{0,h} \right|
\leq \smash{\max_{h \in [H]} \overbrace{\bigl| \bphi(a, \bx_t)^{\transp} \btheta^{\star}_{h} \bigr|}^{\leq 1}
+ \max_{h \in [H]} \overbrace{\bigl| \bphi(a, \bx_t)^{\transp} \hbtheta_{0,h} \bigr|}^{\leq
\sqrt{d}/\lambda}
\leq 1 + \frac{\sqrt{d}}{\lambda}}\,,
\]
where we resorted to the Cauchy-Schwarz inequality
$\Arrowvert \bphi(a, \bx_1) \bigr\Arrowvert_2 \, \bigl\Arrowvert \hbtheta_{0,h} \bigr\Arrowvert_2
\leq 1 \times \sqrt{d}/\lambda$, since $\hbtheta_{0,h} = \bone \in \R^d$.

\begin{proof}
	By a triangle inequality and by the boundedness stated in \Cref{assumption:bound-R-linear},
	\begin{multline}
		\left| \sum_{h \in \cH} \bb_t(h) \, \bphi(a, \bx_t)^{\transp} \btheta^{\star}_{h}
			- \sum_{h \in \cH} \hbb_t(h) \, \bphi(a, \bx_t)^{\transp} \hbtheta_{t-1,h} \right| \\
		\leq \smash{\underbrace{\Biggl| \sum_{h \in \cH} \bigl(\bb_t(h) - \hbb_t(h) \bigr) \overbrace{\bphi(a, \bx_t)^{\transp} \btheta^{\star}_{h}}^{|\,\cdot\,| \leq 1} \Biggr|}_{\leq \Arrowvert \bb_t - \hbb_t \Arrowvert_1}} +
		\label{eq:rest-proof-eq-simplified}
		\Biggl| \underbrace{\sum_{h \in \cH} \hbb_t(h) \bphi(a, \bx_t)^{\transp} \bigl(\btheta^{\star}_{h} - \hbtheta_{t-1,h} \bigr)}_{=
			(\hbb_t \otimes \bphi(a, \bx_t))^{\transp} (\btheta^{\star} - \hbtheta_{t-1})} \Biggr|\,.
	\end{multline}
    The rest of the proof bounds the second term in the upper bound of \Cref{eq:rest-proof-eq-simplified}.
    To that end, we rewrite $r'_{\tau}(a_{\tau})$ as
	\[
		r'_{\tau}(a_{\tau})
        = \bigl( (\bb_{\tau} - \hbb_{\tau}) \otimes \bphi(a_{\tau},\bx_{\tau}) \bigr)^{\transp} \btheta^\star
        + \bigl(\hbb_\tau \otimes \bphi(a_\tau, \bx_\tau)\bigr)^{\!\transp} \btheta^{\star} + \eta'_{\tau}(a_{\tau})
	\]
    and also note that by the definition of $G_{t-1}$ in~\Cref{eq:estimate-theta-simplified},
	\begin{align*}
		\btheta^{\star} - \hbtheta_{t-1}
		& = G_{t-1}^{-1} \left( G_{t-1} \, \btheta^{\star} -
		\sum_{\tau=1}^{t-1} \bigl(\hbb_\tau \otimes \bphi(a_\tau, \bx_\tau)\bigr) r'_\tau(a_\tau) \right) \\
		& = G_{t-1}^{-1} \left(\lambda \btheta^{\star} -
		\sum_{\tau=1}^{t-1} \bigl(\hbb_\tau \otimes \bphi(a_\tau, \bx_\tau)\bigr) \Bigl(r'_\tau(a_\tau)  -
		\bigl(\hbb_\tau \otimes \bphi(a_\tau, \bx_\tau)\bigr)^{\!\transp} \btheta^{\star}\Bigr) \right).
	\end{align*}
    Thanks to these two equalities, we may decompose the second term in the upper bound of \Cref{eq:rest-proof-eq-simplified} as
    \begin{multline*}
    \bigl( \hbb_t \otimes \bphi(a, \bx_t) \bigr)^{\transp} (\btheta^{\star} - \hbtheta_{t-1})
    = \bigl( \hbb_t \otimes \bphi(a, \bx_t) \bigr)^{\transp} \, G_{t-1}^{-1}
    \bigl( \lambda \btheta^{\star} - \Sdiff{t-1} - \Seta{t-1} \bigr)
    \\ \mbox{where} \qquad
    \Sdiff{t-1} = \sum_{\tau=1}^{t-1} \bigl(\hbb_\tau \otimes \bphi(a_\tau, \bx_\tau)\bigr)
    \bigl( (\bb_{\tau} - \hbb_{\tau}) \otimes \bphi(a_{\tau},\bx_{\tau}) \bigr)^{\transp} \btheta^\star
    \quad \mbox{and} \quad
    \Seta{t-1} = \sum_{\tau=1}^{t-1} \bigl(\hbb_\tau \otimes \bphi(a_\tau, \bx_\tau)\bigr) \eta'_{\tau}(a_{\tau})\,.
    \end{multline*}
    A Cauchy-Schwarz inequality for the inner product induced by $G_{t-1}^{-1}$, together with a triangle inequality,
    entails
    \[
    \Bigl| \bigl( \hbb_t \otimes \bphi(a, \bx_t) \bigr)^{\transp} (\btheta^{\star} - \hbtheta_{t-1}) \Bigr|
    \leq \Bigl\Arrowvert \hbb_t \otimes \bphi(a, \bx_t) \Bigr\Arrowvert_{G_{t-1}^{-1}} \,
    \Bigl( \lambda \Arrowvert \btheta^{\star} \Arrowvert_{G_{t-1}^{-1}} + \Arrowvert \Sdiff{t-1} \Arrowvert_{G_{t-1}^{-1}}
    + \Arrowvert \Seta{t-1} \Arrowvert_{G_{t-1}^{-1}} \Bigr)\,.
    \]
    Since $G_{t-1} \succeq \lambda \id{dH}$,
    \begin{equation}
    \label{eq:G-norm-bdl}
    \forall \bu \in \R^{dH}\,, \qquad \Arrowvert \bu \Arrowvert_{G_{t-1}^{-1}}
    \leq \frac{\Arrowvert \bu \Arrowvert_2}{\sqrt{\lambda}}\,.
    \end{equation}
    Therefore, using~\Cref{assumption:bound-R-linear},
    \[
    \lambda \Arrowvert \btheta^{\star} \Arrowvert_{G_{t-1}^{-1}} \leq \sqrt{\lambda} \, \Arrowvert \btheta^{\star} \Arrowvert_2 \,,
    \leq \sqrt{\lambda H}\,C_{\btheta^{\star}}
    \]
    and, again by~\Cref{assumption:bound-R-linear}, by~\Cref{eq:bbbphi-norm}, and by a triangular inequality,
    \[
    \Arrowvert \Sdiff{t-1} \Arrowvert_{G_{t-1}^{-1}} \leq \frac{1}{\sqrt{\lambda}}
    \Biggl\Arrowvert \sum_{\tau=1}^{t-1} \sum_{h \in [H]}
	\overbrace{\bphi(a_\tau,\bx_\tau)^{\transp} \btheta^\star_h}^{|\,\cdot\,| \leq 1} \, \bigl( \bb_\tau(h) \,
	- \hbb_\tau(h) \bigr) \, \overbrace{\hbb_\tau \otimes \bphi(a_\tau,\bx_\tau)}^{\Arrowvert\,\cdot\,\Arrowvert_2 \leq 1}
	\Biggr \Arrowvert_2 \leq \frac{1}{\sqrt{\lambda}}
    \sum_{\tau=1}^{t-1} \bigl\Arrowvert \bb_\tau - \hbb_\tau \bigr\Arrowvert_1\,.
    \]
    As for the final term $\smash{\Arrowvert \Seta{t-1} \Arrowvert_{G_{t-1}^{-1}}}$, it is exactly
    of the form discussed in \citet[Theorem~1 and Lemma~10, recalled below as \Cref{lm:linucb-dev}]{abbasi2011improved}, with
    \[
    d' = dH\,, \qquad G = \lambda\id{dH}\,, \qquad X_\tau = \hbb_\tau \otimes \bphi(a_\tau, \bx_\tau)\,, \qquad
    \eta_\tau = \eta'_\tau(a_\tau)\,, \qquad \cF_\tau = \cFallp_{\tau+1}\,;
    \]
    indeed, $a_\tau$ and $X_\tau$ are $\cFallp_{\tau}$--measurable while the $\eta'_t(a)$, and
    thus also $\eta'_t(a_t)$, are $\cFallp_{\tau+1}$--measurable and $v_\eta$--sub-Gaussian
    conditionally to $\cFallp_\tau$, as stated in \Cref{eq:ass:SG}.
    In addition, \Cref{eq:bbbphi-norm} guarantees that $\Arrowvert X_\tau \Arrowvert_2 \leq 1$.
    We get that probability at least $1 - \delta$, for all $t \geq 1$,
	\[
	\Biggl\Arrowvert \sum_{\tau=1}^{t-1} \eta'_{\tau}(a_\tau) \, \hbb_\tau \otimes \bphi(a_\tau,\bx_\tau) \Biggr\Arrowvert_{G_{t-1}^{-1}}
	\leq v_\eta \sqrt{2\ln(1/\delta) + dH \ln\bigl(1 + t/(\lambda d H)\bigr)}\,.
	\]
    The proof is concluded by collecting all the bounds above.
\end{proof}

For the convenience of the reader, we restate
the key classic deviation inequality used above.
We recall that sub-Gaussian random variables are necessarily centered.
Lemma~10 by \citet{abbasi2011improved} is exactly
\Cref{lm:matrix-determinant-bis} stated at the end of this appendix.

\begin{lemma}[{\citealp[Theorem~1 and Lemma~10]{abbasi2011improved}}]
\label{lm:linucb-dev}
Consider a filtration $(\cF_t)_{t \geq 0}$ and two
stochastic processes,
a scalar-valued process $(\eta_t)_{t \geq 1}$ such that $\eta_t$ is $\cF_t$--measurable and
$v_\eta$--sub-Gaussian conditionally to $\cF_{t-1}$,
and a $d'$--vector-valued process $(X_t)_{t \geq 1}$ such that $X_t$ is $\cF_{t-1}$--measurable
and $\Arrowvert X_t \Arrowvert_2 \leq 1$ a.s.
For $\lambda > 0$, let
\[
S_t = \sum_{\tau=1}^{t} \eta_{\tau} X_\tau
\qquad \mbox{and} \qquad
G_t = \lambda \id{d'} + \sum_{\tau=1}^{t} X_\tau X_\tau^{\transp}\,.
\]
Then, with probability at least $1-\delta$, for all $t \geq 1$,
\[
\Arrowvert S_t \Arrowvert_{G_{t}^{-1}}
\leq v_\eta \, \sqrt{2\ln(1/\delta) + d'\ln\bigl(1 + t/(\lambda d')\bigr)} \,.
\]
\end{lemma}

We consider the Box~B strategy with a belief
estimation subroutine satisfying \Cref{assumption:bound-belief}
(for which we recall that the belief error function $\Ubelief$
is known) and the confidence bonuses, for $t \geq 2$ and $a \in \cA$,
\begin{align}
\label{eq:def-esp-simple}
\epsilon_{t,a} =
& \ \Ubelief(t,\delta/2) \\
\nonumber
& \quad + \Bigl\Arrowvert \hbb_t \otimes \bphi(a, \bx_t) \Bigr\Arrowvert_{G_{t-1}^{-1}} \, \Biggl( \frac{1}{\sqrt{\lambda}}\sum_{\tau=1}^{t-1}
\Ubelief(\tau,\delta/2) + \sqrt{\lambda H} \, C_{\btheta^{\star}}
+ v_\eta \sqrt{2\ln(2/\delta) + dH \ln\bigl(1 + t/(\lambda d H)\bigr)} \Biggr)
\end{align}
and $\epsilon_{1,a} = 1 + \sqrt{d}/\lambda$.
The confidence bonuses $\epsilon_{t,a}$
correspond to the upper bounds of
\Cref{lm:est-simple}, denoted by $\epsilon'_{t,a}$ in the proof below, up to the replacements of $\delta$ by $\delta/2$ and
of the unknown
$\Arrowvert \bb_{\tau} - \hbb_{\tau} \Arrowvert_1$
by their high-probability bounds $\Ubelief(\tau,\delta/2)$.
Recall the statement of our theorem, which we may now prove below.

\totalregretboundsimplified*

\begin{proof}
	We denote by
	\begin{align*}
	\epsilon'_{t,a} =
	& \ \bigl\Arrowvert \bb_t - \hbb_t \bigr\Arrowvert_1 \\
	& \quad + \Bigl\Arrowvert \hbb_t \otimes \bphi(a, \bx_t) \Bigr\Arrowvert_{G_{t-1}^{-1}} \, \Biggl( \frac{1}{\sqrt{\lambda}}\sum_{\tau=1}^{t-1}
	\bigl\Arrowvert \bb_\tau - \hbb_\tau \bigr\Arrowvert_1 + \sqrt{\lambda H} \, C_{\btheta^{\star}}
	+ v_\eta \sqrt{2\ln(2/\delta) + dH \ln\bigl(1 + t/(\lambda d H)\bigr)} \Biggr)
	\end{align*}
	the upper bound read in \Cref{lm:est-simple} for the risk $\delta/2$,
	and let $\epsilon'_{1,a} = 1 + \sqrt{d}/\lambda$.
	With this piece of notation,
	\Cref{lm:est-simple} guarantees, in particular, that with probability at least $1-\delta/2$,
	for all $t \geq 1$,
	\begin{align}
    \label{eq:csq-lmsimple-1}
	\max_{a \in \cA} \sum_{h \in \cH} \bb_t(h) \, \bphi(a, \bx_t)^{\transp} \btheta^{\star}_{h}
	& \leq \max_{a \in \cA} \left\{ \epsilon'_{t,a} + \sum_{h \in \cH} \hbb_t(h) \, \bphi(a, \bx_t)^{\transp} \hbtheta_{t-1,h} \right\} \\
    \label{eq:csq-lmsimple-2}	
    \mbox{and} \qquad\qquad
	\sum_{h \in \cH} \hbb_t(h) \, \bphi(a_t, \bx_t)^{\transp} \hbtheta_{t-1,h}
	& \leq \epsilon'_{t,a_t} + \sum_{h \in \cH} \bb_t(h) \, \bphi(a_t, \bx_t)^{\transp} \btheta^{\star}_{h}\,.
	\end{align}
	On the other hand, \Cref{assumption:bound-belief} ensures that
	with probability at least $1-\delta/2$,
	\begin{equation}
    \label{eq:T0-Ubelief-1}
	\forall t \in [ T_0, \,\, T], \qquad
	\bigl\Arrowvert \bb_t - \hbb_t \bigr\Arrowvert_1 \leq \Ubelief(t,\delta/2)\,,
	\qquad \mbox{where} \qquad
	T_0 \eqdef \max\Bigl\{2, \,\, \lceil T_{\cB,\bM,\nu}\bigl(1 + \ln(2 /\delta)\bigr)
	\rceil \Bigr\}\,;
	\end{equation}
	therefore, with probability at least $1-\delta/2$,
	\begin{equation}
    \label{eq:T0-Ubelief-2}
    \forall t \in [ T_0, \,\, T], \qquad
	\sum_{\tau=1}^{t-1}
	\bigl\Arrowvert \bb_\tau - \hbb_\tau \bigr\Arrowvert_1
	\leq 2 (T_0 - 1) + \sum_{\tau=1}^{t-1} \Ubelief(t,\delta/2)
	\quad \mbox{and} \quad
	\epsilon'_{t,a} \leq \epsilon_{t,a} + \frac{2 (T_0 - 1)}{\sqrt{\lambda}}\,
	\Bigl\Arrowvert \hbb_t \otimes \bphi(a, \bx_t) \Bigr\Arrowvert_{G_{t-1}^{-1}}\,.
	\end{equation}
    \Cref{eq:G-norm-bdl,eq:bbbphi-norm} ensure that
	\begin{equation}
    \label{eq:T0-Ubelief-3}
	\forall t \in [ T_0, \,\, T], \qquad
	\Bigl\Arrowvert \hbb_t \otimes \bphi(a, \bx_t) \Bigr\Arrowvert_{G_{t-1}^{-1}} \leq \frac{\Bigl\Arrowvert \hbb_t \otimes \bphi(a, \bx_t)
    \Bigr\Arrowvert_2}{\sqrt{\lambda}} \leq \frac{1}{\sqrt{\lambda}}
	\qquad \mbox{thus, finally,} \qquad
	\epsilon'_{t,a} \leq \epsilon_{t,a} + 2(T_0 - 1)/\lambda\,.
	\end{equation}
    By a union bound, substituting these bounds in \Cref{eq:csq-lmsimple-1,eq:csq-lmsimple-2} and using the definition of $a_t$
    as the argument of some maximum, we get
    that with probability at least $1-\delta$,
	for all $t \in [T_0,T]$,
	\begin{align*}
	\max_{a \in \cA} \sum_{h \in \cH} \bb_t(h) \, \bphi(a, \bx_t)^{\transp} \btheta^{\star}_{h}
	& \leq \frac{2 (T_0 - 1)}{\lambda} + \max_{a \in \cA} \left\{ \epsilon_{t,a} + \sum_{h \in \cH} \hbb_t(h) \, \bphi(a, \bx_t)^{\transp} \hbtheta_{t-1,h} \right\} \\
    & = \frac{2 (T_0 - 1)}{\lambda} + \epsilon_{t,a_t} + \sum_{h \in \cH} \hbb_t(h) \, \bphi(a_t, \bx_t)^{\transp} \hbtheta_{t-1,h} \\
    \mbox{and} \hspace{3.5cm}
	\sum_{h \in \cH} \hbb_t(h) \, \bphi(a_t, \bx_t)^{\transp} \hbtheta_{t-1,h}
	& \leq \frac{2 (T_0-1)}{\lambda} + \epsilon_{t,a_t} + \sum_{h \in \cH} \bb_t(h) \, \bphi(a_t, \bx_t)^{\transp} \btheta^{\star}_{h}\,\,, \\
    \mbox{thus} \hspace{3.4cm}
    \max_{a \in \cA} \sum_{h \in \cH} \bb_t(h) \, \bphi(a, \bx_t)^{\transp} \btheta^{\star}_{h}
    & - \sum_{h \in \cH} \bb_t(h) \, \bphi(a_t, \bx_t)^{\transp} \btheta^{\star}_{h}
    \leq \frac{4 (T_0-1)}{\lambda} + 2 \epsilon_{t,a_t}\,.
	\end{align*}
    For $t \leq T_0-1$, the difference above, called the instantaneous pseudo-regret, is always bounded by~$2$
    due to the boundedness stated in \Cref{assumption:bound-R-linear}.

    Therefore, the regret~\eqref{eq:def-regret} is bounded, with probability at least $1-\delta$, by
	\begin{align}
		\nonumber
        R_T & =  \sum_{t=1}^{T} \left( \max_{a \in \cA} \sum_{h \in [H]} \bb_t(h) \, \bphi(a, \bx_t)^{\transp} \btheta^{\star}_{h} -
			\sum_{h \in [H]} \bb_t(h) \, \bphi(a_t, \bx_t)^{\transp} \btheta^{\star}_{h} \right)  \leq 2 (T_0 - 1) +
			\sum_{t=T_0}^T \bigl(2\eps_{t, a_t} + 4 (T_0 - 1)/\lambda\bigr) \\
		\label{eq:pointer-2-T0REGR}
        & \leq 2 (T_0 - 1) + \frac{4T\,(T_0 - 1)}{\lambda}
		+  2 \sum_{t=1}^T \Ubelief(t, \delta/2)
		+ 2 \sum_{t=1}^{T} \Bigl\Arrowvert \hbb_t \otimes \bphi(a, \bx_t) \Bigr\Arrowvert_{G_{t-1}^{-1}} \\
        \nonumber
		& \qquad \qquad \qquad \qquad \times \, \Biggl( \frac{1}{\sqrt{\lambda}}\sum_{\tau=1}^{t-1}
		\Ubelief(\tau,\delta/2) + \sqrt{\lambda H} \, C_{\btheta^{\star}}
		+ v_\eta \sqrt{2\ln(2/\delta) + dH \ln\bigl(1 + t/(\lambda d H)\bigr)}\Biggr)\,,
	\end{align}
	where we obtained the second inequality by substituting the expression of $\eps_{t, a_t}$
    and by replacing $T_0$ by $1$ in the summation indices.
    We further bound the expression above by upper bounding each $t$ by $T$ in the sum of three terms in parentheses
    and by using
    \[
    \sum_{t=1}^{T} \Bigl\Arrowvert \hbb_t \otimes \bphi(a, \bx_t) \Bigr\Arrowvert_{G_{t-1}^{-1}}
    \leq \sqrt{2 \, d H T \, \ln\Bigl( 1 + T/ (d H \lambda) \Bigr)}\,,
    \]
    which follows from \Cref{lm:bound-cum-sum-simplfied-bis} applied
    to the vectors $\by_\tau = \hbb_\tau \otimes \bphi(a_\tau, \bx_\tau)$,
	of dimension $dH$ and with Euclidean norm smaller than~$1$
	as indicated in \Cref{eq:bbbphi-norm}.
    We get the following final bound:
    with probability at least $1-\delta$,
    \begin{align}
    \label{eq:regret-bound-final-simplified}
		R_T & \leq 2 (T_0 - 1) + \frac{4T\,(T_0 - 1)}{\lambda}
		+  2 \sum_{t=1}^T \Ubelief(t, \delta/2)
		+ 2 \sqrt{2 \, d H T \, \ln\Bigl( 1 + T/ (d H \lambda) \Bigr)} \\
    \nonumber
		& \qquad \qquad \qquad \qquad \times \, \Biggl( \frac{1}{\sqrt{\lambda}}\sum_{\tau=1}^{T}
		\Ubelief(\tau,\delta/2) + \sqrt{\lambda H} \, C_{\btheta^{\star}}
		+ v_\eta \sqrt{
			2\ln(2/\delta) + dH \ln\bigl(1 + T/(\lambda d H)\bigr)}\Biggr)\,.
	\end{align}
    This upper bound is of order $\sqrt{T}$ up to logarithmic factors
    when no belief estimation is required (i.e., when $\Ubelief$ is null and $T_0=1$,
    and for $\lambda$ given by a constant).
	
	Keeping in mind that $T_0$ is of order $\ln T$ as far as its dependency on $T$
    is concerned, the second part of \Cref{assumption:bound-belief} ensures that
    the upper bound of \Cref{eq:regret-bound-final-simplified} is of order,
    up to poly-logarithmic factors,
    \[
    \tO \Bigl( T/\lambda +
	\sqrt{T} + \sqrt{T} \, \bigl( \sqrt{T/\lambda} + \sqrt{\lambda} \bigr)
	\Bigr) = \tO \bigl( T^{3/4} \bigr)\,,
    \]
    where we exploited the choice $\lambda = T^{1/2}$, which is the optimal
    choice of the form $T^a$ as far as orders of magnitude in $T$
    up to poly-logarithmic factors are concerned.
\end{proof}

For the sake of self-completeness,
we state the so-called elliptic potential lemma, which is extracted from \citet[Section~C]{abbasi2011improved};
we actually even re-prove it here because we extend it later to staged updates in Appendix~\ref{sec:bound-cum-sum}
and base our extension on the classic proof below.

\begin{lemma}
	\label{lm:bound-cum-sum-simplfied-bis}
	Consider vectors $\by_t \in \R^d$ with $\Arrowvert \by_t \Arrowvert_2 \leq 1$, a parameter $\lambda \geq 1$, and the Gram matrices
	$V_0 = \lambda \id{d}$ and
	\[
	V_t = \lambda \id{d} + \sum_{\tau=1}^{t} \by_\tau \by_\tau^{\transp} \quad \mbox{for} \quad t \geq 1 \,.
	\]
	For all integers $T \geq 1$, we have:
	\[
		\sum_{t=1}^{T} \bigl\Arrowvert V_{t-1}^{-1/2} \, \by_{t} \bigr\Arrowvert_2
		\leq \sqrt{2 d T \ln \bigl( 1 + T/(d\lambda) \bigr)}\,.
	\]
\end{lemma}
\begin{proof}
	The classic proof is extracted from \citet[Section~C]{abbasi2011improved}.
	By the Cauchy-Schwarz inequality,
	\begin{equation}
		\label{eq:sum-per-round-raw}
		\sum_{t=1}^T \bigl\Arrowvert V_{t-1}^{-1/2} \, \by_t \bigr\Arrowvert_2
		\leq \sqrt{T \, \sum_{t=1}^T \bigl\Arrowvert V_{t-1}^{-1/2} \, \by_t \bigr\Arrowvert_2^2}
		= \sqrt{T \, \sum_{t=1}^T \by_t^{\transp} V_{t-1}^{-1} \by_t}\,.
	\end{equation}
	Now, for $t \geq 1$, given that $V_{t-1} \succeq \lambda \id{d}$, where $\lambda \geq 1$,
	and given that $\Arrowvert \by_t \Arrowvert_2 \leq 1$ by assumption, we have
	\[
	0 \leq \by_t^{\transp} \, V_{t-1}^{-1} \, \by_t \leq \frac{1}{\lambda} \, \Arrowvert \by_t \Arrowvert_2^2 \leq 1\,.
	\]
	We resort to the inequality $u \leq 2 \ln(1 + u)$ for $u \in [0,1]$,
	and then to \Cref{lm:matrix-determinant-bis}, to get
	\begin{equation}
		\label{eq:reminder-end}
		\sum_{t=1}^T \by_t^{\transp} V_{t-1}^{-1} \by_t
		\leq 2 \sum_{t=1}^T \ln \bigl(1 + \by_t^{\transp} \, V_{t-1}^{-1} \, \by_t\bigr)
		= 2 \sum_{t=1}^T \ln \frac{\det(V_t)}{\det(V_{t-1})}
		= 2 \ln \frac{\det(V_T)}{\det(V_{0})} \,,
	\end{equation}
	where $\det(V_{0}) = \lambda^d$ and an upper bound on $\det(V_T)$ is given by
	\Cref{lm:determinant-trace-inequality-bis}.
	Collecting all inequalities concludes the proof.
\end{proof}

\begin{lemma}[{Matrix determinant lemma, see, e.g., \citealp[Lemma 1.1]{ding2007eigenvalues}}]
	\label{lm:matrix-determinant-bis}
	For a $d \times d$ invertible matrix $A$ and vectors $\by,\,\by' \in \R^d$,
	\[
	\det\bigl(A + \by' \by^{\transp} \bigr) = \bigl(1 + \by^{\transp} A^{-1} \by' \bigr) \, \det(A)\,.
	\]
\end{lemma}

\begin{lemma}[{Determinant--trace inequality, see, e.g., \citealp[Lemma 10]{abbasi2011improved}}]
	\label{lm:determinant-trace-inequality-bis}
	With the notation and assumptions of Lemma~\ref{lm:bound-cum-sum-simplfied-bis},
	\[
	\det(V_t) \leq (\lambda + t/d)^d\,.
	\]
\end{lemma}

\clearpage
\section{Challenges Overcome and Details on the Complex Statistical Dependencies at Stake}
\label{sec:challenge-overcome}
\label{app:other-techn-tool}

The most complex reward model~\eqref{eq:reward-our}
considers a direct dependency on the hidden state
in the reward, through the $\btheta^{\star}_{h_t}$
term. We first explain why and how this
model introduces complex statistical dependencies between
hidden states and actions taken, through the observed rewards.
We then highlight some of the technical challenges
that we faced, and we explain, at a high level, how we overcame them.

\subsection{The Complex Statistical Dependencies Induced by Model~\eqref{eq:reward-our}}

Consider any filtration $\cF_t$ such that $\sigma(\bx_{1:t}) \subseteq \cF_t \subseteq \cFobs_t$.
In model~\eqref{eq:reward-our}, we have, by the assumptions on the noise terms and by the tower rule:
for all $a \in \cA$,
\begin{equation}
\label{eq:core-difficulty}
\E\bigl[ r_t(a) \mid \cF_t \bigr] = \bphi(a, \bx_t)^{\transp} \sum_{h \in [H]} \P(h_t = h \mid \cF_t) \, \btheta^{\star}_h\,.
\end{equation}
This equality extends to $r_t(a_t)$ provided that $a_t$ is $\cF_t$--measurable,
which imposes a first constraint on $\cF_t$. When $a_t$ is not $\cF_t$--measurable,
it is actually difficult to relate
$\E\bigl[ r_t(a_t) \mid \cF_t \bigr]$ to the quantities~\eqref{eq:core-difficulty};
in particular, in general,
\[
\E\bigl[ r_t(a_t) \mid \bx_{1:t} \bigr] \qquad \mbox{and}
\qquad \bphi(a_t, \bx_t)^{\transp} \sum_{h \in [H]} \P(h_t = h \mid \bx_{1:t}) \, \btheta^{\star}_h
= \bphi(a_t, \bx_t)^{\transp} \sum_{h \in [H]} \bb_t(h) \, \btheta^{\star}_h
\]
are different quantities.

To mimic the reduction to linear contextual bandits that was possible
for the simplified reward model~\eqref{eq:nelson-model-gener},
the posterior probabilities $\P(h_t = h \mid \cF_t)$ read in~\eqref{eq:core-difficulty} should be estimated.
This is our second constraint.

There is a tension between the two constraints
and the two extreme candidates $\sigma(\bx_{1:t})$ and $\cFobs_t$.

First, the natural filtration to get measurability of $a_t$ is $\cFobs_t$; indeed, the actions picked $a_t$
depend on the rewards obtained in the past, not only on the contexts.
However, note that
\[
\P(h_t = h \mid \bx_{1:t},\,a_t)
\quad \mbox{and} \quad
\bb_t(h) = \P(h_t = h \mid \bx_{1:t})
\]
are different in general: the action $a_t$
depends on past rewards, and therefore depends on $h_{t-1}$,
which entails some dependency between $a_t$ and $h_t$.
More generally,
\[
\P(h_t = h \mid \cF_t)
\quad \mbox{and} \quad
\bb_t(h) = \P(h_t = h \mid \bx_{1:t})
\]
are different in general when $a_t$ is $\cF_t$--measurable.

Second, the natural filtration to estimate
easily the beliefs $\P(h_t = h \mid \cF_t)$ is $\sigma(\bx_{1:t})$:
we discard reward and may then open the toolbox of HMM estimation.
Also, the pseudo-regret is formulated in terms of beliefs, i.e.,
posterior distributions based on $\sigma(\bx_{1:t})$.

Our solution basically consists of introducing a carefully constructed filtration in the middle
of these two extremes; we denote it by $\cU_t$ and try now to get
an idea of some nice properties we expect from it.

\subsection{Caveats Encountered when Estimating Beliefs}
\label{sec:beliefs-estim-intro}

\textbf{Difficulties.} As explained above, it is handy that the action $a_t$ be $\cU_t$--measurable. Now, to make useful decisions,
the filtration $\cU_t$ cannot be reduced to just contexts and thus
needs to include at least some (quantities based on the) rewards. However, two issues arise.

First, posterior distributions with respect to $\cU_t$
are then difficult to compute even with the knowledge of the underlying HMM parameters, and thus are even
more difficult to estimate.
Indeed, beliefs are typically computed based on Bayes' update rules (see Appendix~\ref{sec:belief-estimation-error}),
and such rules work should then take rewards into account; this is what \citet{RegimeBandits-2021}
performed in the case of Bernoulli rewards, but finding a general and computationally efficient solution beyond
that case seems challenging.

Second, posterior distributions with respect to $\cU_t$ should be close enough to the beliefs
(i.e., from posterior distributions with respect to $\bx_{1:t}$).

\textbf{Ideas.} The first idea is estimate posteriori distributions independently, only based on contexts, and discard any
additional information coming from rewards;
the loss in efficiency should be acceptable given the definition~\eqref{eq:def-regret} of the pseudo-regret,
which relies on beliefs $\bb_t$, i.e., on quantities only based on the contexts.
It turns out that such belief estimation only based on contexts is standard
(see Appendix~\ref{sec:belief-estimation-error}): it suffices to estimate the HMM parameters,
for which standard procedures exist, and estimation errors on the beliefs
can be obtained as functions of the estimation errors of the HMM parameters.

The second, and main, idea is to design the filtration $\cU_t$ so that
\[
\bb_t(h) = \P(h_t = h \mid \bx_{1:t})
\qquad \mbox{and} \qquad
\overline{\bb}_t(h) = \P(h_t = h \mid \cU_t)
\]
are close, while maintaining the $\cU_t$--measurability of $a_t$,
i.e., including enough information on past rewards in $\cU_t$.

\textbf{Solution.}
The solution is essentially based on a typical property of HMMs:
that they forget exponentially fast the initial state. Put differently,
after a sufficient time, HMMs not initialized with the same distributions
over states are almost equal in distributions. This property suggests to
work in stages of sufficient length: we update only periodically the $\cU_t$ to include
a new sufficient statistic based on the rewards of the stage just passed.

These sufficient statistics consist of estimates of the parameters
$(\btheta^\star_h)_{h \in [H]}$ of the reward model. The actions $a_t$
are picked based on the latest such estimates available (and on the estimated beliefs),
in an optimistic fashion, through the consideration of an upper confidence bound.

\subsection{Other Technical Tool from \citet{nelson2022ContextLatent}}
\label{sec:other-tool}

A final ingredient in our solution, discussed next, is to leverage an approach by \citet{nelson2022ContextLatent}
for proving deviation inequalities in cases where the LinUCB approach by \citet{abbasi2011improved}
is not applicable. This approach relies
on an $\L^2$--Markov inequality (suited for random variables with
second-order moments, that are not necessarily sub-Gaussian).

Note that this approach was somewhat unnecessary in the setting of \citet{nelson2022ContextLatent}:
provided that they strengthen a bit their assumption on the noise
terms from a second-order moment assumption (\Cref{assumption:idio-noise})
to a well-accepted sub-Gaussian (\Cref{ass:SGnoise}),
\citet{nelson2022ContextLatent} could have resorted to the classic LinUCB tools,
with absolutely \emph{no need} of HMM forgetting properties (like \Cref{assumption:strong-mixing-hmm}), to get their
regret bound: we showed this in Appendix~\ref{appendix:simplified-analysis}.

However, as we explain in this section, this approach is particularly convenient in the case of our
more complex reward model~\eqref{eq:reward-our}. To do so, we
first introduce some stylized quantities to be controlled in the proofs.

\paragraph{Stylized quantities considered.}
The proofs---see, in particular, Appendix~\ref{sec:proof:lm:true-vs-belief} for the most complex reward model~\eqref{eq:reward-our}
and \Cref{lm:est-simple} for the simplified version~\eqref{eq:nelson-model-gener}
by \citet{nelson2022ContextLatent}---indicate
that they key quantities to be controlled are of the stylized form
\[
\Delta_T = \sum_{t=1}^T \left( r_t(a_t) - \bphi(a, \bx_t)^{\transp} \sum_{h \in [H]} \bb_t(h) \, \btheta^{\star}_h \right).
\]
The exact quantities appearing in the proofs of the regret
bounds are slightly more complicated (in particular, they are vector-valued) but we capture here
the essence of the arguments.

\paragraph{Simplified reward model.}
By definition of the simplified reward model~\eqref{eq:nelson-model-gener},
\[
\Delta_T = \sum_{t=1}^T \eta'_t(a_t)\,,
\]
where the $\eta'_t(a_t)$ are martingale increments, e.g., with respect to $\cFobsp_{t+1}$.
Indeed, on the one hand, $\eta'_t(a_t)$ is $\cFobsp_{t+1}$--measurable
as the difference between $r'_t(a_t)$,
which is one of the variables generating $\cFobsp_{t+1}$, and
some quantity measurable with respect to $\sigma(\bx_{1:t})$ and $a_t$,
where $a_t$ itself is $\cFobsp_{t}$--measurable.
On the other hand, the assumption that the noise term is independent from
the present and the past
entails that $\E\bigl[ \eta'_t(a) \mid \cFobsp_t \bigr] = 0$ for all $a \in \cA$,
so that, using again that $a_t$ is $\cFobsp_{t}$--measurable, we also have
\[
\E\bigl[ \eta'_t(a_t) \mid \cFobsp_t \bigr] = 0\,.
\]

Because the $\eta'_t(a_t)$ form bounded martingale increments,
with probability at least $1-\delta$,
the martingale $\Delta_T$
is smaller than something of the order of $\sqrt{T \ln(1/\delta)}$,
e.g., by the Hoeffding-Azuma inequality.

\paragraph{Most complex reward model: issue.}
The main issue in the most complex reward model~\eqref{eq:reward-our}
is that we do not deal with martingale increments,
due, in particular, to the direct dependencies of rewards on hidden states.
More precisely, we have, by the definition~\eqref{eq:reward-our}, that
\[
\Delta_T = \sum_{t=1}^T \bphi(a_t, \bx_t)^{\transp} \! \left( \btheta^{\star}_{h_t} - \sum_{h \in [H]} \bb_t(h) \, \btheta^{\star}_h
\right) + \sum_{t=1}^T \eta_t(a_t)\,.
\]
An argument similar to above shows that the $\eta_t(a_t)$ are martingale increments
with respect to $\cFall_{t+1}$, thus their sum is controlled.

The remaining question is thus to control
\[
\sum_{t=1}^T d_t\,, \qquad \mbox{where} \qquad
d_t = \bphi(a_t, \bx_t)^{\transp} \! \left( \btheta^{\star}_{h_t} -
\sum_{h \in [H]} \bb_t(h) \, \btheta^{\star}_h \right).
\]
However, the $d_t$ do not form martingale increments, nor are sufficiently close
to martingale increments.
The issue mostly lies in guaranteeing that $d_t$ is $\cF_{t+1}$--measurable;
it is straightforward to find filtrations such that $\E\bigl[ d_t \mid \cF_t \bigr]$ is close to $0$,
e.g., $\cF_t = \cU_t$:
\[
\E\bigl[ d_t \mid \cU_t \bigr] =
\bphi(a_t, \bx_t)^{\transp} \! \left( \sum_{h \in [H]} \bigl(
\overline{\bb}_t(h) - \bb_t(h) \bigr) \btheta^{\star}_h \right),
\]
which is small as all terms $\overline{\bb}_t(h) - \bb_t(h)$ are small,
see Appendix~\ref{sec:beliefs-estim-intro}.

Any natural filtration $\cF_t$ we would consider (and $\cU_t$ in particular)
includes contexts $\bx_{1:t}$ and is such that $a_t$ is $\cF_t$--measurable.
For such filtration, the requirement of $\cF_{t+1}$--measurability of $d_t$
entails that $\btheta^{\star}_{h_t}$ should be $\cF_{t+1}$--measurable.
Say, for simplicity, that $h_t$ is $\cF_{t+1}$--measurable.

With the same arguments as above,
the constraint $\E\bigl[ d_t \mid \cF_t \bigr]$ close to $0$ would then impose
that
\[
\bb_t(h) = \P(h_t = h \mid \bx_{1:t}) \qquad \mbox{and} \qquad
\P(h_t = h \mid \cF_t)
\]
should be close, but $\P(h_t = h \mid \cF_t)$
corresponds to a quantity of the form
$\P(h_t = h \mid h_{t-1},\,\bx_t)$
and is likely to differ much from $\bb_t(h)$.
Indeed, the former is mostly a function of $h_{t-1}$,
while the latter depends only weakly on the underlying states
due to the HMM forgetting properties.

\paragraph{Most complex reward model: solution.}
It turns out that we do not need that the $d_t$
are sufficiently close to martingale increments
to control their sum: via Markov's inequality,
it suffices that their sum is small in $\L^2$--norm,
e.g., of the order of $T^a$ for $a < 1$.
Then, with high probability, the sum itself is small:
with probability at least $1-\delta$
\[
\sum_{t=1}^T d_t \leq \sqrt{\frac{\E \Bigl[ \bigl( d_1^2 + \ldots + d_T^2 \bigr) \Bigr]}{\delta}}
\leq \sqrt{\frac{\O(T^a)}{\delta}}\,.
\]

\citet{nelson2022ContextLatent} illustrate this
for the sum of noise terms $\eta'_t(a)$ in
their simplified model,
when relaxing the sub-Gaussian noise \Cref{ass:SGnoise}
to \Cref{assumption:idio-noise} on
bounded conditional second-order moments.
This relaxation seems unimportant
and we provide a simple and straightforward-to-prove
regret bound under \Cref{ass:SGnoise}
(see Appendix~\ref{appendix:simplified-analysis}).
Yet, the analysis that \citet{nelson2022ContextLatent} developed,
while not fully useful for them, is powerful and
may be mimicked (and extended) to handle the sum of $d_t$.

Indeed, we write
\[
\E \! \left[ \left( \sum_{t=1}^T d_t^2 \right) \right]
= \underbrace{\sum_{t=1}^T \E\bigl[d_t^2\bigr]}_{\mbox{\small of order $T$}}
+ 2 \sum_{1 \leq t < t' \leq T} \E\bigl[d_t d_{t'}\bigr]\,,
\]
where we used \Cref{assumption:bound-R-linear}, i.e., the boundedness of $d_t$, to see that the sum of
expected square terms is of order $T$.
We get, by the tower rule
and by the fact that $\bb_t(h)$ and $\bb_{t'}(h)$ are $\cU_{t'}$--measurable:
\[
\E\bigl[d_t d_{t'}\bigr] = \E\Bigl[ \E\bigl[d_t d_{t'} \mid \cU_{t'} \bigr] \Bigr]
= \E\!\left[ \bphi(a_t, \bx_t)^{\transp} \btheta^{\star}_{h_t}
\bphi(a_{t'}, \bx_{t'})^{\transp} \! \biggl( \btheta^{\star}_{h_{t'}} -
\sum_{h \in [H]} \overline{\bb}_{t'}(h) \, \btheta^{\star}_h \biggr)
\right] + \psi\bigl(\overline{\bb}_t - \bb_{t'}\bigr)\,,
\]
where $\psi\bigl(\overline{\bb}_t - \bb_{t'}\bigr)$ is a term
bounded by something of the order of $\Arrowvert \overline{\bb}_t - \bb_{t'} \Arrowvert_1$
and is therefore small, by the HMM forgetting property.
The other term in the rewriting above involves differences between
\[
\P\bigl(h_\tau = h \an h_{\tau'}=h' \mid \cU_{\tau'} \bigr)
\qquad \mbox{and} \qquad \overline{\bb}_t\,,
\]
which, again, are small when $t$ and $t'$ are sufficiently
separated due to the HMM forgetting property.

\paragraph{Formal proofs.}
Formal proofs of all the vague statements of this appendix
will be provided in the course of the proof of \Cref{theorem:total-regret-bound},
in Appendix~\ref{appendix:complex-analysis}.

\clearpage
\section{Belief Estimation}
\label{sec:belief-estimation-error}

The aim of this section is to detail why the following assumption made
on the belief estimation subroutine is reasonable.

\boundbelief*

More precisely, we show that this assumption holds at least for the large class of HMMs
satisfying \Cref{assumption:regularity-hmm} below.
To state the latter, we let $\bE$ denote the emission matrix, indexed by $\cX \times [H]$ and obtained by concatenating the emission
distributions $\nu_h$ seen as $X$--dimensional column vectors as $h \in [H]$.
We will assume, among others, that $\bE$ has full column rank: this imposes, in particular,
that $H$ is smaller than the cardinality of $\cX$. Recall the notion of singular value:
$\sigma$ is a singular value of $\bE$ if $\sigma^2$ is an eigenvalue of the square matrix $\bE^{\transp}\bE$.

\regularityhmm*

To that end, we consider a combination of a so-called spectral method
(see its algorithmic statement in Appendix~\ref{sec:spectral}
and see references after the statement of \Cref{lm:beliefestimationerror})
to estimate the HMM parameters $\bM$ and $\bE$, i.e., the $\nu_h$ for $h \in [H]$,
together with a Bayes' update rule to deduce estimated belief from these estimated parameters.
We also add an alignment step (see Appendix~\ref{sec:align})
to keep track of the hidden states.

\paragraph{The Bayes' update rule.}
More precisely, for $t \geq 1$, we denote by $\hbM_{t}$ and $\hnu_{t,h}$
the estimates of $\bM$ and of the $\nu_h$ obtained based on the contexts
$\bx_1,\ldots,\bx_t$; we also consider some estimation $\hat{\bpi}$
of the distribution of $h_1$, for instance, the uniform
distribution over $\cH$.
The Bayes' update rule then works as follows.
For $\tau = 1$, we compute, for all $h \in \cH$:
\begin{equation}
	\label{eq:Bayes-est-1}
	\hbb_{t, 1}(h) = \frac{\hnu_{t,h}(\bx_1) \, \hat{\bpi}(h)}{\displaystyle{\sum_{h^{''} \in \cH} \hnu_{t,h^{''}}(\bx_1)\, \hat{\bpi}(h'')}}\,.
\end{equation}
We then compute successively, for all $2 \leq \tau \leq t$, for all $h \in \cH$,
\begin{equation}
\label{eq:Bayes-est-2}
\hbb_{t,\tau}(h) = \frac{\hnu_{t,h}(\bx_{\tau}) \displaystyle{\sum_{h' \in \cH} \hbb_{t,\tau-1}(h') \, \hM_{t, h', h}}}{\displaystyle{\sum_{h^{''} \in \cH} \hnu_{t,h^{''}}(\bx_{\tau}) \sum_{h' \in \cH} \hbb_{t,\tau-1}(h')\, \hM_{t, h', h''}}}\,.
\end{equation}
We finally issue
\begin{equation}
\label{eq:Bayes-est-3}
\hbb_{t} = \hbb_{t,t}\,.
\end{equation}
We call the successive updates above the Bayes' update rule.
When $\bM$, $\bpi$, and the $\nu_h$ are used instead of their estimates
in the formulas above, we obtain the true beliefs $\bb_t$.

We detail in the rest of this appendix how to obtain
the following result, which shows that \Cref{assumption:bound-belief} is reasonable.

\beliefestimationerror*

\paragraph{References for the spectral method.}
The HMM parameter estimation via spectral method is a standard procedure commonly used for bandits with latent states or for partially observable Markov decision processes [POMDP] where underlying states follow an HMM. It was proposed by \citet{spectral-2012} and further developed by \citet{anandkumar-2012-spectral-hmm} and \citet{anandkumar-2014-tensor-mixture}. For instance, \citet{RegimeBandits-2021} and \citet{spectral-pomdps-2016} apply this the method proposed by \citet[Section~4.2]{anandkumar-2012-spectral-hmm}, in combination with the power iteration method from \citealp{anandkumar-2014-tensor-mixture}. We restate the algorithm of \citet[Section~4.2]{anandkumar-2012-spectral-hmm} in Appendix~\ref{sec:spectral}.

\subsection{Proof of \Cref{lm:beliefestimationerror}: Belief-Estimation Error}
\label{sec:proof-beliefestimationerror}

We explain how the belief-estimation error bound from \Cref{lm:beliefestimationerror} follows from the application
of two known results on HMM parameter estimation (one for the estimation of the parameters
themselves, one for the guarantees induced on the estimation of the beliefs), together with two simple
additions:
an alignment step to ensure coherence of the labeling of hidden states and a
twist to get a fully known belief error function.

First, as detailed in Appendix~\ref{sec:details-how-HMM},
\citet[Theorem 3]{spectral-pomdps-2016} and \citet[Proposition~1, Appendix~B]{RegimeBandits-2021}
offer some estimator error guarantees, which can be instantiated under \Cref{assumption:regularity-hmm}
as follows (keeping the notation of the second reference), with constants
\begin{equation*}
C_1 = \frac{21}{\sigma^2} C_3
\quad \mbox{and} \quad
C_2 = \frac{4}{\sigma} \left( \sqrt{H} + \frac{21 H}{\sigma^2} \right) C_3\,,
\quad \mbox{where} \quad
C_3 = \frac{16}{\epsilon_{\bM}^{3/2}} \left( 1 + \frac{12}{\epsilon_{\bM}^2 \sigma^3} +  \frac{256}{\epsilon_{\bM}^2 \sigma^2}\right).
\end{equation*}
Based on \citet[Equation~28]{anandkumar-2014-tensor-mixture}, we also define
\[
C_0 = \min\left\{ (56\cdot9\cdot 102)^{-1}, (100\cdot 168)^{-1}, \Delta' \right\},
\]
where $\Delta' > 0$ is a numerical constant defined in \citet[Lemma~B.5 with $\Delta = 1/50$]{anandkumar-2014-tensor-mixture},
though not in closed-form.
Compared to the mentioned references,
we rather consider the Frobenius norm instead of the spectral norm,
which introduces an additional $\sqrt{H}$ factor
in the estimation bound for $\bM$.

\begin{proposition}[{Instantiation of \citealp[Proposition 1]{RegimeBandits-2021}, itself based on \citealp[Theorem 3]{spectral-pomdps-2016}}]
\label{proposition:bound-hmm-parameters}
Under \Cref{assumption:regularity-hmm}, the threshold
\begin{equation}
\label{eq:closedform-tildeT}
\widetilde{T}_{\cB,\bM,\nu} = 2 \left(\frac{12}{\epsilon_{\bM}^2 \sigma^2}\right)^{\!\! 2} \, \bigl(\ln\bigl(2X\bigr) + 1 \bigr)\, \max\left\{ \frac{16 H^{1/3}}{C_0^{2/3} \epsilon_{\bM}^{1/3}}, \,\,
\frac{3H}{C_0^{2} \epsilon_{\bM}\sigma^2}, \,\, 1 \right\}
\end{equation}
is such that
for all $t \geq \widetilde{T}_{\cB,\bM,\nu}\bigl(1 + \ln(1/\delta)\bigr)$,
with probability $1 - \delta$,
the estimates $\tilde{\bM}_t$ and $\tilde{\nu}_t$ from the spectral method can be well computed and satisfy, up to some permutation $\rho$ of $[H]$,
\[
\bigl\Arrowvert \tilde{\bM}_t^{(\rho)} - \bM \bigr\Arrowvert_2 \leq C_2 \sqrt{\frac{2H\ln(6X/\delta)}{t}}
\qquad\quad \mbox{and} \qquad\quad
\forall h \in \cH, \qquad \Arrowvert \tilde{\nu}_{t,\rho(h)} - \nu_h \Arrowvert_2 \leq C_1 \sqrt{\frac{2\ln(6X / \delta)}{t}}\,,
\]
where $\tilde{\bM}_t^{(\rho)} = \bigl[ \tilde{M}_{t,\rho(h),\rho(h')} \bigr]_{h,h' \in [H]}$.
\end{proposition}

The result above is ``up to some permutation $\rho$ of $[H]$'': this underlines that
the spectral method does absolutely not guarantee that what was called ``state~$h$'' in round $t$ will
correspond to the same `state~$h$'' in round $t+1$.
In terms of beliefs, this means that the belief function
obtained from the estimates of \Cref{proposition:bound-hmm-parameters}
are good up to the labeling of the hidden states.
Since the end result is about the $\ell_1$--error
between the estimated beliefs and the true beliefs,
the ordering of hidden labels does not matter as long
as that ordering is constant over time.
This is what the alignment procedure described in Appendix~\ref{sec:align}
will guarantee.

For now, we move to controlling the $\ell_1$--error
between the estimated beliefs and the true beliefs,
which corresponds to the second statement
of \Cref{lm:beliefestimationerror}
and is independent of the ordering of hidden labels,
as it corresponds to some global evaluation.

\citet{de2017consistent} developed the following bound linking
the estimation errors of the HMM parameters to the estimation error
on the beliefs (again, the result holds up to permutations).

\begin{proposition}[{\citealp[Proposition 2.1]{de2017consistent}, see also \citealp[Proposition~3]{RegimeBandits-2021}}]
\label{proposition:bound-estimated-belief}
Under \Cref{assumption:regularity-hmm}, letting
\[
L_0 = 4 \, \frac{1 - \epsilon_{\bM}}{\epsilon_{\bM}^2}\,,
\qquad
L_1 = 4 \left(\frac{1 - \epsilon_{\bM}}{\epsilon_{\bM}}\right)^{\!\! 2} \frac{1}{e_{\nu, \min}}\,,
\qquad
L_2 = 4 \, \frac{(1 - \epsilon_{\bM})^2}{\epsilon_{\bM}^3}\,,
\]
the beliefs $\tilde{\bb}_t$ obtained from (a suitable permutation of) the estimates $\tilde{\bpi}$, and $\tilde\bM_{t}$ and $\tilde\nu_{t,h}$
through the Bayes' rule~\eqref{eq:Bayes-est-1}--\eqref{eq:Bayes-est-3} satisfy
\[
\bigl\Arrowvert \tilde{\bb}_t - \bb_t \bigr\Arrowvert_1 \leq
L_0 \left(1 -  \frac{\epsilon_{\bM}}{1-\epsilon_{\bM}}\right)^{\!\! t-1}\bigl\Arrowvert \tilde{\bpi}-\bpi \bigr\Arrowvert_2
+ L_1 \sum_{h \in \cH} \bigl\Arrowvert \tilde\nu_{t,h} - \nu_h \bigr\Arrowvert_1
+ L_2 \bigl\Arrowvert \tilde\bM_t - \bM \bigr\Arrowvert_2 \,.
\]
\end{proposition}

We now
combine \Cref{proposition:bound-hmm-parameters,proposition:bound-estimated-belief}
and perform some simple upper boundings, where the second follows
from the Cauchy-Schwarz inequality:
\[
\bigl\Arrowvert \bpi - \tilde{\bpi} \bigr\Arrowvert_2
= \sqrt{\sum_{h \in \cH} \smash{\bigl( \underbrace{\bpi(h) - \tilde{\bpi}(h)}_{\in [-1,1]} \bigr)^2}}
\leq \sqrt{\bigl\Arrowvert \bpi - \tilde{\bpi} \bigr\Arrowvert_1} \leq \sqrt{2}
\qquad \mbox{and} \qquad
\bigl\Arrowvert \tilde\nu_{t,h} - \nu_h \bigr\Arrowvert_1
\leq \sqrt{X} \,\, \bigl\Arrowvert \tilde\nu_{t,h} - \nu_h \bigr\Arrowvert_2\,.
\]
We also perform union bounds and use
\begin{equation}
\label{eq:UB-ref-PropZhou}
\frac{\delta}{t(t+1)} \qquad \mbox{instead of} \qquad \delta
\end{equation}
at each round $t \geq 1$:
this is to ensure that the result of \Cref{proposition:bound-hmm-parameters}
holds simultaneously for all $t \geq \widetilde{T}_{\cB,\bM,\nu}$
with probability at least $1-\delta$.

We get, from \Cref{proposition:bound-hmm-parameters,proposition:bound-estimated-belief} that
with probability $1 - \delta$,
\begin{multline}
\label{eq:Ubelief-almost}
\forall t \geq \widetilde{T}_{\cB,\bM,\nu}\Bigl(1 + \ln\bigl(t(t+1)/\delta\bigr)\Bigr)\,, \qquad \mbox{up to identifying
a correct permutation}, \\
\bigl\Arrowvert \tilde{\bb}_t - \bb_t \bigr\Arrowvert_1 \leq
\sqrt{2} L_0 \left(1 -  \frac{\epsilon_{\bM}}{1-\epsilon_{\bM}}\right)^{\!\! t-1}
+ \bigl( L_1 C_1 H \sqrt{X} + L_2 C_2 \sqrt{H} \bigr) \sqrt{\frac{2\ln\bigl(6Xt(t+1)/\delta\bigr)}{t}}\,.
\end{multline}
The right-hand side cannot be our $\Ubelief$ function, as it depends on unknown quantities $\epsilon_{\bM},C_1,C_2,L_0,L_1,L_2$ (the latter
depend on the unknown HMM parameters).
We do not follow the mitigations alluded at in \citet[Section~3.3]{RegimeBandits-2021} or \citet[Remark~3]{spectral-pomdps-2016},
consisting of estimating these quantities (this looks as difficult as estimating the HMM parameters)
or replacing them by some hyperparameters tuned by hand; we rather
bound them as functions of $t$ and increase the threshold $\widetilde{T}_{\cB,\bM,\nu}$ to
compensate for that.

We note that
\begin{multline*}
\left( 1 - \frac{\epsilon_{\bM}}{1-\epsilon_{\bM}} \right)^{\!\! t-1} \leq \e^{-\sqrt{t-1}}
\qquad \mbox{as soon as} \qquad
\epsilon'_{\bM}(t-1) \leq -\sqrt{t-1}\,, \\ \mbox{i.e.,} \qquad
t \geq 1 + 1/\bigl(\epsilon'_{\bM}\bigr)^2\,, \qquad
\mbox{where} \qquad
\epsilon'_{\bM} = \ln \! \left( 1 - \frac{\epsilon_{\bM}}{1-\epsilon_{\bM}} \right).
\end{multline*}
Thus, we let
\begin{align}
\label{eq:closedformT}
T_{\cB,\bM,\nu} = \max \biggl\{ 6 \, \widetilde{T}_{\cB,\bM,\nu} \Bigl(1 + \ln\bigl(\widetilde{T}_{\cB,\bM,\nu} \bigr)\Bigr)\,, \ \  \exp(L_1 C_1 + L_2 C_2)\,, \ \  & \exp\bigl(\sqrt{2} L_0\bigr)\,, \\
\nonumber
& 1 + 1/\bigl(\epsilon'_{\bM}\bigr)^2\,, \ \  6 \, T_\Delta \bigl(1 + \ln(T_\Delta )\bigr) \biggr\}\,,
\end{align}
where we recall that $\widetilde{T}_{\cB,\bM,\nu}$ is defined in \Cref{eq:closedform-tildeT}
and where $T_\Delta$ is defined in \Cref{eq:Tdelta} below. Then,
for $t \geq T_{\cB,\bM,\nu} \bigl(1 + \ln(1/\delta) \bigr)$, we have
\[
L_1 C_1 + L_2 C_2 \leq \ln\bigl(T_{\cB,\bM,\nu}\bigr) \leq \ln \Bigl(t/\bigl(1 + \ln(1/\delta) \bigr) \Bigr) \leq \ln(t), \qquad
\sqrt{2} L_0  \leq \ln(t)\,, \qquad
\left( 1 - \frac{\epsilon_{\bM}}{1-\epsilon_{\bM}} \right)^{\!\! t-1} \leq \e^{-\sqrt{t-1}}\,,
\]
so that the right-hand side of \Cref{eq:Ubelief-almost} is indeed smaller
than the quantity $\Ubelief(t, \delta)$ defined in \Cref{lm:beliefestimationerror}
for these $t \geq T_{\cB,\bM,\nu} \bigl(1 + \ln(1/\delta) \bigr)$.
We also need to make sure that these $t$ satisfy the condition of
\Cref{eq:Ubelief-almost}: this is the case as (see proof right below)
\begin{equation}
\label{eq:Tttgeq}
t \geq T_{\cB,\bM,\nu}\bigl(1 + \ln(1/\delta)\bigr)
\qquad \mbox{entails} \qquad
t \geq \widetilde{T}_{\cB,\bM,\nu}\Bigl(1 + \ln\bigl(t(t+1)/\delta\bigr)\Bigr)\,.
\end{equation}
This concludes the proof of the belief-estimation error part of the lemma
up to identifying the suitable permutations, a topic which we
discuss below in Appendix~\ref{sec:align}, and up to proving~\eqref{eq:Tttgeq},
which do next.

\paragraph{Proof of \Cref{eq:Tttgeq}.}
From the assumption $t \geq T_{\cB,\bM,\nu}\bigl(1 + \ln(1/\delta)\bigr)$ and from
the definition of $T_{\cB,\bM,\nu}$, which guarantees that
$T_{\cB,\bM,\nu} \geq 6 \, \widetilde{T}_{\cB,\bM,\nu} \bigl(1 + \ln(\widetilde{T}_{\cB,\bM,\nu})\bigr)$, we get
\begin{equation}
	\label{eq:2lnt+0}
	\frac{t}{\widetilde{T}_{\cB,\bM,\nu}} \geq 6 \Bigl(1 + \ln\bigl(\widetilde{T}_{\cB,\bM,\nu} \bigr)\Bigr)
	\bigl(1 + \ln(1/\delta)\bigr) \geq 6 \bigl(1 + \ln(1/\delta)\bigr)\,.
\end{equation}
From the intermediate inequality above, we also show later that
\begin{equation}
\label{eq:2lnt+1}
 \frac{t}{\widetilde{T}_{\cB,\bM,\nu}} \geq 3 \ln(t) \geq 2 \ln(t+1)\,,
\end{equation}
where the second inequality holds because $\ln(t+1) \leq 1.5\,\ln(t)$ for $t \geq 6$,
a condition that is satisfied in particular here.
The conclusion then follows from combining the three bounds established
above:
\[
	1 + \ln\biggl(\frac{t(t+1)}{\delta}\biggr) = \bigl( 1 + \ln(1/\delta) \bigr) + \ln(t) + \ln(t+1) \leq
\left( \frac{1}{6} + \frac{1}{3} + \frac{1}{2} \right) \frac{t}{\widetilde{T}_{\cB,\bM,\nu}} = \frac{t}{\widetilde{T}_{\cB,\bM,\nu}}\,.
\]
It only remain to prove the intermediate inequality of~\Cref{eq:2lnt+1}.

For the first inequality below, we use that $u \mapsto u/3 - \ln(u)$ is increasing for $u \geq 3$
and apply this property to the left inequality of \Cref{eq:2lnt+0},
and for the second inequality below, we use that $6 < \exp(2)$ and $1 + \ln(x) \leq x$ for $x > 0$:
\[
\frac{t}{3 \widetilde{T}_{\cB,\bM,\nu}} - \ln\biggl( \frac{t}{\widetilde{T}_{\cB,\bM,\nu}} \biggr) \geq
\frac{ 6 \Bigl(1 + \ln\bigl(\widetilde{T}_{\cB,\bM,\nu} \bigr)\Bigr)}{3} -
\ln\biggl( \underbrace{6 \Bigl(1 + \ln\bigl(\widetilde{T}_{\cB,\bM,\nu} \bigr)\Bigr)}_{\leq e^2 \widetilde{T}_{\cB,\bM,\nu}} \biggr) \geq \ln\bigl(\widetilde{T}_{\cB,\bM,\nu}\bigr)\,,
\]
which rewrites as
\[
\ln(t) = \ln\bigl(\widetilde{T}_{\cB,\bM,\nu}\bigr) + \ln\biggl(\frac{t}{ \widetilde{T}_{\cB,\bM,\nu}}\biggr)
\leq \frac{t}{3 \widetilde{T}_{\cB,\bM,\nu}}\,,
\]
which is exactly the intermediate inequality of \Cref{eq:2lnt+1}.

\subsection{Proof of \Cref{lm:beliefestimationerror}: Coherence Statement}
\label{sec:align}

In this section, we introduce an alignment step to ensure that, after a sufficiently large number of rounds, the latent-state labels of the estimated HMM parameters are consistent over time and thus can be mapped with some constant ordering of states.
This also implies that the corresponding estimated belief vectors are expressed in one common latent state coordinate system,
so that the words ``a suitable permutation of'' are not required anymore in \Cref{proposition:bound-estimated-belief}.

What follows was already alluded at, but not described in this level detail,
by \citet[Appendix~C, proof of Theorem~3, step~3]{spectral-pomdps-2016}.
The latter reference raises the issue that
``the columns of estimated matrices are up to different permutations over states, i.e., these
matrices have different columns ordering'', and suggests an alignment procedure.

\paragraph{Algorithmic statement.}
More formally, let $\mathfrak{S}$ denote the set of all permutations $\rho$ of $[H]$.
At step $t+1$, when obtaining the estimates
$\tilde\bM_{t+1}$ and $\tilde\nu_{t+1,h}$ considered in
\Cref{proposition:bound-hmm-parameters}, we transform
them into the estimates $\hat\bM_{t+1}$ and $\hat\nu_{t+1,h}$ by picking
a permutation $\rho_{t+1}$ and considering
\[
\forall h \in [H], \quad \hat\nu_{t+1,h} = \tilde\nu_{t+1,\rho_{t+1}(h)}
\qquad \mbox{and} \qquad
\hat\bM_{t+1} = \bigl[ \tilde\M_{t+1,\rho_{t+1}(h),\rho_{t+1}(h')} \bigr]_{h,h' \in [H]}\,,
\]
where the permutation is picked as $\displaystyle{\rho_{t+1} \in \argmin_{\pi \in \mathfrak{S}} \ \max_{h \in [H]} \bigl\Arrowvert \hat\nu_{t,h} - \tilde{\nu}_{t+1,\pi(h)} \bigr\Arrowvert_2}$.

At step $t=1$, we leave estimators unchanged, i.e., pick $\rho_1$ given by the identity.

\paragraph{Analysis.}
We consider the same union bound
as the one performed in \Cref{eq:UB-ref-PropZhou},
so that, in particular,
\begin{align*}
\P \Biggl( \forall t \geq \widetilde{T}_{\cB,\bM,\nu}\Bigl(1 & + \ln\bigl(t(t+1)/\delta\bigr)\Bigr)\,, \quad \exists \rho'_t \in \mathfrak{S}
\quad \mbox{s.t.} \\
& \smash{\forall h \in \cH, \qquad \Arrowvert \tilde{\nu}_{t,\rho'_t(h)} - \nu_h \Arrowvert_2
\leq C_1 \sqrt{\frac{2\ln\bigl(6X t(t+1)/ \delta\bigr)}{t}} \Biggr) \geq 1 - \delta\,.}
\end{align*}
Denote by
\[
\Delta = \min_{h \ne h'} \Arrowvert \nu_h - \nu_{h'} \Arrowvert_2 \,;
\]
this quantity is positive as, by \Cref{assumption:regularity-hmm},
$\bE$ is full rank and thus,
the emission distributions $\nu_h$ are, in particular, all different.
Now, consider a time $T_\Delta$ such that
\[
\forall t \geq T_\Delta \Bigl(1 + \ln\bigl(t(t+1)/\delta\bigr)\Bigr), \quad\qquad
C_1 \sqrt{\frac{2\ln\bigl(6X t(t+1)/ \delta\bigr)}{t}}  < \frac{\Delta}{4}\,;
\]
for instance,
\begin{equation}
\label{eq:Tdelta}
T_\Delta = 1+ \left\lceil \frac{32 \, C_1^2 \ln(6X)}{\Delta^2}\right\rceil
\end{equation}
is a suitable value, as for $t \geq T_\Delta \Bigl(1 + \ln\bigl(t(t+1)/\delta\bigr)\Bigr)$,
and using that $\ln(6X) \geq 1$,
\[
C_1 \sqrt{\frac{2\ln\bigl(6X t(t+1)/ \delta\bigr)}{t}}
\leq C_1 \sqrt{2 \,\, \frac{\ln(6X) \Bigl( 1 + \ln \bigl(t(t+1)/ \delta\bigr)\Bigr)}{ T_\Delta \Bigl(1 + \ln\bigl(t(t+1)/\delta\bigr)\Bigr)}}
= C_1 \sqrt{\frac{2\ln(6X)}{T_\Delta}} < \frac{\Delta}{4}\,.
\]
We consider the same threshold $T_{\cB,\bM,\nu}$ as in \Cref{eq:closedformT};
in particular, given the definition of $T_{\cB,\bM,\nu}$,
the proof of \Cref{eq:Tttgeq} with $T_\Delta$ instead of $\widetilde{T}_{\cB,\bM,\nu}$
guarantees that
\[
t \geq T_{\cB,\bM,\nu}\bigl(1 + \ln(1/\delta)\bigr) \qquad \mbox{entails} \qquad
t \geq T_\Delta \Bigl(1 + \ln\bigl(t(t+1)/\delta\bigr)\Bigr)\,.
\]
Together with~\Cref{eq:Tttgeq}, we thus have proved so far that
\begin{equation}
\label{eq:incevdelta}
\P \Bigl( \forall t \geq T_{\cB,\bM,\nu}\bigl(1 + \ln(1/\delta)\bigr)\,, \quad \exists \rho'_t \in \mathfrak{S} \quad \mbox{s.t.} \quad
\forall h \in \cH, \qquad \Arrowvert \tilde{\nu}_{t,\rho'_t(h)} - \nu_h \Arrowvert_2
< \Delta/4 \Bigr) \geq 1 - \delta\,.
\end{equation}
The claimed coherence can now be formally stated as follows.

\begin{lemma}
Under the same $1-\delta$ probability event considered
in the end of Appendix~\ref{sec:proof-beliefestimationerror}, which includes the event of \Cref{eq:incevdelta},
we have that for all $t \geq T_{\cB,\bM,\nu}\bigl(1 + \ln(1/\delta)\bigr)$,
the permutation $\rho'_{t}$ of \Cref{eq:incevdelta} is unique, and so
is the permutation $\rho_{t+1}$ defined in the algorithmic statement above.
In addition, there exists a permutation $\bar{\rho}$ such that
\[
\Delta/4 > \max_{h \in [H]} \bigl\Arrowvert \hat{\nu}_{t,h} - \nu_{\bar{\rho}(h)} \bigr\Arrowvert_2 \to 0
\qquad \mbox{while} \qquad
\forall h \in [H], \quad \forall h' \ne \bar{\rho}(h), \qquad
\bigl\Arrowvert \hat{\nu}_{t,h} - \nu_{h'} \bigr\Arrowvert_2 > 3\Delta/4\,;
\]
thus, the algorithm keeps track of the latent states
and uses a consistent hidden-state labeling
after~${T_{\cB,\bM,\nu}\bigl(1 + \ln(1/\delta)\bigr)}$.
\end{lemma}

\begin{proof}
Denote $t_0 =  T_{\cB,\bM,\nu}\bigl(1 + \ln(1/\delta)\bigr)$
and fix $t \geq t_0$.
Consider a suitable permutation $\rho'_{t}$.
We note that under the event of interest, by a triangle inequality,
\[
\forall m' \ne \rho'_{t}(h), \qquad
\Arrowvert \tilde{\nu}_{t,m'} - \nu_h \Arrowvert_2
\geq \Arrowvert \nu_m - \nu_h \Arrowvert_2 -
\Arrowvert \tilde{\nu}_{t,m'} - \nu_m \Arrowvert_2
> 3\Delta/4\,,
\]
where we introduced $m$ such that $\rho'_{t}(m) = m'$,
so that $\Arrowvert \tilde{\nu}_{t,m'} - \nu_m \Arrowvert_2 < \Delta/4$;
in particular, $m \ne h$, so that $\Arrowvert \nu_m - \nu_h \Arrowvert_2 \geq \Delta$.
This shows that $\rho'_{t}$ is unique.

Now define $\bar{\rho} = (\rho'_{t_0})^{-1} \circ \rho_{t_0}$.
We prove by induction on $t \geq t_0$ that
\[
\max_{h \in [H]} \bigl\Arrowvert \hat{\nu}_{t,h} - \nu_{\bar{\rho}(h)} \bigr\Arrowvert_2 < \Delta/4\,,
\]
together with the uniqueness of $\rho_{t+1}$.
For $t=t_0$, by definition, $\hat{\nu}_{t_0, h} = \tilde{\nu}_{t_0, \rho_{t_0}(h)}$, so that
\begin{align*}
\max_{h \in [H]} \bigl\Arrowvert \hat{\nu}_{t_0,h} - \nu_{\bar{\rho}(h)} \bigr\Arrowvert_2
& = \max_{h \in [H]} \bigl\Arrowvert \tilde{\nu}_{t_0, \rho_{t_0}(h)} -
\nu_{(\rho'_{t_0})^{-1}(\rho_{t_0}(h))} \bigr\Arrowvert_2 \\
& = \max_{h' \in [H]} \bigl\Arrowvert \tilde{\nu}_{t_0, h'} -
\nu_{(\rho'_{t_0})^{-1}(h')} \bigr\Arrowvert_2
= \max_{h' \in [H]} \bigl\Arrowvert \tilde{\nu}_{t_0, \rho'_{t_0}(h')} -
\nu_{h'} \bigr\Arrowvert_2 < \Delta/4\,,
\end{align*}
where we performed various re-indexations of $[H]$ based on permutations and where we
used the definition of $\rho'_{t_0}$ to get the final inequality.
Assume that the induction property holds at some $t \geq t_0$, and define
$\pi_{t+1} = \rho'_{t+1} \circ \bar{\rho}$.
Then, by the triangle inequality,
\begin{align*}
\max_{h \in [H]} \bigl\Arrowvert \hat{\nu}_{t,h} - \tilde{\nu}_{t+1, \pi_{t+1}(h)} \bigr\Arrowvert_2
& \leq \max_{h \in [H]} \bigl\Arrowvert \hat{\nu}_{t,h} - \nu_{\bar{\rho}(h)} \bigr\Arrowvert_2 +
\max_{h \in [H]} \bigl\Arrowvert \tilde{\nu}_{t+1, \rho'_{t+1}(\bar{\rho}(h))} - \nu_{\bar{\rho}(h)} \bigr\Arrowvert_2 \\
& \leq \max_{h \in [H]} \bigl\Arrowvert \hat{\nu}_{t,h} - \nu_{\bar{\rho}(h)} \bigr\Arrowvert_2
+ \max_{h \in [H]} \bigl\Arrowvert \tilde{\nu}_{t+1,\rho'_{t+1}(h)} - \nu_{h} \bigr\Arrowvert_2 < \Delta/2\,,
\end{align*}
where we used that the first maximum is smaller than $\Delta/4$ by the induction hypothesis,
and that the second maximum is also smaller than $\Delta/4$ by the definition of $\rho'_{t+1}$.
On the other hand, if $\pi \neq \pi_{t+1}$, then there exists $h \in [H]$ such that $\pi(h) \neq \rho'_{t+1}\bigl(\bar{\rho}(h)\bigr)$,
thus $(\rho'_{t+1})^{-1}\bigl(\pi(h)\bigr) \neq \bar{\rho}(h)$; using
a triangle inequality,
we have
\begin{align*}
\bigl\Arrowvert \tilde{\nu}_{t+1,\pi(h)} - \nu_{\bar{\rho}(h)} \bigr\Arrowvert_2 & \geq
\bigl\Arrowvert \nu_{(\rho'_{t+1})^{-1}(\pi(h))} - \nu_{\bar{\rho}(h)} \bigr\Arrowvert_2 - \bigl\Arrowvert \tilde{\nu}_{t+1,\pi(h)} - \nu_{(\rho'_{t+1})^{-1}(\pi(h))} \bigr\Arrowvert_2 \\
& =
\underbrace{\bigl\Arrowvert \nu_{(\rho'_{t+1})^{-1}(\pi(h))} - \nu_{\bar{\rho}(h)} \bigr\Arrowvert_2}_{> \Delta} -
\underbrace{\bigl\Arrowvert \tilde{\nu}_{t+1,\rho'_{t+1}(h')} - \nu_{h'} \bigr\Arrowvert_2}_{< \Delta/4}
> 3 \Delta/4\,,
\end{align*}
where the $< \Delta/4$ part comes from the definition and uniqueness of $\rho'_{t+1}$
used with $h' = (\rho'_{t+1})^{-1}(\pi(h))$, and the $> \Delta$ part from
the very definition of $\Delta$.
Therefore, by another triangle inequality and another use of the induction hypothesis,
	\[
		\bigl\Arrowvert \hat{\nu}_{t,h} - \tilde{\nu}_{t+1,\pi(h)} \bigr\Arrowvert_2
		\geq
		\bigl\Arrowvert \tilde{\nu}_{t+1, \pi(h)} - \nu_{\bar{\rho}(h)} \bigr\Arrowvert_2 - \bigl\Arrowvert \hat{\nu}_{t,h} - \nu_{\bar{\rho}(h)} \bigr\Arrowvert_2 >  3 \Delta/4 -\Delta/4 = \Delta/2\,.
	\]
This shows the uniqueness of $\rho_{t+1}$ and its closed-form expression $\rho_{t+1} = \pi_{t+1} =
\rho'_{t+1} \circ \bar{\rho}$.
Finally, substituting this expression and using the definition of $\rho'_{t+1}$,
\[
\forall h \in [H], \qquad
\bigl\Arrowvert \hat{\nu}_{t+1,h} - \nu_{\bar{\rho}(h)} \bigr\Arrowvert_2
=
\bigl\Arrowvert \tilde{\nu}_{t+1,\rho_{t+1}(h)} - \nu_{\bar{\rho}(h)} \bigr\Arrowvert_2
=
\bigl\Arrowvert \tilde{\nu}_{t+1,\rho'_{t+1}(\bar{\rho}(h))} - \nu_{\bar{\rho}(h)} \bigr\Arrowvert_2
< \Delta/4\,,
\]	
which closes the induction and completes the proof.
\end{proof}

\subsection{Details on How \Cref{proposition:bound-hmm-parameters} Follows from Existing Results}
\label{sec:details-how-HMM}

\Cref{proposition:bound-hmm-parameters}
follows from various results scattered throughout
\citet{spectral-2012}, which introduced the spectral method, \citet{anandkumar-2012-spectral-hmm} and \citet{anandkumar-2014-tensor-mixture}.
\citet{spectral-pomdps-2016} and \citet{RegimeBandits-2021} extended the method to more complex settings involving
Markov decision processes (of which HMMs are special cases) and they offered a synthetical view of the constants
involved in the estimation bound (though some of these constants are larger or depend on more complex quantities
due to the consideration of Markov decision processes, which involve actions for the learner).
More precisely, we follow below the exposition by
\citet[Lemma~5, Lemma~8, Theorem~3, Theorem~16]{spectral-pomdps-2016} and most importantly,
\citet[Proposition~1]{RegimeBandits-2021}, with the needed modifications:
the constants $C_1$, $C_2$, $C_3$ and the threshold $\widetilde{T}_{\cB,\bM,\nu}$ of~\Cref{sec:proof-beliefestimationerror}
are as in \citet[Appendix~B]{RegimeBandits-2021} except for removing an unnecessary term $\arrowvert \cA \arrowvert$ from $C_1$
(it only arises due to their more complex setting) and up to substituting upper or lower bounds on
some quantities, as detailed below.

\paragraph{First series of quantities.} \Cref{assumption:regularity-hmm} entails that the hidden states form an ergodic Markov chain, which thus admits a unique stationary distribution $\pi$,
is geometrically mixing in the following sense and with the following parameters
(see \citealp{kontorovich2014uniform}, \citealp[Theorems~2.7.2 and~2.7.4]{pomdp-krishnamurthy-2016},
\citealp[Appendix~B]{RegimeBandits-2021}):
\[
\forall t \geq 2, \qquad
\max_{h' \in [H]} \, \sum_{h \in [H]} \bigl\arrowvert \P(h_t = h \mid h_1 = h') - \pi(h) \bigr\arrowvert \leq 4 (1 - \epsilon_{\bM})^{t-1} \,.
\]
Further, due to the boundedness of $\bM$ and the fact that $h_1$ follows $\pi$, for all $h \in [H]$ and all $t \geq 1$,
\[
\P(h_t = h) = \pi(h) = \sum_{h' \in [H]} \pi(h') \, \underbrace{\bM_{h', h}}_{\geq \epsilon_{\bM}} \geq \epsilon_{\bM}\,.
\]

\paragraph{Second series of quantities.} We define the following multi-view matrices $\bA_1, \bA_2, \bA_3 \in [0, 1]^{X \times H}$ for $t \geq 2$,
\begin{equation*}
	\begin{split}
\forall x \in \cX, \ \forall h \in [H], \qquad \qquad & \bA_1(x, h) = \P(\bx_{t-1} = x \mid h_t = h)\,, \\
		& \bA_2(x, h) = \P(\bx_{t} = x \mid h_t = h)\,, \\
		& \bA_3(x, h) = \P(\bx_{t+1} = x \mid h_t = h)\,,
	\end{split}
\end{equation*}
and we are interested in $\min\bigl\{\sigma_{\min}(\bA_1), \sigma_{\min}(\bA_2), \sigma_{\min}(\bA_3)\bigr\}$, where $\sigma_{\min}(\bA_i)$
is the smallest singular value of the matrix $\bA_i$, for $i \in \{1,2,3\}$.
We have $\bA_2 = \bE$ and $\bA_3 = \bE \bM^{\transp}$. The closed-form expression for $\bA_1$
is slightly more complex as we have to go backwards in the HMM:
\begin{multline*}
\bA_1(x, h) = \P(\bx_{t-1} = x \mid h_t = h) = \sum_{h' \in \cH} \overbrace{\P(\bx_{t-1} = x \mid h_{t-1} = h', \ h_t=h)}^{\mbox{\tiny no dep. on $h$ and} \ = \bE_{x,h'}}
\P(h_{t-1}=h' \mid h_t=h)\,, \\
\mbox{where} \qquad \P(h_{t-1}=h' \mid h_t=h) =
\frac{\P(h_t=h \mid h_{t-1}=h')\,\P(h_{t-1}=h')}{\P(h_t=h)}
= \frac{\pi(h')\,\bM_{h', h}}{\pi(h)}
\end{multline*}
(we recall that the Marhov chain of hidden states is initialized with
the stationary distribution $\pi$),
so that
\[
\bA_1 = \bE \diag(\pi) \bM \diag(\pi)^{-1}\,,
\]
where $\diag$ transforms a vector into a square diagonal matrix
with diagonal coefficients given by the vector.
We now use that for two matrices $\bE,\bE'$ of compatible sizes, we have
$\sigma_{\min}(\bE \, \bE') \geq \sigma_{\min}(\bE) \, \sigma_{\min}(\bE')$.
Also, given that all components of $\pi$ are in the
interval $[\epsilon_{\bM},1]$,
\[
\sigma_{\min} \bigl( \diag( \pi) \bigr) \geq \epsilon_{\bM}
\qquad \mbox{and} \qquad
\sigma_{\min} \bigl( \diag(\bp_{t})^{-1} \bigr) \geq 1\,.
\]
We obtain $\sigma_{\min}(\bA_2) = \sigma_{\min}(\bE)$, as well as
\[
\sigma_{\min}(\bA_1) \geq \sigma_{\min}(\bE)\,\sigma_{\min}(\bM)\,\epsilon_{\bM}
\qquad \mbox{and} \qquad
\sigma_{\min}(\bA_3) \geq \sigma_{\min}(\bE)\,\sigma_{\min}(\bM)\,.
\]
Given that $\bM$ is a stochastic matrix, we have $\epsilon_{\bM} < 1$ and $\sigma_{\min}(\bM) \leq 1$, and thus
\[
\min\bigl\{\sigma_{\min}(\bA_1), \sigma_{\min}(\bA_2), \sigma_{\min}(\bA_3)\bigr\}
\geq \sigma \eqdef \sigma_{\min}(\bE)\,\sigma_{\min}(\bM)\,\epsilon_{\bM}\,.
\]

\paragraph{Third series of quantities.}
Finally, we also introduce the co-occurrence matrix $\bA_4 \in [0, 1]^{X \times X}$:
\[
\forall (i,j) \in \cX \times \cX, \qquad
\bA_4(i, j) = \E\bigl[\Ind{\bx_{t+1}=i} \, \Ind{\bx_{t-1}=j}\bigr] = \P(\bx_{t+1}=i, \bx_{t-1}=j)
\]
and are interested in $\sigma_{\min}(\bA_4)$.
In a HMM, $\bx_{t+1}$ and $\bx_{t-1}$ are conditionally independent given $h_t$,
as both are drawn independently conditional on $h_{t+1}$ and $h_{t-1}$;
therefore, for any $i,j \in \cX$,
\[
\P(\bx_{t+1}=i, \bx_{t-1}=j) = \sum_{h \in [H]} \P(h_t = h) \, \P(\bx_{t+1}=i \mid h_t = h) \, \P(\bx_{t-1}=j \mid h_t = h)\,,
\]
and thus, $\bA_4 = \bA_3 \diag(\pi) \bA_1^{\transp}$.
Similarly as above, this rewriting entails 	
\[
\sigma_{\min}(\bA_4) \geq \sigma_{\min}(\bA_3) \, \sigma_{\min}\bigl( \diag(\pi) \bigr) \, \sigma_{\min} \bigl(\bA_1^{\transp} \bigr)
\geq \bigl( \sigma_{\min}(\bE)\,\sigma_{\min}(\bM)\,\epsilon_{\bM} \bigr)^2 = \sigma^2\,.
\]

\subsection{Reminder on the Spectral Method for MHH Parameter Estimation}
\label{sec:spectral}

Finally, for the sake of self-completedness,
we recall the spectral method for HMM parameter estimation, which is the method considered
in \Cref{proposition:bound-hmm-parameters}.
The exposition follows closely \citet[Section~4.2]{anandkumar-2012-spectral-hmm}.

Remember that we assume that the context $\cX$ space is finite, with cardinality $X$. With no loss of generality, we
may therefore identify it with the canonical vectors in $\R^X$,
i.e., up to numbering the elements in $\cX$ and substituting the $i$--th element, where $i \in [X]$,
by the column vector $\be_i = (0,\ldots,0,1,0\ldots,0)^{\transp} \in \R^X$, where the unique element $1$ is in $i$--th position.

We extend the tensor product notation to products of three elements:
for all vectors $\bu, \, \bv, \, \bw \in \R^X$,
the two-dimensional matrix $\bu \otimes \bv$ and the three-dimensional
matrix $\bu \otimes \bv \otimes \bw$ are defined component wise by
\[
\forall (i,j,k) \in [X]^3, \qquad
(\bu \otimes \bv)_{i, j} = \bu_i \bv_j \quad \mbox{and} \quad
(\bu \otimes \bv \otimes \bw)_{i, j, k} = \bu_i \bv_j \bw_k\,.
\]
We may now restate the special case of the estimation of HMM parameters
as\footnote{The indexing conventions are slightly different
here and in their article, so that extra transpositions appear here.} detailed in \citet[Section~4.2]{anandkumar-2012-spectral-hmm}.

\begin{figure}[!hbt]
	\begin{nbox}[title={Box C: Spectral estimation of HMM parameters from observations $\bx_{1:t}$}]
		\textbf{Inputs:}
        known number $H$ of hidden states;
		observations ${\bx}_1, {\bx}_2, \ldots, {\bx}_t$, where $t \geq 3$, in the
        form of canonical vectors \smallskip
		
        \textbf{Parameter:} invertible matrix $\Gamma \in \R^{H \times H}$, e.g., a random rotation matrix; denote by $\Gamma_{h,\,\cdot\,} \in \R^H$ its $h$--th row transposed into a column-vector \smallskip

		\textbf{Output:} estimates $\tilde{\bE}_{t}$ and $\tilde{\bM}_{t}$
        of the emission and transition matrices $\bE$ and $\bM$ \medskip
		
		\begin{enumerate}
			\item Compute the empirical moments
			\[
				\tilde{P}^{(t)}_{3, 1} = \frac{1}{t - 2} \sum_{s=2}^{t-1} {\bx}_{s+1} \otimes {\bx}_{s-1}\,, \quad
				\tilde{P}^{(t)}_{3, 2} = \frac{1}{t - 2} \sum_{s=2}^{t-1} {\bx}_{s+1} \otimes {\bx}_{s}\,, \quad
				\tilde{P}^{(t)}_{3, 1, 2} = \frac{1}{t - 2} \sum_{s=2}^{t-1} {\bx}_{s+1} \otimes {\bx}_{s-1} \otimes {\bx}_{s}\,.
			\]
			
			\item Compute $\tilde{U}_3, \tilde{U}_1 \in \R^{X \times H}$, the matrices whose columns are, respectively, the left and right singular vectors of $\tilde{P}^{(t)}_{3, 1}$ associated with its largest $H$ singular values. \smallskip \\
            Compute $\tilde{U}_2 \in \R^{X \times H}$, the matrix whose columns are the right singular vectors of $\tilde{P}^{(t)}_{3, 2}$ associated with its largest $H$ singular values. \medskip
			
			\item Define the contraction of the third-order tensor $\tilde{P}^{(t)}_{3, 1, 2}(\bz)$ along its third mode
            for each vector $\bz \in \R^X$ by
			\[
			\bigl[\tilde{P}^{(t)}_{3, 1, 2}(\bz) \bigr]_{i, j} :=
			\sum_{k=1}^X \bigl[ \tilde{P}^{(t)}_{3, 1, 2}\bigr]_{i,j,k} \, \bz_k\,,
			\]
			and let
			\[
				\tilde{B}^{(t)}_{3, 1, 2}(\bz) =
				\Bigl(\tilde{U}_{3}^{\transp} \tilde{P}^{(t)}_{3, 1, 2} (\bz) \tilde{U}_{1}\Bigr)
				\Bigl(\tilde{U}_{3}^{\transp} \tilde{P}^{(t)}_{3, 1} \tilde{U}_{1}\Bigr)^{-1}\,;		
			\]
            then, compute a matrix $\tilde{R} \in \R^{H \times H}$, with columns of unit Euclidean norm, such that
			\[
				\tilde{R}^{-1} \, \tilde{B}^{(t)}_{3, 1, 2}(\tilde{U}_{2} \Gamma_{1,\,\cdot\,}) \, \tilde{R} =
				\diag\bigl( \tilde{\lambda}_{1, 1}, \tilde{\lambda}_{1, 2}, \dots, \tilde{\lambda}_{1, H} \bigr)\,;
			\]
			if this is not possible, redraw $\Gamma$ and repeat this step. \medskip
			
			\item For each $h \in [H]$, using the matrix $\tilde{R}$ computed in Step~3, define $\tilde{\lambda}_{h, 1}, \tilde{\lambda}_{h, 2}, \dots, \tilde{
\lambda}_{h, H}$ as the diagonal entries of
			\[
			\tilde{R}^{-1} \, \tilde{B}^{(t)}_{3, 1, 2}(\tilde{U}_{2} \Gamma_{h,\,\cdot\,}) \, \tilde{R}
			\]
		    and form the matrix $\tilde{L} \in \R^{H \times H}$ with entries $\tilde{L}_{h, j} := \tilde{\lambda}_{h, j}$ for all $h, j \in [H]$. \medskip
			
			\item Define $\tilde{O}_t = \tilde{U}_2 \Gamma^{-1} \tilde{L}$ and output
			\[
				\tilde{\bE}_{t} = \tilde{O}_t
				\qquad \mbox{and} \qquad
				\tilde{\bM}_{t} = \Bigl(\bigl(\tilde{U}_3^{\transp} \tilde{O}_{t} \bigr)^{-1} \tilde{R} \Bigr)^{\!\transp}\,.
			\]
		\end{enumerate}
	\end{nbox}
\end{figure}

\paragraph{Remark on the practical implementation.}
In our numerical experiments, we will post-process $\tilde{M}_{t}$ and $\tilde{E}_{t}$ by clipping small negative entries and performing
the needed normalizations so that they define valid emission and transition matrices.

\clearpage
\section{HMM Forgetting Properties and Related Reminders}
\label{appendix:general}

This appendix justifies the exponentially fast forgetting of the
initial distribution of the HMM stated in
\Cref{assumption:strong-mixing-hmm}.

\assexpfastmixing*

This assumption involves quantities of the form $\P_E(E' \mid \cG)$,
where $E$ and $E'$ are events and $\cG$ is a $\sigma$--algebra,
all defined on the same underlying probability
space $(\Omega,\cF)$: we
provide some reminders on such quantities---including their definitions---in
Appendix~\ref{sec:Bayes-cond-proba} below.

See also Appendix~\ref{app:hmm-obs-ctxs} for reminders
on why we only condition by $\bx_{s+1:t}$ in the probability distributions above.

For now, we compare the assumption above to the alternative forgetting property
used by \citet[see Corollary C.4.1 therein]{nelson2022ContextLatent}, which is of the form:
there exists $\gamma \in [0,1)$ and a constant $C$, both depending on the
HMM, such that
\begin{equation}
\label{eq:forgetting-nelson-E}
	\E \left[\sum_{j \in [H]} \Bigl\arrowvert \P_{\{h_s=h\}}(h_t = j \mid \bx_{s+1:t}) - \P(h_t = j \mid \bx_{s+1:t}) \Bigr\arrowvert \right] \leq C \exp\bigl(-\gamma (t - s)\bigr)\,.
\end{equation}
This constant $\gamma$ is given by $\e^{-\gamma'/2}$, where $\gamma'$
is the minimal mixing rate of the transition matrix $\bM$:
\[
\gamma' = \min_{h', h'' \in \cH} \sum_{h \in \cH} \min\{M_{h,h'},M_{h,h''}\}\,.
\]
\citet{nelson2022ContextLatent} apply some inequalities for Markov processes established by
\citealp[see, in particular, Theorem~3]{Boyen1998Stochastic} to show the
property stated in \Cref{eq:forgetting-nelson-E}.

The forgetting property used by \citet{nelson2022ContextLatent}
is (by far) less demanding as the condition is to be satisfied in expectation
compared to the one of \Cref{assumption:strong-mixing-hmm}.
However, this is perfectly consistent with the
fact that \citet{nelson2022ContextLatent}
only provide bounds in expectation while the present article
instead aims for high-probability bounds (see
the discussion in \Cref{sec:intro}).

\subsection{Bayes' Formula for Probabilities Conditional to $\sigma$--Algebras}
\label{sec:Bayes-cond-proba}

We assume that $P(E) > 0$ and let $\P_E = \P(\,\cdot \mid E)$
denote the conditional probability with respect to event $E$. This is a probability
distribution over $(\Omega,\cF)$ and its conditional probability $\P_E(\,\cdot\, \mid \cG)$
with respect to the $\sigma$--algebra $\cG \subseteq \cF$
is thus well defined.

We recall in the lemma below
how to apply rigourously Bayes' theorem in this context.

\begin{lemma}
	\label{lm:bayes-event-random-variable}
	With the notation above and under the condition $P(E) > 0$, we have
	\[
	\P(E \cap E' \mid \cG) = \P(E \mid \cG) \,\times\, \P_E(E' \mid \cG)\,.
	\]
\end{lemma}

\begin{proof}
	A characterization of the conditional expectation $\E[Z \mid \cG]$ of a nonnegative random variable $Z \geq 0$ is that it is a $\cG$--measurable random variable satisfying
	\begin{equation}
		\label{eq:charact-esp-cond}
		\forall A \in \cG, \qquad \E[Z \, \IndEv{A}] = \E\bigl[ E[Z\mid\cG] \, \IndEv{A}]\,.
	\end{equation}
	We thus should prove that for all events $A \in \cG$,
	\begin{equation}
		\label{eq:charact-esp-cond-2}
		\E\bigl[ \IndEv{A} \, \P(E \cap E' \mid \cG) \bigr] = \E\bigl[ \IndEv{A} \, \P(E \mid \cG)  \,\times\,  \P_{E}(E' \mid \cG) \bigr]\,.
	\end{equation}
	By \Cref{eq:charact-esp-cond} for the second equality, the left-hand side of \Cref{eq:charact-esp-cond-2} can be rewritten as
	\[
	\E\bigl[ \IndEv{A} \, \P(E \cap E' \mid \cG) \bigr] = \E \bigl[ \IndEv{A} \, \E[\IndEv{E} \, \IndEv{E'} \mid \cG] \bigr]
	= \E[\IndEv{A} \, \IndEv{E} \, \IndEv{E'}]\,.
	\]
	By $\cG$--measurability of $\P_{E}(E' \mid \cG)$ for the second equality and
	by~\Cref{eq:charact-esp-cond} for the third equality,
	the right-hand side of \Cref{eq:charact-esp-cond-2} can first be rewritten as
	\begin{align*}
		\E \bigl[ \IndEv{A} \, \P(E \mid \cG)  \,\times\,  \P_{E}(E' \mid \cG) \bigr]
		& = \E \bigl[ \IndEv{A} \, \E[\IndEv{E} \mid \cG]  \,\times\,  \P_{E}(E' \mid \cG) \bigr] \\
		& = \E \bigl[ \IndEv{A} \, \E[\IndEv{E} \, \P_{E}(E' \mid \cG) \mid \cG] \bigr]
		= \E \bigl[ \IndEv{A} \, \IndEv{E} \, \P_{E}(E' \mid \cG) \bigr]\,.
	\end{align*}
	We continue the calculation by noting, for the first and last equalities below, that, by definition,
	$\E[ \, \cdot \,\, \IndEv{E} ] = \P(E) \times \E_E[ \, \cdot \,\, \IndEv{E}]$,
	where $\E_E$ denotes the conditional expectation with respect to the event $E$,
	and, for the third equality, by resorting again to \Cref{eq:charact-esp-cond} with $\E_E$:
	\[
		\E \bigl[ \IndEv{A} \, \IndEv{E} \, \P_{E}(E' \mid \cG) \bigr]
		= \E_E \bigl[ \IndEv{A} \, \P_{E}(E' \mid \cG) \bigr]
		= \E_E \bigl[ \IndEv{A} \, \E_{E}[\IndEv{E'} \mid \cG] \bigr]
		= \E_E[ \IndEv{A} \, \IndEv{E'}]
		= \E[ \IndEv{A} \, \IndEv{E} \, \IndEv{E'}] \,.
	\]
	The proof is concluded by collecting all equalities.
\end{proof}

\subsection{Consequences of the Hidden Markov Model Formulation}
\label{app:hmm-obs-ctxs}

Before we discuss (and prove) that \Cref{assumption:strong-mixing-hmm}
is natural, it is useful to state a reminder on how some conditionings may be
simplified.

Recall the definition of $\cFall_{t}$ from \Cref{sec:problem-formulation}:
\[
\cFall_{t} = \sigma\biggl( \Bigl(  h_\tau,\, \bx_\tau, \, \bigl( \eta_\tau(a) \bigr)_{a \in \cA} \Bigr)_{\tau \leq t-1}, \, h_t, \, \bx_t \biggr)\,;
\]
this is the richer filtration we consider.

The HMM model implies that, for $\tau < \tau'$, conditionally on $h_\tau$, the distribution of
$h_{\tau'}$ is independent of past and present information, that is, $\cFall_{\tau}$ and $\eta_\tau(a)$, but of course, not from future information corresponding to rounds $\tau+1$ till $\tau'$, like
the contexts $\bx_{\tau+1:\tau'}$.

More formally and for example, we have, for $\tau < \tau'$,
\[
\P( h_{\tau'} = h' \mid \cFall_{\tau}, \, \eta_\tau(a_\tau), \, \bx_{\tau+1:\tau'})
= \P( h_{\tau'} = h' \mid h_{\tau}, \, \bx_{\tau+1:\tau'})\,.
\]
Via the same tools as in Appendix~\ref{sec:Bayes-cond-proba},
this entails, in particular, that for all $h \in [H]$,
\begin{equation}
	\label{eq:HHM-indep}
	\P_{\{h_\tau = h\}}( h_{\tau'} = h' \mid \cFall_{\tau}, \, \eta_\tau(a_\tau), \, \bx_{\tau+1:\tau'})
	= \P_{\{h_\tau = h\}}( h_{\tau'} = h' \mid \bx_{\tau+1:\tau'})\,.
\end{equation}

\subsection{Some Classic Condition Leading to \Cref{assumption:strong-mixing-hmm}}
\label{appendix:HMM-forgetting}

In this appendix, we show how \Cref{assumption:strong-mixing-hmm}
follows from the assumption below,
considered by \citet[Chapter~3]{olivier-2005-inf-hmm} and rewritten in our context.
Remember that we consider an homogeneous HMM (the distributions of transitions and emissions
do not depend on the round), which is why the assumption is only
stated with hidden states $h_1$ and $h_2$.

\begin{assumption}[{\citealp[Assumption~59, ``strong mixing condition'']{olivier-2005-inf-hmm}}]
	\label{assumption:strong-mixing-condition}
	There exists a transition kernel $K: \cX \times \cH \to (0, 1)$ and measurable functions $\zeta^{-}, \zeta^{+}: \cX \to (0, +\infty)$,
	with $\zeta^{-} \leq \zeta^{+}$, such that
	for all Borel sets $E$ of $\cX$,
	and all $h, h' \in [H]$,
	\[
	\int_E \zeta^{-}(\bx) \, K(\bx, h') \, \d \bx \leq \P(h_2 = h', \, \bx_2 \in E \mid h_1 = h) \leq \int_E \zeta^{+}(\bx) \, K(\bx, h') \,\d \bx \,.
	\]
\end{assumption}

\citet[Proposition~61]{olivier-2005-inf-hmm}
almost immediately entails the following lemma.

\begin{lemma}
	\Cref{assumption:strong-mixing-condition} entails \Cref{assumption:strong-mixing-hmm}, whenever
	\[
	\sup_{x \in \cX} \biggl( 1 - \frac{\zeta^{-}(\bx)}{\zeta^{+}(\bx)} \biggr) < 1\,.
	\]
	This is the case at least when $\cX$ is finite.
\end{lemma}

\begin{proof}
	By homogeneity, it suffices to prove \Cref{assumption:strong-mixing-hmm} for $s = 1$ and $t \geq 2$.
	\citet[Proposition~61, based on Assumption~\ref{assumption:strong-mixing-condition} above]{olivier-2005-inf-hmm}
	guarantees that for all $t \geq 2$, for all pairs $h,h' \in [H]$ of hidden states,
	\begin{equation}
		\label{eq:Cappe-61}
		\sum_{j \in [H]} \Bigl| \P_{\{h_1=h\}}(h_t=j \mid \bx_{1:t}) - \P_{\{h_1=h'\}}(h_t=j \mid \bx_{1:t}) \Bigr|
		\leq 2 \prod_{\tau=2}^{t} \left(1 - \frac{\zeta^{-}(\bx_\tau)}{\zeta^{+}(\bx_\tau)}\right).
	\end{equation}
	Under the HMM property, $h_t$ is independent of $\bx_1$ given $h_1$; thus $\bx_{1}$ can be removed from the
	conditionings in the left-hand sides of the above inequality (see Appendix~\ref{app:hmm-obs-ctxs}). Each of the terms $1 - \zeta^{-}(\bx_\tau)/\zeta^{+}(\bx_\tau)$ is non-negative, since
	$\zeta^{-} \leq \zeta^{+}$ by assumption.
	We further bound the right-hand side of \Cref{eq:Cappe-61} by $2\gamma^{t-1}$ (which is the upper bound claimed
	by \Cref{assumption:strong-mixing-hmm}), where
	\[
	\gamma = \sup_{x \in \cX} \biggl( 1 - \frac{\zeta^{-}(\bx)}{\zeta^{+}(\bx)} \biggr).
	\]
	When $\cX$ is finite, \Cref{assumption:strong-mixing-condition} imposed that $\zeta^{-}(\bx) > 0$
	for all $\bx \in \cX$, thus each of the finitely many terms in the defining maximum of $\gamma$
	are strictly smaller than~$1$, therefore, so is $\gamma$.
\end{proof}

\clearpage
\section{Proof of \Cref{theorem:total-regret-bound}}
\label{appendix:complex-analysis}

The aim of this section is to prove the main result, which we restate below.

\tregretbound*

The analysis follows the same structure as the one
in Appendix~\ref{appendix:simplified-analysis}
for the simplified model.
Therein, the main piece in establishing the regret bound of~\Cref{theorem:total-regret-bound-simplified}
consisted of building confidence intervals in \Cref{lm:est-simple}:
the total regret bound was basically given by $2$ times the sum of the errors margins
of these confidence intervals.
The counterpart to \Cref{lm:est-simple} is the following.
The only difference is that union bounds must be performed with greater care,
hence the consideration of stage-varying confidence levels $1-\delta_s$.

\begin{lemma}
	\label{lm:rt-bound}
    Fix errors levels $\delta_s \in (0,1)$ for each $s \geq 1$.
	Under Assumptions~\ref{assumption:bound-R-linear}--\ref{assumption:idio-noise}--\ref{assumption:strong-mixing-hmm},
	for all $s \geq 1$, for all $t \in [(s-1) \ell + 1, \, s\ell]$, with probability
	at least $1 - 2\delta_s$,
	\begin{equation}
		\label{eq:lm-rt-bound}
		\forall a \in \cA, \qquad \left| \sum_{h \in \cH} \bb_t(h) \bphi(a, \bx_t)^{\transp} \btheta^{\star}_{h}
		- \sum_{h \in \cH} \hbb_t(h) \bphi(a, \bx_t)^{\transp} \hbtheta_{(s-1) \ell,h} \right| \leq \eps'_{t,s,\lambda,\delta_s,a}\,,
	\end{equation}
	where $\eps'_{t,1,\lambda,\delta_1,a} = 1+\sqrt{d}/\lambda$ for $s=1$, and for $s \geq 2$,
	\begin{multline*}
	\eps'_{t,s,\lambda,\delta_s,a} = \bigl\Arrowvert \bb_t - \hbb_t \bigr\Arrowvert_1 + \biggl\Arrowvert G_{(s -1) \ell}^{-1} \Bigl(\hbb_t \otimes \bphi(a, \bx_t)\Bigr) \biggr\Arrowvert_2 \, \Biggl( \lambda \sqrt{H} \, C_{\btheta^{\star}} + \sqrt{ \frac{4(s-1)(1+s\gamma)\ell}{\delta_s(1-\gamma)}}
			+ \sqrt{\frac{1}{\delta_s} C_{\eta} (s-1) \ell } \\
            + \frac{2 (s-1) \gamma}{1 - \gamma} + \sum_{\tau=1}^{(s-1)\ell} \Arrowvert \bb_{\tau} - \hbb_{\tau} \Arrowvert_1
			 \Biggr)\,.
	\end{multline*}
\end{lemma}

Appendix~\ref{appendix:reward-estimation-bound} provides the proof of
\Cref{lm:rt-bound}, while
Appendix~\ref{appendix:proof-total-regret} proves \Cref{theorem:total-regret-bound}
based on \Cref{lm:rt-bound}.

\subsection{Proof of \Cref{lm:rt-bound}}
\label{appendix:reward-estimation-bound}

The proof adapts the one of \Cref{lm:est-simple}:
the very beginning is similar, up to considering
$\hbtheta_{(s-1) \ell,h}$ instead of $\hbtheta_{t-1,h}$,
but the core of the proof is significantly different,
as the LinUCB approach by \citet{abbasi2011improved}
cannot be followed anymore; see details on the reasons for this
non-applicability in Appendix~\ref{app:other-techn-tool}.

The proof below actually details the claims and proof structure
presented in Appendix~\ref{app:other-techn-tool},
partly based on some $\L^2$--Markov-based deviation
inequality by~\citet{nelson2022ContextLatent}.

\begin{proof}
The deterministic bound $1+\sqrt{d}/\lambda$ actually holds
for all $s$ and $t$, see the comments after the statement of \Cref{lm:est-simple}.
The rest of the proof thus only covers the case $s \geq 2$. By a triangle inequality
	and by leveraging again \Cref{assumption:bound-R-linear}, the target quantity can be bounded by
	\begin{align}
		\nonumber
		\lefteqn{\left| \sum_{h \in \cH} \bb_t(h) \bphi(a, \bx_t)^{\transp} \btheta^{\star}_{h}
			- \sum_{h \in \cH} \hbb_t(h) \bphi(a, \bx_t)^{\transp} \hbtheta_{(s-1) \ell,h} \right|} \\
		\nonumber
		& \leq \Biggl| \sum_{h \in \cH} \bigl(\bb_t(h) - \hbb_t(h) \bigr) \overbrace{\bphi(a, \bx_t)^{\transp} \btheta^{\star}_{h}}^{|\,\cdot\,| \leq 1} \Biggr|
		+ \left| \sum_{h \in \cH} \hbb_t(h) \bphi(a, \bx_t)^{\transp} \bigl(\btheta^{\star}_{h} - \hbtheta_{(s-1) \ell,h} \bigr) \right| \\
		\label{eq:rest-proof-eq}
		& \leq \underbrace{\sum_{h \in \cH} \bigl|(\bb_t(h) - \hbb_t(h)\bigr|}_{= \Arrowvert \bb_t - \hbb_t \Arrowvert_1} +
		\Biggl| \underbrace{\sum_{h \in \cH} \hbb_t(h) \bphi(a, \bx_t)^{\transp} \bigl(\btheta^{\star}_{h} - \hbtheta_{(s-1) \ell,h} \bigr)}_{=
			(\hbb_t \otimes \bphi(a, \bx_t))^{\transp} (\hbtheta_{(s -1) \ell} - \btheta^{\star})} \Biggr|.
	\end{align}
	The rest of the proof bounds the second term of the upper bound above. We first rewrite the differences
	$\btheta^{\star} - \hbtheta_{(s-1) \ell}$ in terms of the payoffs, using the definitions around~\Cref{eq:estimate-theta}:
	\begin{multline*}
		\btheta^{\star} - \hbtheta_{(s-1) \ell}
		= G_{(s -1) \ell}^{-1} \left( G_{(s -1) \ell} \, \btheta^{\star} -
		\sum_{\tau=1}^{(s -1) \ell} \bigl(\hbb_\tau \otimes \bphi(a_\tau, \bx_\tau)\bigr) r_\tau(a_\tau) \right) \\
		= G_{(s -1) \ell}^{-1} \left( \lambda \id{dH} \btheta^{\star} -
		\sum_{\tau=1}^{(s -1) \ell} \bigl(\hbb_\tau \otimes \bphi(a_\tau, \bx_\tau)\bigr) \Bigl(r_\tau(a_\tau)  -
		\bigl(\hbb_\tau \otimes \bphi(a_\tau, \bx_\tau)\bigr)^{\!\transp} \btheta^{\star}\Bigr) \right) .
	\end{multline*}
	We substitute $r_{\tau}(a_{\tau})$ in the expression above, but first rewrite it :
	\begin{align*}
		& r_{\tau}(a_{\tau}) \defeq \bphi(a_{\tau}, \bx_{\tau})^{\transp} \btheta^{\star}_{h_{\tau}} + \eta_{\tau}(a_{\tau}) \\
		& = \bphi(a_{\tau},\bx_{\tau})^{\transp} \biggl( \btheta^\star_{h_{\tau}}
		- \sum_{h' \in [H]} \bar{\bb}_{\tau}(h') \btheta^\star_{h'} \biggr)
		+ \sum_{h' \in [H]}  \bphi(a_{\tau},\bx_{\tau})^{\transp} \btheta^\star_{h'} \Bigl( \bar{\bb}_{\tau}(h') - \hbb_{\tau}(h') \Bigr)
        + \bigl(\hbb_\tau \otimes \bphi(a_\tau, \bx_\tau)\bigr)^{\!\transp} \btheta^{\star}
		+ \eta_{\tau}(a_{\tau})\,.
	\end{align*}
    The second term in \Cref{eq:rest-proof-eq} may therefore be rewritten as
    \begin{align*}
    \sum_{h \in \cH}  \hbb_t(h) \bphi(a, \bx_t)^{\transp} \bigl(\btheta^{\star}_{h} - \hbtheta_{(s-1) \ell,h} \bigr) & =
    \Bigl(\hbb_t \otimes \bphi(a, \bx_t)\Bigr)^{\transp} \Bigl(\hbtheta_{(s -1) \ell} - \btheta^{\star} \Bigr) \\
    & = \Bigl(\hbb_t \otimes \bphi(a, \bx_t)\Bigr)^{\transp} \, G_{(s -1) \ell}^{-1} \,
    \bigl( \Sdiffp{(s-1) \ell} + \Sbeliefp{(s-1) \ell} + \Setap{(s-1) \ell} - \lambda \id{dH} \btheta^{\star} \bigr)\,,
    \end{align*}
    where
    \begin{align*}
    \Sdiffp{(s-1) \ell} & = \sum_{\tau=1}^{(s-1) \ell} \bphi(a_\tau,\bx_\tau)^{\transp} \biggl( \btheta^\star_{h_\tau}
	- \sum_{h' \in [H]} \bar{\bb}_\tau(h') \btheta^\star_{h'} \biggr) \hbb_\tau \otimes \bphi(a_\tau,\bx_\tau) \,, \\
    \Sbeliefp{(s-1) \ell} & = \sum_{\tau=1}^{(s-1)\ell} \sum_{h' \in [H]}  \bphi(a_\tau,\bx_\tau)^{\transp} \btheta^\star_{h'} \bigl( \bar{\bb}_\tau(h')
	- \hbb_\tau(h')  \bigr) \hbb_\tau \otimes \bphi(a_\tau,\bx_\tau)\,, \\
    \Setap{(s-1) \ell} & = \sum_{\tau=1}^{(s-1)\ell} \eta_{\tau}(a_\tau) \, \hbb_\tau \otimes \bphi(a_\tau,\bx_\tau)\,.
    \end{align*}
    The Euclidean norm of each of these three term is bounded in a series of lemmas below:
    $\Sdiffp{(s-1) \ell}$ in \Cref{lm:true-vs-belief},
    $\Setap{(s-1) \ell}$ in \Cref{lm:bound-idiosyncratic}, and
    $\Sbeliefp{(s-1) \ell}$ in \Cref{lm:bound-belief-vs-estimated}.

    More precisely, we bound the second term in \Cref{eq:rest-proof-eq}
	by a Cauchy-Schwarz inequality and a triangle inequality:
	\begin{align*}
		\lefteqn{\left| \sum_{h \in \cH}  \hbb_t(h) \bphi(a, \bx_t)^{\transp} \bigl(\btheta^{\star}_{h} - \hbtheta_{(s-1) \ell,h} \bigr) \right|} \\
			& = \Biggl| \Bigl(\hbb_t \otimes \bphi(a, \bx_t)\Bigr)^{\transp} \, G_{(s -1) \ell}^{-1} \,
    \bigl( \Sdiffp{(s-1) \ell} + \Sbeliefp{(s-1) \ell} + \Setap{(s-1) \ell} - \lambda \id{dH} \btheta^{\star} \bigr) \Biggr| \\
		& \leq \biggl\Arrowvert G_{(s -1) \ell}^{-1} \Bigl(\hbb_t \otimes \bphi(a, \bx_t)\Bigr) \biggr\Arrowvert_2
		\, \Bigl( \bigl\Arrowvert \Sdiffp{(s-1) \ell} \bigr\Arrowvert_2 + \bigl\Arrowvert \Setap{(s-1) \ell} \bigr\Arrowvert_2 +
    \bigl\Arrowvert \Sbeliefp{(s-1) \ell} \bigr\Arrowvert_2 + \Arrowvert \lambda\btheta^{\star} \Arrowvert_2 \Bigr)\,,
	\end{align*}
    where $\Arrowvert \lambda\btheta^{\star} \Arrowvert_2 \leq \lambda \sqrt{H} \, C_{\btheta^{\star}}$
	by~\Cref{assumption:bound-R-linear}.
	\Cref{lm:true-vs-belief} ensures that with probability at least $1-\delta_s$,
	\[
	\bigl\Arrowvert \Sdiffp{(s-1) \ell} \bigr\Arrowvert_2
    = \Biggl\Arrowvert \sum_{\tau=1}^{(s-1) \ell} \bphi(a_\tau,\bx_\tau)^{\transp} \biggl( \btheta^\star_{h_\tau}
	- \sum_{h' \in [H]} \bar{\bb}_\tau(h') \btheta^\star_{h'} \biggr) \hbb_\tau \otimes \bphi(a_\tau,\bx_\tau) \Biggr\Arrowvert_2
	\leq \sqrt{ \frac{4(s-1)(1+s\gamma)\ell}{\delta_s(1-\gamma)}} \, ,
	\]
    where we performed some bounding to get a more compact bound.
	\Cref{lm:bound-idiosyncratic} ensures that with probability at least $1-\delta_s$,
	\[
	\bigl\Arrowvert \Setap{(s-1) \ell} \bigr\Arrowvert_2 =
    \Biggl\Arrowvert \sum_{\tau=1}^{(s-1)\ell} \eta_{\tau}(a_\tau) \, \hbb_\tau \otimes \bphi(a_\tau,\bx_\tau) \Biggr\Arrowvert_2
	\leq \sqrt{ \frac{1}{\delta_s} \, C_{\eta}  (s - 1) \ell } \, .
	\]
    Finally, \Cref{lm:bound-belief-vs-estimated} guarantees that with probability~$1$,
	\[
    \bigl\Arrowvert \Sbeliefp{(s-1) \ell} \bigr\Arrowvert_2 =
	\Biggl\Arrowvert \sum_{\tau=1}^{(s-1)\ell} \sum_{h' \in [H]}  \bphi(a_\tau,\bx_\tau)^{\transp} \btheta^\star_{h'} \bigl( \bar{\bb}_\tau(h')
	- \hbb_\tau(h')  \bigr) \hbb_\tau \otimes \bphi(a_\tau,\bx_\tau) \Biggr\Arrowvert_2  \leq \frac{2 (s-1) \gamma}{1 - \gamma} + \sum_{\tau=1}^{(s-1)\ell} \bigl\Arrowvert \bb_\tau - \hbb_\tau \bigr\Arrowvert_1\,.
	\]
	The proof is concluded by collecting all the bounds above
	and by applying a union bound.
\end{proof}

\subsubsection{Bound on $\bigl\Arrowvert  \Sdiffp{s\ell} \bigr\Arrowvert_2$}
\label{sec:proof:lm:true-vs-belief}

This is the most difficult term to bound
and we follow the approach described at a high level in Appendix~\ref{app:other-techn-tool}:
this approach constitutes the key technical contribution by \citet{nelson2022ContextLatent}.
In particular, we apply Markov's inequality in $\mathbb{L}^2$--norm;
this has the drawback that the associated high-probability bound
depends on the risk $\delta$ through $\sqrt{1/\delta}$ instead
of $\sqrt{\ln(1/\delta)}$ in the LinUCB approach (see Appendix~\ref{appendix:simplified-analysis}).

\begin{lemma}
	\label{lm:true-vs-belief}
	Under Assumptions~\ref{assumption:bound-R-linear} and~\ref{assumption:strong-mixing-hmm},
	for all $s \geq 1$,
	\[
	\E \Bigl[ \bigl\Arrowvert  \Sdiffp{s\ell} \bigr\Arrowvert_2^2 \Bigr] = \E \!\left[ \Biggl\Arrowvert \sum_{\tau=1}^{s \ell} \bphi(a_\tau,\bx_\tau)^{\transp} \biggl( \btheta^\star_{h_\tau}
	- \sum_{h \in [H]} \bar{\bb}_\tau(h) \btheta^\star_h \biggr) \hbb_\tau \otimes \bphi(a_\tau,\bx_\tau) \Biggr\Arrowvert_2^2 \right]
	\leq 4 s \ell + \frac{2  s(s + 1) \ell \gamma}{1 - \gamma} \, .
	\]
	Thus, for all $\delta \in (0, 1)$, for each $s \geq 1$, with probability at least $1-\delta$,
	\[
	\bigl\Arrowvert  \Sdiffp{s\ell} \bigr\Arrowvert_2 =
    \Biggl\Arrowvert \sum_{\tau=1}^{s \ell} \bphi(a_\tau,\bx_\tau)^{\transp} \biggl( \btheta^\star_{h_\tau}
	- \sum_{h \in [H]} \bar{\bb}_\tau(h) \btheta^\star_h \biggr) \hbb_\tau \otimes \bphi(a_\tau,\bx_\tau) \Biggr\Arrowvert_2
	\leq \sqrt{ \frac{1}{\delta} \biggl( 4 s \ell + \frac{2  s(s + 1) \ell \gamma}{1 - \gamma}  \biggr)} \, .
	\]
\end{lemma}

\begin{proof}
	The second inequality follows from the first one via Markov's inequality. We thus only prove the first inequality below.
	
	\emph{Step 1: Preparation.}
	Introduce the scalar-valued random variables
	\[
	z_\tau = \bphi(a_\tau,\bx_\tau)^{\transp} \biggl( \btheta^\star_{h_\tau}
	- \sum_{h \in [H]} \bar{\bb}_\tau(h) \btheta^\star_h \biggr) \,, \qquad \mbox{where} \quad |z_\tau| \leq 2
	\]
	by \Cref{assumption:bound-R-linear} and the fact that $\bar{\bb}_\tau$ is a probability distribution.
	Therefore, by developing the squared norm and by applying the inequalities above, as well as the bound of~\Cref{eq:bbbphi-norm},
	to the diagonal terms only, the target quantity may be rewritten as and bounded by
	\begin{align}
		\nonumber
		\lefteqn{\E \!\left[ \biggl\Arrowvert \sum_{\tau=1}^{s \ell} z_\tau \,\, \hbb_\tau \otimes \bphi(a_\tau,\bx_\tau) \biggr\Arrowvert_2^2 \right]} \\
		\nonumber
		& = \sum_{\tau=1}^{s \ell} \sum_{\tau'=1}^{s \ell} \E \!\left[  z_\tau z_{\tau'} \, \biggl(\hbb_\tau \otimes \bphi(a_\tau,\bx_\tau)\biggr)^{\!\!\transp} \biggl(\hbb_{\tau'} \otimes \bphi(a_{\tau'},\bx_{\tau'})\biggr) \right] \\
		\label{eq:diag+offdiag}
		& \leq \underbrace{\sum_{\tau=1}^{s \ell} \E\bigl[ z_\tau^2 \bigr]}_{\leq 4 s \ell}
		+ 2 \sum_{1 \leq \tau < \tau' \leq s \ell}  \E \Biggl[  z_\tau z_{\tau'} \, \underbrace{\biggl(\hbb_\tau \otimes \bphi(a_\tau,\bx_\tau)\biggr)^{\!\!\transp} \biggl(\hbb_{\tau'} \otimes \bphi(a_{\tau'},\bx_{\tau'})\biggr)}_{\mbox{\tiny to be dealt with}} \Biggr]\,.
	\end{align}
    We recall that we introduced $\cU_t = \sigma\bigl( \bx_{1:t}, \, \bigl( \hbtheta_{s \ell} \bigr)_{s \leq s_t -1} \bigr)$
    for $t \geq 1$.
    Now, we note that the inner product in the cross terms above (marked as ``to be dealt with'')
	is $\sigma(\cU_{\tau'})$--measurable,
	as, in particular, actions $a_\tau$ and $a_{\tau'}$ (since the algorithm proceeds in stages) and estimated beliefs $\hbb_{\tau'}$ and $\hbb_{\tau'}$ are so.
	By the Cauchy-Schwarz inequality and the bound of \Cref{eq:bbbphi-norm}, it is also seen to be smaller than~$1$. Therefore, by the tower
	rule, we further bound the $\tau < \tau'$ cross term above by
	\begin{multline*}
		\E \Biggl[ z_\tau z_{\tau'} \, \overbrace{\biggl(\hbb_\tau \otimes \bphi(a_\tau,\bx_\tau)\biggr)^{\!\!\transp} \biggl(\hbb_{\tau'} \otimes \bphi(a_{\tau'},\bx_{\tau'})\biggr)}^{\mbox{\tiny is $\sigma(\cU_{\tau'})$--measurable}} \Biggr] \\
		= \E \Biggl[  \E \bigl[z_\tau z_{\tau'} \mid \cU_{\tau'} \bigr] \, \underbrace{\biggl(\hbb_\tau \otimes \bphi(a_\tau,\bx_\tau)\biggr)^{\!\!\transp} \biggl(\hbb_{\tau'} \otimes \bphi(a_{\tau'},\bx_{\tau'})\biggr)}_{\leq 1} \Biggr]
		\leq \E \biggl[ \Bigl| \E \bigl[z_\tau z_{\tau'} \mid \cU_{\tau'} \bigr] \Bigr| \biggr]\,.
	\end{multline*}
	To get the claimed bound, we prove that for each $1 \leq \tau < \tau' \leq s \ell$,
	\begin{numcases}{\Bigl| \E \bigl[z_\tau z_{\tau'} \mid \cU_{\tau'} \bigr] \Bigr| \leq }
		\label{eq:ztau}
		\nonumber 2 \, \gamma^{\tau' - \tau} & if $\tau \geq (s_{\tau'} - 1)\ell + 1$, \\
		& \qquad i.e., if $\tau$ belongs to the same stage as $\tau'$, \\
		2 \, \gamma^{\tau' -  (s_{\tau'} - 1) \ell} & if $\tau \leq (s_{\tau'} - 1)\ell$, \nonumber \\
		\nonumber
		& \qquad i.e., if $\tau$ and $\tau'$ belong to different stages,
	\end{numcases}
	which we do next in the subsequent steps of the proof. Then, based on \Cref{eq:ztau}, we obtain
	\begin{align*}
		\lefteqn{\sum_{1 \leq \tau < \tau' \leq s \ell} \Bigl| \E \bigl[z_\tau z_{\tau'} \mid \cU_{\tau'} \bigr] \Bigr|} \\
		& \quad\qquad \leq 2 \,\sum_{s'=1}^{s}  \sum_{\tau'= (s'-1) \ell + 1}^{s' \ell} \Biggl( \, \smash{\overbrace{\sum_{\tau = 1}^{(s_{\tau'} - 1) \ell} 	\, \gamma^{\tau' -  (s_{\tau'} - 1) \ell}}^{= (s' - 1) \ell \, \gamma^{\tau' -  (s' - 1) \ell}} + \overbrace{\sum_{\tau = (s_{\tau'} - 1) \ell + 1}^{\tau'-1} \, \gamma^{\tau' - \tau}}^{\leq \gamma/(1-\gamma)}} \, \Biggr) \\
		& \quad\qquad \leq 2 \,\sum_{s'=1}^{s}  \sum_{\tau'= (s'-1) \ell + 1}^{s' \ell} \biggl((s' - 1) \ell \, \gamma^{\tau' -  (s' - 1) \ell} + \frac{\gamma}{1 - \gamma} \biggr) \leq \frac{2 \gamma}{1 - \gamma} \sum_{s'=1}^{s} s' \ell  = \frac{s(s + 1) \ell \gamma}{1 - \gamma}\,,
	\end{align*}
	from which the first inequality of the lemma follows by \Cref{eq:diag+offdiag}.
    It only remains to show \Cref{eq:ztau}.
	
	\emph{Step 2: Proof of \Cref{eq:ztau}, part 1.}
	In this step, we show that
	\begin{equation}
		\label{eq:Eztau-part1}
		\Bigl| \E \bigl[z_\tau z_{\tau'} \mid \cU_{\tau'} \bigr] \Bigr| \leq
		\sum_{h \in [H]} \P(h_\tau = h \mid \cU_{\tau'}) \sum_{h' \in [H]} \biggl|
		\P_{\{h_\tau = h\}}( h_{\tau'} = h' \mid \cU_{\tau'}) - \P \bigl(h_{\tau'} = h' \mid \cU_{\tau'} \bigr) \biggr| .
	\end{equation}
	In the closed-form expression for $z_\tau z_{\tau'}$, the only quantities that
	are not $\cU_{\tau'}$--measurable are $\btheta^\star_{h_\tau}$ and $\btheta^\star_{h_{\tau'}}$;
	the other terms are $\cU_{\tau'}$--measurable: $\bphi(a_\tau,\bx_\tau)$
	and $\bphi(a_{\tau'},\bx_{\tau'})$, as well as $\bar{\bb}_\tau(h)$ and $\bar{\bb}_{\tau'}(h)$.
	Also, by definition of $\bar{\bb}_{\tau'}$,
	\[
	\E \bigl[ \btheta^\star_{h_{\tau'}} \mid \cU_{\tau'} \bigr]
	= \sum_{h \in [H]} \bar{\bb}_{\tau'}(h) \, \btheta^\star_{h}\,.
	\]
	Therefore,
	\begin{align*}
		\lefteqn{\E \bigl[z_\tau z_{\tau'} \mid \cU_{\tau'} \bigr]} \\
		& = \E\!\left[ \bphi(a_\tau,\bx_\tau)^{\transp} \biggl( \btheta^\star_{h_\tau}
		- \sum_{h \in [H]} \bar{\bb}_\tau(h) \btheta^\star_h \biggr)
		\bphi(a_{\tau'},\bx_{\tau'})^{\transp} \biggl( \btheta^\star_{h_{\tau'}}
		- \sum_{h \in [H]} \bar{\bb}_{\tau'}(h) \btheta^\star_h \biggr) \midl \cU_{\tau'} \right] \\
		& = \E\!\left[ \bphi(a_\tau,\bx_\tau)^{\transp} \btheta^\star_{h_\tau} \,
		\bphi(a_{\tau'},\bx_{\tau'})^{\transp} \biggl( \btheta^\star_{h_{\tau'}}
		- \sum_{h \in [H]} \bar{\bb}_{\tau'}(h) \btheta^\star_h \biggr) \midl \cU_{\tau'} \right] .
	\end{align*}
	We continue the calculation by applying formulas of the form:
	for all $\cU_{\tau'}$--measurable functions $F$,
	\[
	\E \Bigl[ F\bigl(\btheta^\star_{h_\tau},\btheta^\star_{h_{\tau'}}\bigr) \midB  \cU_{\tau'} \Bigr]
	= \sum_{h,h' \in [H]} \P\bigl(h_\tau = h \an h_{\tau'}=h' \mid \cU_{\tau'} \bigr)\,
	F\bigl(\btheta^\star_{h},\btheta^\star_{h'}\bigr)\,.
	\]
	To do so, we consider the functions
	\begin{align*}
		G : \bigl(\btheta^\star_{h},\btheta^\star_{h'}\bigr) & \longmapsto
		\bphi(a_\tau,\bx_\tau)^{\transp} \btheta^\star_{h} \,
		\bphi(a_{\tau'},\bx_{\tau'})^{\transp} \btheta^\star_{h'} \\
		\mbox{and} \qquad \qquad \qquad \qquad
		\btheta^\star_{h} & \longmapsto {-}
		\bphi(a_\tau,\bx_\tau)^{\transp} \btheta^\star_{h} \,
		\bphi(a_{\tau'},\bx_{\tau'})^{\transp} \sum_{j \in [H]} \bar{\bb}_{\tau'}(j) \btheta^\star_j\,,
	\end{align*}
	and get the rewriting
	\begin{align*}
		\E \bigl[z_\tau z_{\tau'} \mid \cU_{\tau'} \bigr]
		= \sum_{h,h' \in [H]} \P\bigl(h_\tau = h & \an h_{\tau'}=h' \mid \cU_{\tau'} \bigr) \, G\bigl(\btheta^\star_h,\btheta^\star_{h'}\bigr) \\[-.25cm]
		& - \sum_{h \in [H]} \P\bigl(h_\tau = h \mid \cU_{\tau'} \bigr) \sum_{j \in [H]} \bar{\bb}_{\tau'}(j) \, G\bigl(\btheta^\star_h,\btheta^\star_j\bigr)\,.
	\end{align*}
	The inequality claimed in \Cref{eq:Eztau-part1}
	follows by noting that \smash{$\bigl| G\bigl(\btheta^\star_h,\btheta^\star_{h'}\bigr) \bigr| \leq 1$}
	(by \Cref{assumption:bound-R-linear}) and by applying \Cref{lm:bayes-event-random-variable}
    (Bayes' formula with expectations conditional to $\sigma$--algebras).
	
	\emph{Step 3: Proof of \Cref{eq:ztau}, part 2.}
	Given the bound of \Cref{eq:Eztau-part1}, it suffices to show that for all $h \in [H]$,
	\begin{multline}
		\label{eq:step3}
		\sum_{h' \in [H]} \biggl|
		\P_{\{h_\tau = h\}}( h_{\tau'} = h' \mid \cU_{\tau'}) - \P \bigl(h_{\tau'} = h' \mid \cU_{\tau'} \bigr) \biggr| \\ \leq
		\begin{cases}
			2 \, \gamma^{\tau' - \tau} & \mbox{if} \ \tau \geq (s_{\tau'} - 1)\ell + 1, \ \mbox{i.e., if $\tau$ belongs to the same stage as $\tau'$}, \\
			2 \, \gamma^{\tau' -  (s_{\tau'} - 1) \ell} & \mbox{if} \ \tau \leq (s_{\tau'} - 1)\ell, \ \mbox{i.e., if $\tau$ and $\tau'$ belong to different stages.}
		\end{cases}
	\end{multline}
	In the case when $\tau \geq (s_{\tau'} - 1)\ell + 1$, i.e., when $s_{\tau} = s_{\tau'}$, we combine a law of total probability with
	\Cref{lm:bayes-event-random-variable} to get the decomposition
	\[
	\P \bigl(h_{\tau'} = h' \mid \cU_{\tau'} \bigr)
	= \sum_{j \in [H]} \P \bigl(h_{\tau} = j \mid \cU_{\tau'} \bigr) \, \P_{\{h_\tau = j\}} \bigl(h_{\tau'} = h' \mid \cU_{\tau'} \bigr)\,.
	\]
	Also, by the HMM conditional independence discussed around \Cref{eq:HHM-indep}, since $\cU_{\tau'}$ is generated by estimates $\hbtheta_{s \ell}$ with $s \leq s_{\tau'} - 1 = s_{\tau - 1}$, 	which are therefore more in the past than $h_\tau$, and by the contexts $\bx_{1:\tau'}$,
we have
	\[
	\P_{\{h_\tau = h\}}( h_{\tau'} = h' \mid \cU_{\tau'}) =
	\P_{\{h_\tau = h\}}( h_{\tau'} = h' \mid \bx_{\tau+1:\tau'})\,.
	\]
	Using successively these equalities (together with a triangle inequality), the quantity of interest may be upper bounded by
	\begin{align*}
		\lefteqn{\sum_{h' \in [H]} \biggl|
			\P_{\{h_\tau = h\}}( h_{\tau'} = h' \mid \cU_{\tau'}) - \P \bigl(h_{\tau'} = h' \mid \cU_{\tau'} \bigr) \biggr|} \\
		& \leq \sum_{j \in [H]} \P \bigl(h_{\tau} = j \mid \cU_{\tau'} \bigr) \sum_{h' \in [H]}
		\biggl|
		\P_{\{h_\tau = h\}}( h_{\tau'} = h' \mid \cU_{\tau'}) - \P_{\{h_\tau = j\}} \bigl(h_{\tau'} = h' \mid \cU_{\tau'} \bigr) \biggr| \\
		& = \sum_{j \in [H]} \P \bigl(h_{\tau} = j \mid \cU_{\tau'} \bigr) \underbrace{\sum_{h' \in [H]}
			\biggl|
			\P_{\{h_\tau = h\}}( h_{\tau'} = h' \mid \bx_{\tau+1:\tau'}) - \P_{\{h_\tau = j\}} \bigl(h_{\tau'} = h' \mid \bx_{\tau+1:\tau'} \bigr) \biggr|}_{
			\leq 2 \gamma^{\tau'-\tau}} \,,
	\end{align*}
	where the $\leq 2 \gamma^{\tau'-\tau}$ bound follows from \Cref{assumption:strong-mixing-hmm}.
	This proves \Cref{eq:step3} in the first case, when $\tau$ belongs to the same stage as $\tau'$.
	
	For the second case, when $\tau \leq (s_{\tau'} - 1)\ell$, i.e., $\tau$ belongs to an stage earlier than the one of $\tau'$,
	we adapt the argument above by also introducing $h_{(s_{\tau'} - 1)\ell}$.
	Two combinations of a law of total probability together with \Cref{lm:bayes-event-random-variable} and a triangle inequality
	entail the following bound on the quantity of interest:
	\begin{align*}
		\lefteqn{\sum_{h' \in [H]} \biggl|
			\P_{\{h_\tau = h\}}( h_{\tau'} = h' \mid \cU_{\tau'}) - \P \bigl(h_{\tau'} = h' \mid \cU_{\tau'} \bigr) \biggr|} \\
		& \leq \sum_{h' \in [H]} \sum_{i \in [H]} \sum_{j \in [H]} \P \bigl(h_{(s_{\tau'} - 1)\ell} = i \mid \cU_{\tau'} \bigr) \,
		\P \bigl(h_{(s_{\tau'} - 1)\ell} = j \mid \cU_{\tau'} \bigr) \\
		& \qquad \qquad \qquad \; \times \biggl|
		\P_{\{h_\tau = h \ \mbox{\tiny and} \ h_{(s_{\tau'} - 1)\ell} = i\}}( h_{\tau'} = h' \mid \cU_{\tau'})
		- \P_{\{h_{(s_{\tau'} - 1)\ell} = j\}} \bigl(h_{\tau'} = h' \mid \cU_{\tau'} \bigr) \biggr| \,.
	\end{align*}
	Given that $\cU_{\tau'}$ is generated by estimates $\hbtheta_{s \ell}$ with $s \leq s_{\tau'} - 1$ which
	only depend on information till round $(s_{\tau'} - 1)\ell$, and by the contexts $\bx_{1:\tau'}$, we have,
	by the HMM conditional independence discussed around \Cref{eq:HHM-indep}, that for all $j,h' \in [H]$,
	\[
	\P_{\{h_{(s_{\tau'} - 1)\ell} = j\}}( h_{\tau'} = h' \mid \cU_{\tau'})
	= \P_{\{h_{(s_{\tau'} - 1)\ell} = j\}}( h_{\tau'} = h' \mid \bx_{(s_{\tau'}-1) \ell +1 :\tau'})\,.
	\]
	Actually, since $\tau \leq (s_{\tau'} -1) \ell$, we even have, with the same arguments, for all $j,h,h' \in [H]$,
	\[
	\P_{\{h_\tau = h \ \mbox{\tiny and} \ h_{(s_{\tau'} - 1)\ell} = j\}}( h_{\tau'} = h' \mid \cU_{\tau'})
	= \P_{\{h_{(s_{\tau'} - 1)\ell} = j\}}( h_{\tau'} = h' \mid \bx_{(s_{\tau'}-1) \ell +1:\tau'})\,.
	\]
	Substituting these equalities in the bound established above, and resorting to \Cref{assumption:strong-mixing-hmm}
	entails
	\begin{align*}
		\lefteqn{\sum_{h' \in [H]} \biggl|
			\P_{\{h_\tau = h\}}( h_{\tau'} = h' \mid \cU_{\tau'}) - \P \bigl(h_{\tau'} = h' \mid \cU_{\tau'} \bigr) \biggr|} \\
		& \leq \sum_{j \in [H]} \sum_{i \in [H]} \P \bigl(h_{(s_{\tau'} - 1)\ell} = i \mid \cU_{\tau'} \bigr) \,
		\P \bigl(h_{(s_{\tau'} - 1)\ell} = j \mid \cU_{\tau'} \bigr) \\
		& \qquad \times \underbrace{\sum_{h' \in [H]} \biggl|
			\P_{\{h_{(s_{\tau'} - 1)\ell} = i\}}( h_{\tau'} = h' \mid \bx_{(s_{\tau'}-1) \ell + 1:\tau'})
			- \P_{\{h_{(s_{\tau'} - 1)\ell} = j\}} \bigl(h_{\tau'} = h' \mid \bx_{(s_{\tau'}-1) \ell + 1:\tau'} \bigr) \biggr|}_{
			\leq 2 \gamma^{\tau'-(s_{\tau'} - 1)\ell}} \,.
	\end{align*}
	This proves \Cref{eq:step3} in the second case, and concludes the proof of
	the lemma.
\end{proof}

\subsubsection{Bound on $\bigl\Arrowvert  \Setap{s\ell} \bigr\Arrowvert_2$}
\label{sec:proof:lm:bound-idiosyncratic}

To bound the term $\bigl\Arrowvert \Setap{s\ell} \bigr\Arrowvert_2$,
we mimic, and simplify, the proof conducted right before
for \Cref{lm:true-vs-belief}: we adapt its Step~1 (and do not need Steps~2 and~3).
Actually, under the stronger noise \Cref{ass:SGnoise},
a LinUCB-type approach as in Appendix~\ref{appendix:simplified-analysis}
could have been followed (i.e., \Cref{lm:linucb-dev}
could have been applied).
We however prefer to mimic and simplify the proof
of \Cref{lm:true-vs-belief}.

\begin{lemma}
	\label{lm:bound-idiosyncratic}
	Under the Assumptions~\ref{assumption:bound-R-linear} and~\ref{assumption:idio-noise}, for all $s \geq 1$,
	\[
	\E \Bigl[ \bigl\Arrowvert \Setap{s\ell} \bigr\Arrowvert_2^2 \Bigr] = \E \!\left[ \Biggl\Arrowvert \sum_{\tau=1}^{s\ell} \eta_{\tau}(a_\tau) \, \hbb_\tau \otimes \bphi(a_\tau,\bx_\tau) \Biggr\Arrowvert_2^2 \right]
	\leq C_{\eta} s \ell \, .
	\]
	Thus, for all $\delta \in (0, 1)$, for each $s \geq 1$, with probability at least $1-\delta$,
	\[
	\bigl\Arrowvert  \Setap{s\ell} \bigr\Arrowvert_2 =
    \Biggl\Arrowvert \sum_{\tau=1}^{s\ell} \eta_{\tau}(a_\tau) \, \hbb_\tau \otimes \bphi(a_\tau,\bx_\tau) \Biggr\Arrowvert_2
	\leq \sqrt{ \frac{1}{\delta} \, C_{\eta}  s \ell } \, .
	\]
\end{lemma}

\begin{proof}
	The second inequality follows from the first one via Markov's inequality.
	For the first inequality, we develop the squared norm and apply the bound of \Cref{eq:bbbphi-norm}:
	\begin{multline*}
		\E \!\left[ \biggl\Arrowvert \sum_{\tau=1}^{s \ell} \eta_{\tau}(a_\tau) \,\, \hbb_\tau \otimes \bphi(a_\tau,\bx_\tau) \biggr\Arrowvert_2^2 \right] \\
		\leq \sum_{\tau=1}^{s \ell} \E\bigl[\eta_{\tau}(a_\tau)^2 \bigr]
		+ 2 \sum_{1 \leq \tau < \tau' \leq s \ell}  \E \! \left[ \eta_{\tau}(a_\tau) \eta_{\tau'}(a_\tau') \Bigl(\hbb_\tau \otimes \bphi(a_\tau,\bx_\tau)\Bigr)^{\!\transp} \Bigl(\hbb_{\tau'} \otimes \bphi(a_{\tau'},\bx_{\tau'})\Bigr)\right] \,.
	\end{multline*}
	
	For the first component of above inequality,
    by \Cref{assumption:idio-noise} and using that the selected action $a_t$ is $\cFall_{t}$--measurable,
    we have, for all $\tau \geq 1$,
    \[
    \E\Bigl[\eta_{\tau}(a_\tau)^2 \,\Big|\, \cFall_{\tau} \Bigr]
    = \sum_{a \in \cA} \Ind{a_t = a} \,\, \E\Bigl[\eta_{\tau}(a)^2 \,\Big|\, \cFall_{\tau} \Bigr] \leq C_{\eta}\,,
    \]
    so that, by the tower rule,
	\[
	\sum_{\tau=1}^{s \ell} \E\Bigl[\eta_{\tau}(a_\tau)^2 \Bigr]
	= \sum_{\tau=1}^{s \ell}\E \biggl[\E\Bigl[\eta_{\tau}(a_\tau)^2 \,\Big|\, \cFall_{\tau} \Bigr] \biggr]
	\leq s \ell C_{\eta} \,.
	\]
	
	For the second sum, fix a pair $1 \leq \tau < \tau' \leq s \ell$.
	We use that $a_\tau$ is measurable w.r.t.\ $\cFall_{\tau}$, and that $\cFall_{\tau'}$
	is generated by $\bx_{\tau}$, the $\eta_\tau(a)$, and other variables, to show that
	the random variables
	$\eta_{\tau}(a_\tau)$ and $\bphi(a_{\tau}, \bx_{\tau})$ are all $\cFall_{\tau'}$--measurable.
	We also have that $\hbb_{\tau}$ and $\hbb_{\tau'}$ are measurable w.r.t.\ $\bx_1,\ldots,\bx_{\tau'}$,
	thus w.r.t.\ $\cFall_{\tau'}$, and by similar arguments, $a_{\tau'}$ and $\bphi(a_{\tau'}, \bx_{\tau'})$
    are also $\cFall_{\tau'}$--measurable.
	Therefore, by the tower rule and by~\Cref{assumption:idio-noise},
	\begin{multline*}
		\E \! \left[ \eta_{\tau}(a_\tau) \eta_{\tau'}(a_\tau') \Bigl(\hbb_\tau \otimes \bphi(a_\tau,\bx_\tau)\Bigr)^{\!\transp} \Bigl(\hbb_{\tau'} \otimes \bphi(a_{\tau'},\bx_{\tau'})\Bigr)\right]  \\
		= \E \Biggl[  \eta_{\tau}(a_\tau) \Bigl(\hbb_\tau \otimes \bphi(a_\tau,\bx_\tau)\Bigr)^{\!\transp} \Bigl(\hbb_{\tau'} \otimes \bphi(a_{\tau'},\bx_{\tau'})\Bigr) \sum_{a \in \cA} \Ind{a_\tau' = a} \,\, \smash{\underbrace{\E \biggl[\eta_{\tau'}(a) \mid \cFall_{\tau'} \biggr]}_{=0}}\Biggr] = 0\,.
	\end{multline*}
	The proof is concluded by collecting all (in)equalities.
\end{proof}

\subsubsection{Bound on $\bigl\Arrowvert \Sbeliefp{s\ell} \bigr\Arrowvert_2$}
\label{sec:proof-lm:belief-vs-estimated}

To bound the term $\bigl\Arrowvert \Sbeliefp{s\ell} \bigr\Arrowvert_2$,
we also mimic, and simplify, the proof
of \Cref{lm:true-vs-belief} conducted in Appendix~\ref{sec:proof:lm:true-vs-belief}:
we do not need its Steps~1 and~2 and we adapt its Step~3.

\begin{lemma}
	\label{lm:bound-belief-vs-estimated}
	Under Assumptions~\ref{assumption:bound-R-linear} and~\ref{assumption:strong-mixing-hmm}, for all $s \geq 1$,
    with probability~$1$,
	\[
    \bigl\Arrowvert  \Sbeliefp{s\ell} \bigr\Arrowvert_2 =
	\Biggl\Arrowvert \sum_{\tau=1}^{s\ell} \sum_{h \in [H]}  \bphi(a_\tau,\bx_\tau)^{\transp} \btheta^\star_h \Bigl( \bar{\bb}_\tau(h)
	- \hbb_\tau(h)  \Bigr) \hbb_\tau \otimes \bphi(a_\tau,\bx_\tau) \Biggr\Arrowvert_2  \leq \frac{2 s \gamma}{1 - \gamma} + \sum_{\tau=1}^{s\ell} \bigl\Arrowvert \bb_\tau - \hbb_\tau \bigr\Arrowvert_1 \,.
	\]
\end{lemma}

\begin{proof}
	By the triangle inequality and the bounds indicated by \Cref{assumption:bound-R-linear} and~\Cref{eq:bbbphi-norm},
	\[
		\Biggl\Arrowvert \sum_{\tau=1}^{s\ell} \sum_{h \in [H]}
		\overbrace{\bphi(a_\tau,\bx_\tau)^{\transp} \btheta^\star_h}^{|\,\cdot\,| \leq 1} \, \Bigl( \bar{\bb}_\tau(h) \,
		- \hbb_\tau(h) \Bigr) \, \overbrace{\hbb_\tau \otimes \bphi(a_\tau,\bx_\tau)}^{\Arrowvert\,\cdot\,\Arrowvert_2 \leq 1}
		\Biggr \Arrowvert_2
		\leq \sum_{\tau=1}^{s\ell} \sum_{h \in [H]} \bigl| \bar{\bb}_\tau(h) - \hbb_\tau(h) \bigr| = \sum_{\tau=1}^{s\ell}
		\bigl\Arrowvert \bar{\bb}_\tau - \hbb_\tau \bigr\Arrowvert_1 \,.
	\]
	The claimed bound is obtained by another triangle inequality and by \Cref{lm:belief-vs-estimated} below: for all $s \geq 1$,
	\[
	\sum_{\tau=1}^{s\ell}
	\bigl\Arrowvert \bar{\bb}_\tau - \hbb_\tau \bigr\Arrowvert_1
	\leq \sum_{\tau=1}^{s\ell} \Bigl( \bigl\Arrowvert \bar{\bb}_\tau - \bb_\tau \bigr\Arrowvert_1 +
	\bigl\Arrowvert \bb_\tau - \hbb_\tau \bigr\Arrowvert_1 \Bigr) \leq
	\frac{2 s \gamma}{1 - \gamma} + \sum_{\tau=1}^{s\ell} \Arrowvert \bb_\tau - \hbb_\tau \bigr\Arrowvert_1 \,. \\[-.7cm]
	\]
\end{proof}

\begin{lemma}
	\label{lm:belief-vs-estimated}
	Under \Cref{assumption:strong-mixing-hmm} (exponentially fast forgetting of initial condition),
	for each $s \geq 1$, with probability~$1$,
	\[
	\sum_{\tau=(s-1)\ell+1}^{s\ell} \bigl\Arrowvert {\bb}_\tau - \bar{\bb}_\tau \bigr\Arrowvert_1
	\leq \frac{2 \gamma}{1 - \gamma} \, .
	\]
\end{lemma}

\begin{proof}
	The proof is a mere adaptation of Step~3 of the
	proof of \Cref{lm:true-vs-belief} located in Appendix~\ref{sec:proof:lm:true-vs-belief}.
	By two applications of the law of total probability and \Cref{lm:bayes-event-random-variable}
	for the second equality,
	by the conditional independence discussed around \Cref{eq:HHM-indep} for the third equality,
	and by a triangle inequality together with \Cref{assumption:strong-mixing-hmm}
	for the final inequality,
	\begin{align*}
		\lefteqn{\sum_{\tau=(s-1)\ell+1}^{s\ell} \bigl\Arrowvert {\bb}_\tau - \bar{\bb}_\tau \bigr\Arrowvert_1} \\
		& = \sum_{\tau = (s-1) \ell + 1}^{s \ell} \sum_{h \in [H]} \,\, \Bigl| \P(h_\tau = h \mid \bx_{1:\tau})
		- \P(h_\tau = h \mid \cU_\tau) \Bigr| \\
		& = \sum_{\tau = (s-1) \ell + 1}^{s \ell} \sum_{h \in [H]} \,\, \Bigg|
		\sum_{i \in [H]} \P \bigl( h_{(s-1)\ell} = i \mid \bx_{1:\tau} \bigr) \, \P_{\{ h_{(s-1)\ell} = i \}}(h_\tau = h \mid \bx_{1:\tau}) \\
		& \qquad \qquad \qquad \qquad \qquad \qquad \smash{- \sum_{j \in [H]} \P \bigl( h_{(s-1)\ell} = j \mid \cU_{\tau} \bigr) \, \P_{\{ h_{(s-1)\ell} = j \}}(h_\tau = h \mid \cU_\tau)
			\Biggl|} \\[.2cm]
		& = \sum_{\tau = (s-1) \ell + 1}^{s \ell} \sum_{h \in [H]} \,\, \Biggl|
		\sum_{i \in [H]} \P \bigl( h_{(s-1)\ell} = i \mid \bx_{1:\tau} \bigr) \, \P_{\{ h_{(s-1)\ell} = i \}}\bigl(h_\tau = h \mid \bx_{(s-1)\ell+1:\tau}\bigr) \\
		& \qquad \qquad \qquad \qquad \qquad \qquad \smash{-  \sum_{j \in [H]} \P \bigl( h_{(s-1)\ell} = j \mid \cU_{\tau} \bigr) \,
			\P_{\{ h_{(s-1)\ell} = j \}}\bigl(h_\tau = h \mid \bx_{(s-1)\ell+1:\tau}\bigr)
			\Biggr|} \\[.2cm]
		& \leq \sum_{\tau = (s-1)\ell + 1}^{s \ell} \sum_{i \in [H]} \sum_{j \in [H]} \P \bigl(h_{(s-1)\ell} = i \mid \bx_{1:\tau} \bigr) \,
		\P \bigl(h_{(s-1)\ell} = j \mid \cU_{\tau} \bigr) \\
		& \qquad \qquad \; \times \underbrace{\sum_{h \in [H]} \biggl|
			\P_{\{h_{(s-1)\ell} = i\}}( h_{\tau} = h \mid \bx_{(s-1) \ell + 1:\tau})
			- \P_{\{h_{(s-1)\ell} = j\}} \bigl(h_{\tau} = h \mid \bx_{(s-1)\ell+1:\tau} \bigr) \biggr|}_{
			\leq 2 \gamma^{\tau-(s-1)\ell}} \,,
	\end{align*}
	from which the stated bound follows, by the formula for geometric sums.
\end{proof}

\subsection{Proof of \Cref{theorem:total-regret-bound}}
\label{appendix:proof-total-regret}
\label{sec:bound-cum-sum}

This section now proves \Cref{theorem:total-regret-bound}
based on \Cref{lm:rt-bound}: as in Appendix~\ref{appendix:simplified-analysis}---namely,
the proof of \Cref{theorem:total-regret-bound-simplified} based on \Cref{lm:est-simple}---,
the final regret bound is basically given by $2$ times the sum of the upper confidence bounds
stated in \Cref{lm:rt-bound}.
We adapt proof of \Cref{theorem:total-regret-bound-simplified} first, to take into account
the staged nature of the strategy of Box~A, and second,
to carefully take care of unions bounds. Indeed, \Cref{lm:est-simple}
offered a deviation bound uniform over time rounds $t \geq 1$
and with a low $\sqrt{\ln(1/\delta)}$ dependency on the
risk level $\delta \in (0,1)$.
On the contrary, \Cref{lm:rt-bound} only provides
deviation bounds for each stage $s \geq 1$ with a $1/\sqrt{\delta_s}$ dependency
on the risk level $\delta_s$ used for that stage.

The confidence bonuses $\epsilon_{t,a}$ considered in \Cref{theorem:total-regret-bound}
correspond to the upper bounds of
\Cref{lm:rt-bound} up to the replacements of $\delta_s$ by $\delta/(4s_T)$ and
of the unknown $\Arrowvert \bb_{\tau} - \hbb_{\tau} \Arrowvert_1$
by their high-probability bounds $\Ubelief(\tau,\delta/2)$.

\begin{proof}
We do not substitute yet the specific values of $\lambda$ and $\ell$ considered
and recall that $T$ is assumed to be known.
We denote by $\epsilon^\dag_{t,a}$ the upper bound read in \Cref{lm:rt-bound} for
the risk $\delta/(4s_T)$ and by substituting the stage $s_t$ to which a round $t \geq 1$ belongs, i.e.,
$\epsilon^\dag_{t,a} = 1+\sqrt{d}/\lambda$ for $1 \leq t \leq \ell$, and for $\ell+1 \leq t \leq T$,
\begin{multline*}
\epsilon^\dag_{t,a} =
\bigl\Arrowvert \bb_t - \hbb_t \bigr\Arrowvert_1 + \biggl\Arrowvert G_{(s_t -1) \ell}^{-1} \Bigl(\hbb_t \otimes \bphi(a, \bx_t)\Bigr) \biggr\Arrowvert_2 \, \Biggl( \lambda \sqrt{H} \, C_{\btheta^{\star}} + 4\sqrt{ \frac{s_T(s_t-1)(1+s_t\gamma)\ell}{\delta(1-\gamma)}}
			+ \sqrt{\frac{4s_T}{\delta} C_{\eta} (s_t-1) \ell } \\
            + \frac{2 (s_t-1) \gamma}{1 - \gamma} + \sum_{\tau=1}^{(s_t-1)\ell} \Arrowvert \bb_{\tau} - \hbb_{\tau} \Arrowvert_1
			 \Biggr)\,.
\end{multline*}
By \Cref{lm:rt-bound} and a union bound, we have the following high-probability
uniform deviation bound: with probability at least $1-\delta/2$,
\[
\forall 1 \leq t \leq T, \qquad\quad
\forall a \in \cA, \qquad \left| \sum_{h \in \cH} \bb_t(h) \bphi(a, \bx_t)^{\transp} \btheta^{\star}_{h}
		- \sum_{h \in \cH} \hbb_t(h) \bphi(a, \bx_t)^{\transp} \hbtheta_{(s_t-1) \ell,h} \right|
\leq \epsilon^\dag_{t,a}\,.
\]
From the guarantee above, we get similar guarantees as in Equations~\eqref{eq:csq-lmsimple-1}--\eqref{eq:csq-lmsimple-2}:
with probability at least $1-\delta/2$,
\begin{align*}
\max_{a \in \cA} \sum_{h \in \cH} \bb_t(h) \, \bphi(a, \bx_t)^{\transp} \btheta^{\star}_{h}
& \leq \max_{a \in \cA} \left\{ \epsilon^\dag_{t,a} + \sum_{h \in \cH} \hbb_t(h) \, \bphi(a, \bx_t)^{\transp} \hbtheta_{(s_t-1) \ell,h} \right\} \\
\mbox{and} \qquad\qquad
\sum_{h \in \cH} \hbb_t(h) \, \bphi(a_t, \bx_t)^{\transp} \hbtheta_{(s_t-1) \ell,h}
& \leq \epsilon^\dag_{t,a_t} + \sum_{h \in \cH} \bb_t(h) \, \bphi(a_t, \bx_t)^{\transp} \btheta^{\star}_{h}\,.
\end{align*}
We now want to replace the terms $\Arrowvert \bb_{\tau} - \hbb_{\tau} \Arrowvert_1$
by $\Ubelief(\tau,\delta/2)$ and to do so, we adapt the results
developed in \Cref{eq:T0-Ubelief-1,eq:T0-Ubelief-2,eq:T0-Ubelief-3},
which only depend on the belief estimation subroutine
and only require \Cref{assumption:bound-belief}.
More precisely, \Cref{eq:T0-Ubelief-1} remains valid:
with probability at least $1-\delta/2$,
\[
\forall t \in [ T_0, \,\, T], \qquad
\bigl\Arrowvert \bb_t - \hbb_t \bigr\Arrowvert_1 \leq \Ubelief(t,\delta/2)\,,
\qquad \mbox{where} \qquad
T_0 \eqdef \max\Bigl\{2, \,\, \lceil T_{\cB,\bM,\nu}\bigl(1 + \ln(2 /\delta)\bigr)
\rceil \Bigr\}\,.
\]
Thus, given that $G_{(s_t -1) \ell} \succeq \lambda \id{dH}$ for $t \geq \ell+1$ and by \Cref{eq:bbbphi-norm},
we also have, $t \geq \ell+1$,
\[
\biggl\Arrowvert G_{(s_t -1) \ell}^{-1} \Bigl(\hbb_t \otimes \bphi(a, \bx_t)\Bigr) \biggr\Arrowvert_2
\leq \frac{1}{\lambda} \, \Bigl\Arrowvert \hbb_t \otimes \bphi(a, \bx_t) \Bigr\Arrowvert_2 \leq \frac{1}{\lambda}\,.
\]
Finally,
taking into account that $\epsilon^\dag_{t,a} = 1+\sqrt{d}/\lambda = \epsilon_{t,a}$ for $1 \leq t \leq \ell$,
we have (whether $T_0$ is larger or smaller than $\ell$) that
with probability at least $1-\delta/2$,
\[
\forall t \in [ T_0, \,\, T], \qquad
\epsilon^\dag_{t,a} \leq \epsilon_{t,a} + 2(T_0 - 1)/\lambda\,,
\]
where the $\epsilon_{t,a}$ are the confidence bonuses considered in
the statement of \Cref{theorem:total-regret-bound}.

The bounds above, together with the same arguments as in \Cref{eq:T0-Ubelief-3,eq:pointer-2-T0REGR}
and the definition of the Box~B algorithm as picking arms $a_t$ maximizing some empirical
upper confidence bounds,
entail that with probability at least $1-\delta$,
\[
R_T = \sum_{t=1}^{T} \left( \max_{a \in \cA} \sum_{h \in [H]} \bb_t(h) \, \bphi(a, \bx_t)^{\transp} \btheta^{\star}_{h} -
			\sum_{h \in [H]} \bb_t(h) \, \bphi(a_t, \bx_t)^{\transp} \btheta^{\star}_{h} \right)
\leq 2 (T_0 - 1) + \sum_{t=T_0}^T \bigl(2\eps_{t, a_t} + 4 (T_0 - 1)/\lambda\bigr)\,.
\]
We now substitute bounds on the $\eps_{t, a_t}$, by replacing $f_t$ in its definition
by the upper bound $f_T$ and by bounding $(s_T-1)\ell < T$ therein,
and also substitute the closed-form expression for $T_0$:
with probability at least $1-\delta$,
\begin{multline}
	\label{eq:3terms}
	R_T \leq (2 + 4T/\lambda) \, T_{\cB,\bM,\nu}\bigl(1 + \ln(2 /\delta)\bigr)
	+  2 \sum_{t=1}^T \Ubelief(t, \delta/2) \\
	\qquad  + 2 \Gsum
    \Biggl( \lambda \sqrt{H} \, C_{\btheta^{\star}}
    + 4 \sqrt{ \frac{T s_T (1+s_T \gamma)}{\delta(1 - \gamma)} }
			+ \sqrt{\frac{4 s_T}{\delta} T C_{\eta}} + \frac{2 s_T \gamma}{1 - \gamma}
			+ \sum_{\tau=1}^{T} \Ubelief(\tau,\delta/2) \Biggr) \,,
\end{multline}
where
\[
\Gsum = \sum_{s=1}^{s_T} \sum_{t = (s-1) \ell + 1}^{\min\{s \ell,T\}} \,\,
\biggl\Arrowvert G_{(s -1) \ell}^{-1} \Bigl(\hbb_t \otimes \bphi(a_t, \bx_t)\Bigr) \biggr\Arrowvert_2\,.
\]
We bound $\Gsum$ by applying \Cref{lm:bound-cum-sum} below to vectors $\by_\tau = \hbb_\tau \otimes \bphi(a_\tau, \bx_\tau)$,
of dimension $dH$ and with Euclidean norm smaller than~$1$
as indicated in \Cref{eq:bbbphi-norm}, till stage $S = s_T$: we get
the deterministic upper bound
\begin{equation}
	\label{eq:Gsum}
	\Gsum \leq \frac{1}{\sqrt{\lambda}}\sqrt{2 \, d H s_T \ell \, (1 + \ell/\lambda) \, \ln\Bigl( 1 + s_T \ell/ (d H \lambda) \Bigr)}\,.
\end{equation}
\Cref{eq:3terms,eq:Gsum} provide the closed-form regret bound claimed in the statement of \Cref{theorem:total-regret-bound}.

It now suffices to show that it is of order $T^{7/8}$ up to logarithmic factors for the
choices $\ell = \lceil T^{3/4} \rceil$ and $\lambda = T^{3/4}$.
Actually, taking $\ell = \lceil T^{a} \rceil$ (thus
$s_T = T/\ell$ is of order $T^{1-a}$) and $\lambda = T^{b}$,
recalling that $T_{\cB,\bM,\nu}$ is a constant,
we have that the regret bound of \Cref{eq:3terms}
is of order, up to logarithmic terms,
\[
T/\lambda + \sqrt{T} + \sqrt{T/\lambda \, (1+\ell/\lambda)} \Bigl( \lambda + T^{3/2}/\ell + \smash{\underbrace{T/\sqrt{\ell} + T/\ell}_{\leq T^{3/2}/\ell}}
+ \sqrt{T} \Bigr)\,,
\]
i.e., of order $T^c$ where $(\,\cdot\,)_+$ denotes the non-negative part and
\[
c = \max\bigl\{ 1-b, \,\, 1/2, \,\, (1-b)/2 + (a-b)_+/2 + \max\{b, \,\, 3/2-a, \,\, 1/2\} \bigr\}\,;
\]
an optimization over $a \in [0,1]$ and $b \in [0,1]$
leads to $a = b = 3/4$ and $c = 7/8$,
which concludes the proof.
\end{proof}

\paragraph{Elliptic potential with staged updates.}
It only remains to prove the following extension of the
classic elliptic potential lemma (see \Cref{lm:bound-cum-sum-simplfied-bis}
in Appendix~\ref{appendix:simplified-analysis}) to updates in stages.

\begin{lemma}
	\label{lm:bound-cum-sum}
	Consider vectors $\by_t \in \R^d$ with $\Arrowvert \by_t \Arrowvert_2 \leq 1$, a parameter $\lambda \geq 1$, and the Gram matrices
	$V_0 = \lambda \id{d}$ and
	\[
	V_t = \lambda \id{d} + \sum_{\tau=1}^{t} \by_\tau \by_\tau^{\transp} \quad \mbox{for} \quad t \geq 1 \,.
	\]
	For all integers $S \geq 1$ and $\ell \geq 1$, we have:
	\begin{align*}
		\sum_{s=1}^{S} \sum_{\tau = (s-1) \ell + 1}^{s \ell}  \bigl\Arrowvert V_{(s -1) \ell}^{-1} \, \by_{\tau} \bigr\Arrowvert_2
		& \leq \frac{1}{\sqrt{\lambda}}\sum_{s=1}^{S} \sum_{\tau = (s-1) \ell + 1}^{s \ell}  \bigl\Arrowvert V_{(s -1) \ell}^{-1/2} \, \by_{\tau} \bigr\Arrowvert_2 \\
		& \leq \frac{1}{\sqrt{\lambda}}\sqrt{2 \, d S \ell \, (1 + \ell/\lambda) \, \log\bigl( 1 + S\ell/ (d\lambda) \bigr)}\,.
	\end{align*}
\end{lemma}

The sum in \Cref{lm:bound-cum-sum} differs from the sum bounded in \Cref{lm:bound-cum-sum-simplfied-bis}
in two ways: first, it involves terms of the form
\[
\bigl\Arrowvert V_{t-1}^{-1} \, \by_t \bigr\Arrowvert_2 \quad \mbox{instead of} \quad
\bigl\Arrowvert V_{t-1}^{-1/2} \, \by_t \bigr\Arrowvert_2\,,
\]
which leads to an additional $1/\sqrt{\lambda}$ multiplicative term in our bound,
and second, the matrices $V$ are actually ``frozen'' within stages,
which entails the other additional multiplicative factor $\sqrt{1 + \ell/\lambda}$.
The proof below focuses on these two modifications.

\begin{remark}
\citet{Elliptical2020} provide some general study of the sums
\[
\sum_{t=1}^{T} \bigl\Arrowvert V_{t-1}^{-p/2} \by_{t} \bigr\Arrowvert_2\,,
\]
for $p \in (0,+\infty)$.
For $p=2$, as in \Cref{lm:bound-cum-sum} up to staging, they obtain an upper bound of order $\sqrt{Td/\lambda}$.
This corresponds, up to logarithmic factors and up to the $\sqrt{1 + \ell/\lambda}$ term due to staging, to the right-most term of \Cref{lm:bound-cum-sum}.
Therefore, the first inequality of \Cref{lm:bound-cum-sum}, while relying on the simple lower bound $V_{t-1} \succeq \lambda \id{d}$ (see the proof above), looks sharp enough.
\end{remark}

\begin{proof}
We use $V_{t-1} \succeq \lambda \id{d}$ for all $t \geq 1$ to get,
for all $s \geq 1$ and all $\tau \geq 1$,
\begin{equation}
	\label{eq:bd-V-l}
	\bigl\Arrowvert V_{(s -1) \ell}^{-1} \, \by_{\tau} \bigr\Arrowvert_2
	= \bigl\Arrowvert V_{(s -1) \ell}^{-1/2} \, V_{(s -1) \ell}^{-1/2} \, \by_{\tau} \bigr\Arrowvert_2
	\leq \frac{1}{\sqrt{\lambda}} \bigl\Arrowvert V_{(s -1) \ell}^{-1/2} \, \by_{\tau} \bigr\Arrowvert_2\,,
\end{equation}
which yields the first inequality stated in the lemma.
We prove below that for all $s \geq 1$,
\begin{equation}
	\label{eq:bound-within-stage}
	\sum_{\tau = (s-1) \ell + 1}^{s \ell} \bigl\Arrowvert V_{(s -1) \ell}^{-1/2} \,\, \by_{\tau} \bigr\Arrowvert_2  \leq \sqrt{2 \ell \, (1 + \ell/\lambda) \,
		\ln \frac{\det(V_{s \ell})}{\det\bigl(V_{(s -1) \ell}\bigr)}}\,.
\end{equation}
The second inequality then follows from \Cref{eq:bound-within-stage} and the application
of a Cauchy-Schwarz inequality:
\begin{align*}
	\sum_{s=1}^{S} \sum_{\tau = (s-1) \ell + 1}^{s \ell}  \bigl\Arrowvert V_{(s -1) \ell}^{-1/2} \, \by_{\tau} \bigr\Arrowvert_2
	& \leq \sum_{s=1}^{S}
	\sqrt{2 \ell \, (1 + \ell/\lambda) \,
		\ln \frac{\det(V_{s \ell})}{\det\bigl(V_{(s -1) \ell}\bigr)}} \\
	& \leq \sqrt{2 S \ell \, (1 + \ell/\lambda) \,
		\sum_{s=1}^{S} \ln \frac{\det(V_{s \ell})}{\det\bigl(V_{(s -1) \ell}\bigr)}}
= \sqrt{2 S \ell \, (1 + \ell/\lambda) \,
		 \ln \frac{\det(V_{S \ell})}{\det\bigl(V_{0}\bigr)}} \,,
\end{align*}
together with the fact that
$\det(V_{0}) = \lambda^d$ and that the upper bound
$d \log(\lambda + S\ell/ d)$
on $\ln\bigl(\det(V_{S\ell})\bigr)$ is given by
\Cref{lm:determinant-trace-inequality-bis}.
We are thus only left to prove \Cref{eq:bound-within-stage}.

To do so, we show below that
\begin{equation}
	\label{eq:V-1}
	\forall j \in [\ell-1], \qquad
	V_{(s -1) \ell}^{-1} \preceq (1 + \ell/\lambda) \, V^{-1}_{(s -1) \ell + j}\,,
\end{equation}
which, keeping in mind that $\Arrowvert V^{-1/2} \by \Arrowvert_2 = \sqrt{\by^{\transp} V \by}$, directly entails that
\[
\sum_{\tau = (s-1) \ell + 1}^{s \ell} \bigl\Arrowvert V_{(s -1) \ell}^{-1/2} \,\, \by_{\tau} \bigr\Arrowvert_2
\leq \sqrt{1 + \ell/\lambda}
\sum_{\tau = (s-1) \ell + 1}^{s \ell} \bigl\Arrowvert V_{\tau-1}^{-1/2} \,\, \by_{\tau} \bigr\Arrowvert_2\,.
\]
The bound of
\Cref{eq:bound-within-stage} is then obtained via
the arguments between
\Crefrange{eq:sum-per-round-raw}{eq:reminder-end} in the proof of \Cref{lm:bound-cum-sum-simplfied-bis}, applied
within a stage, i.e., within the $\ell$ rounds from $(s-1) \ell + 1$ to $s\ell$.

To prove~\Cref{eq:V-1}, we recall that
\[
V_{(s -1) \ell + j} - V_{(s -1) \ell} = \sum_{\tau=(s -1) \ell+1}^{(s -1) \ell+j} \by_\tau \by_\tau^{\transp}\,,
\]
where $\by_\tau \by_\tau^{\transp} \preceq \id{d}$ since $\Arrowvert \by_\tau \Arrowvert_2 \leq 1$.
We also recall that by definition, $V_{(s -1) \ell} \succeq \lambda \id{d}$.
Therefore,
\[
V_{(s -1) \ell + j} - V_{(s -1) \ell} \preceq j \, \id{d} \preceq (j/\lambda) \, V_{(s -1) \ell}\,,
\qquad \mbox{thus} \quad
V_{(s -1) \ell + j} \preceq (1+\ell/\lambda) \, V_{(s -1) \ell}\,,
\]
which implies \Cref{eq:V-1} after inverting both sides.
\end{proof}

\subsection{Handling Regret Defined in Terms of Actual Rewards}
\label{app:true-regret}

The end of Section~\ref{sec:regret-def} stated that
the results achieved in this article go beyond
the mere case of the pseudo-regret
\[
R_T = \sum_{t=1}^{T} \max_{a \in \cA} \sum_{h \in [H]} \bb_t(h) \, \bphi(a, \bx_t)^{\transp} \btheta^{\star}_{h} - \\
\sum_{t=1}^{T} \sum_{h \in [H]} \bb_t(h) \, \bphi(a_t, \bx_t)^{\transp} \btheta^{\star}_{h}
\]
and also yield a control of a regret defined in terms of actual rewards:
\[
R_T^{\mbox{\rm \tiny actual}} = \sum_{t=1}^{T} r_t(a^{\star}_t) - \sum_{t=1}^{T} r_t(a_t)\,,
\qquad \mbox{where} \qquad 
	a^{\star}_t  \in \argmax_{a \in \cA} \sum_{h \in [H]} \bb_t(h) \, \bphi(a, \bx_t)^{\transp} \btheta^{\star}_{h}\,.
\]
Below, we actually sketch the proof that $R_T^{\mbox{\rm \tiny actual}}$ is close to $R_T$
with high probability, up to an additive term of order $T^{5/8}$.
That proof sketch actually follows the (long and complex) proof
provided above for \Cref{theorem:total-regret-bound}.

\paragraph{First step: noise terms.}
We first control the noise terms. 
It suffices to mimic \Cref{lm:bound-idiosyncratic} and get a scalar version thereof, where 
terms of the form $\hbb_\tau \otimes \bphi(a_\tau,\bx_\tau)$ are replaced by the scalar multiplier $1$.
This shows that the sums of the noise terms are small, and more precisely, that 
with probability at least $1-\delta/3$, 
\begin{align*}
\sum_{t=1}^{T} r_t(a^{\star}_t)
\qquad \mbox{is } & \sqrt{T/ \delta} \mbox{--close to} \qquad
\sum_{t=1}^{T} \bphi(a^{\star}_t, \bx_t)^{\transp} \btheta^{\star}_{h_t} \\
\mbox{and} \qquad 
\sum_{t=1}^{T} r_t(a_t)
\qquad \mbox{is } & \sqrt{T/ \delta} \mbox{--close to} \qquad
\sum_{t=1}^{T} \bphi(a_t, \bx_t)^{\transp} \btheta^{\star}_{h_t} \,.
\end{align*}

\paragraph{Second step: conditional expectations of $h_t$.}
A scalar version of \Cref{lm:true-vs-belief}, based on stage lengths $\ell = \lceil T^{3/4} \rceil$,
ensures that with probability at least $1-\delta/3$, 
\[
\sum_{t=1}^{T} \bphi(a_t, \bx_t)^{\transp} \btheta^{\star}_{h_t}
\qquad \mbox{is } T^{5/8}/\sqrt{\delta} \mbox{--close to} \qquad
\sum_{t=1}^{T} \bphi(a_t, \bx_t)^{\transp} \sum_{h \in [H]} \bar{\bb}_t(h) \, \btheta^{\star}_{h}\,.
\]
The same proof technique as in \Cref{lm:true-vs-belief}, but without stages (and in a scalar fashion), 
similarly entails that with probability at least $1-\delta/3$, 
\[
\sum_{t=1}^{T} \bphi(a^{\star}_t, \bx_t)^{\transp} \btheta^{\star}_{h_t}
\qquad \mbox{is } \sqrt{T/\delta} \mbox{--close to} \qquad
\sum_{t=1}^{T} \bphi(a^{\star}_t, \bx_t)^{\transp} \sum_{h \in [H]} \bb_t(h) \, \btheta^{\star}_{h}\,;
\]
we use here that $a^{\star}_t$ is $\sigma(\bx_{1:t})$--measurable.

\paragraph{Third step: relating posterior probabilities.}
Finally, the difference between
\[
\sum_{t=1}^{T} \bphi(a_t, \bx_t)^{\transp} \sum_{h \in [H]} \bar{\bb}_t(h) \, \btheta^{\star}_{h}
\qquad \mbox{and} \qquad
\sum_{t=1}^{T} \bphi(a_t, \bx_t)^{\transp} \sum_{h \in [H]} \bb_t(h) \, \btheta^{\star}_{h}
\]
is bounded by
\[
\sum_{t=1}^T \bigl\Arrowvert {\bb}_t - \bar{\bb}_t \bigr\Arrowvert_1\,,
\]
which is exactly the quantity that \Cref{lm:belief-vs-estimated} controls: it is of order $T^{1/4}$ given the value
$\ell = \lceil T^{3/4} \rceil$ picked. 

\paragraph{Conclusion.}
Collecting all bounds, we see that
$R_T^{\mbox{\rm \tiny actual}}$
is $T^{5/8}\sqrt{\delta}$--close to $R_T$, with probability at least $1-\delta$.

\clearpage
\section{Numerical Simulations}
\label{appendix:simulation}

The focus of this article is primarily theoretical: the simulations are intended mainly to illustrate the practical behavior of the proposed algorithms, specifically their convergence and their performance relative to relevant baselines. Accordingly, the purpose of this appendix is threefold.

\paragraph{Goal~1: Illustrate the impact of taking into account the latent dynamics.}
First, we compare the belief-based LinUCB strategies (Box~A, and its special case Box~B) with a baseline given by plain LinUCB (introduced by \citealp{abbasi2011improved} and restated in Box~D of this appendix).
This baseline ignores the latent state dynamics and treats the observed contexts as if they were directly sampled from a stochastic environment
and the rewards as if they depended only on the observed contexts and actions.
Alternative baselines we could most immediately think of are based on latent-state bandit models but do not incorporate contextual information into the reward model:
this includes, for example, the works by \citet{RegimeBandits-2021} and \citet{spectral-pomdps-2016}, as well as
\citealp{nelson2022ContextLatent} in their original form.

\paragraph{Goal~2: Illustrate the impact of the hyper-parameters $\ell$ and confidence bonuses $\epsilon_{t,a}$.}
The belief-based strategies considered in this article (stated in Box~A, with Box~B being a special case for stages of length $\ell = 1$)
depend on two hyper-parameters: the stage lengths $\ell \geq 1$, and the form of the confidence bonuses.
We considered two forms in this article: one in \Cref{eq:epsilonta-mostcomplex} to deal with rewards
stemming from the most complex reward model~\eqref{eq:reward-our}, where rewards depend directly on the latent states,
and one in Appendix~\ref{appendix:simplified-analysis}
suited to the simpler reward model~\eqref{eq:nelson-model-orig},
where rewards are functions of the beliefs.
While the theory developed in this article did not consider the case of $\ell \geq 2$ and
the confidence bonuses of Appendix~\ref{appendix:simplified-analysis},
in the experiments, we go beyond these original designs and provide a more comprehensive study of performance
according to these hyper-parameters.
The goal is to disentangle the effect of the form of the confidence bonuses $\epsilon_{t,a}$ from the effect of the update schedule $\ell$.
Since both the staging scheme and the larger confidence bonuses of \Cref{eq:epsilonta-mostcomplex}
are introduced in the theory to cope with the direct dependence of rewards on the latent states, these simulations help determine whether they are merely technical tools for the analysis or whether they also provide a practical advantage.

\paragraph{Goal~3: Provide a realistic application with a larger-scale HMM.}
Third, these simulations also illustrate that the proposed algorithms can be run in a realistic application setting. The simulations are carried out on a light computing setup ($8$ cores, $16$ threads, $32$~GB RAM) and without using a GPU. Even so, the hidden state dynamics in our simulations are already nontrivial: the underlying HMM (described in detail below) features $2$ hidden states and emission distributions over $20$ different values. By comparison, \citet{nelson2022ContextLatent} consider a smaller simulated setting with $2$ hidden states and $4$ emission values. Similarly, small emission spaces are also used in the simulations of \citet{RegimeBandits-2021} and \citet{spectral-pomdps-2016}.

\subsection{Data Preparation; Variables; Latent States and Contexts}
\label{sec:simulation-data-prep}

We use partially simulated but realistic data. A brief summary of the hypothetical simulation background in banking industry is the following: a bank aims to optimize its marketing strategy for a credit product. Each potential client is described by a context (a client profile), and the bank chooses one of three marketing actions: \emph{Call}, \emph{Email}, or \emph{No action}. Rewards are generated by a latent-state-dependent linear function with additive noise. We assume that the environment switches between two unobserved economic latent states: \emph{inflation} and \emph{recession}. The latent state affects both the distribution of client profiles and the reward function. For instance, during a recession, clients tend to have lower revenue and may respond differently to marketing actions than during inflation. The objective is to maximize cumulative rewards by adapting the actions over time.

\paragraph{Data set.}
Our simulations are based on the ``Default of Credit Card Clients'' dataset from UCI Machine Learning Repository \citep{UCI2016DefaultCR}, originally provided by \citet{Yeh2009creditcard}. The dataset is distributed under the Creative Commons Attribution 4.0 International (CC BY 4.0) license and was designed to benchmark algorithms for predicting credit card default probabilities. It contains socio-demographic variables, debt levels, payment histories, and a binary target indicating whether a client defaulted in the next month.

For the purposes of these simulations, each row is interpreted as a potential marketing opportunity. We discard some variables and create additional features as described below.  Our reprocessing is close to the simulation setup of \citet[Appendix~F]{CBwK-LP-2022}, who study a ``market share expansion for loans'' application based on the same dataset.

\paragraph{Variables kept.} We keep the following variables, with mild preprocessing:.
\begin{itemize}
	\item \emph{Age}---client age in years at the time of the campaign, discretized into five levels using cutoffs $27$, $31$, $37$, and $43$ (level~1
denoting the younger age category and level~5 the oldest);
	\item \emph{Education}---client's education level, regrouped into four levels (others, high school, university, and graduate school);
	\item \emph{Marital status}---client's marital status with three levels (single, married, and other).
\end{itemize}

\paragraph{Variables created.} We construct two additional variables:
\begin{itemize}
	\item \emph{Revenue}---a proxy of client's revenue derived from the current debt level in the original dataset by multiplying the latter by $0.2$; the value obtained is further discretized into four levels with cutoffs 10K, 36K, and 54K (level~1 denoting the lowest revenue category and level~$4$ being the highest). \smallskip
	
	\item \emph{Risk score}---following \citet[Appendix~F]{CBwK-LP-2022}, we fit a probability of default using a XGBoost model \citep{Chen2016XGBoost} on the original dataset to estimate the probability that a client defaults on the loan in the following month. Full details of the XGBoost hyper-parameters are provided in \citet[Appendix~F]{CBwK-LP-2022} and the number of trees selected by cross validation equals now $1,\!111$ (this number is slightly different than in the reference
due to differences in the versions of Python and of some packages).
We then discretize the predicted probabilities of default into five equal quantile bins to obtain a finite risk score, with level~$1$ denoting the lowest risk and level~$5$ the highest. (It is unnecessary here to recalibrate the predicted probabilities by dividing them by $4$ and cap the resulting value to $20\%$, as in the reference, since we are only interested in the resulting quantile-based discretization.)
\end{itemize}

\paragraph{Latent states.}
We consider two latent states, referred to as \emph{Inflation} ($h=1$) and \emph{Recession} ($h=2$). As in the theoretical model, the latent state affects both the context distribution and the reward function.

The contexts $\bx_t$ are given by the variables described above, i.e., are 5--uples
of the form (Age, Education, Marital status, Revenue, Risk score).
We assume that the latent state only influences the last two components $\bx'_t$, formed by Revenue and Risk score,
and that the first three components, denoted by $\bw_t$, are independent of the
latent state $h_t$ conditional to $\bx'_t$.
We also assume that the learner is aware of this fact, which entails (by an application of the Bayes' rule) that
\[
\forall t \geq 1, \ \
\forall h \in [2], \qquad
\P(h_t=h \mid \bx_{1:t}) =  \P(h_t=h \mid \bx'_{1:t})\,,
\]
so that belief estimates should be computed based solely on the subcontexts $\bx'_{1:t}$.
The estimation procedures and regret guarantees remain unchanged up to this adaptation.
Assuming this piece of knowledge is reasonable in practice, as the bank usually has prior domain knowledge about which
covariates are likely to reflect the latent economic regime.

However, rewards may and will depend on the entire context $\bx_t = (\bw_t,\bx_t')$.
The modeling above is handy for computational reasons: subcontexts $\bx'_t$ only take $4 \times 5 = 20$ different values
while full contexts $\bx_t$ take $5 \times 4 \times 3 \times 20 = 1,\!200$ different values;
the belief estimation procedure would be computationally expensive to implement
if all five components of the contexts truly depended on the latent states.

We now indicate how we pick the first three components $\bw_t$ independently of the latent states $h_t$
conditional to the last two components $\bx'_t$.

\paragraph{Context generation.}
The generation of the last two components $\bx'_t$ through a HMM
is described in Appendix~\ref{sec:simu-HMM-param} below.
We rather explain here how we generated $\bw_t$ based on $\bx'_t$:
we do so only based on the value of $\bx'_t$, which justifies
the conditional independence to the latent state $h_t$.

More precisely,
the dataset constructed above contains about $30,\!000$ statistical units
and therefore preserves realistic empirical dependence among the original covariates.
At round $t \geq 1$, given the reduced context $\bx'_t$ generated, we
identify the (thousands of) statistical units in the data set with the same values
of Revenue and Risk score and sample (with replacement) one such unit: its values
for Age, Education and Marital status form the first three components $\bw_t$
of the full context $\bx_t = (\bw_t,\bx'_t)$.

\subsection{Parameters for the HMM and the Reward Functions}
\label{sec:simu-HMM-param}

We picked somewhat arbitrary HMM and reward parameters: the goal is not to claim that these parameters are calibrated to a specific market, but rather to create a simulated environment in which the latent state has a visible effect on both contexts and rewards.

\paragraph{HMM Parameters.}
We recall that there are two values for the latent state,
\emph{Inflation} ($h=1$) and \emph{Recession} ($h=2$).
We consider the transition matrix
\[
\bM =
\begin{bmatrix}
	\P(\text{Inflation} \to \text{Inflation}) & \P(\text{Inflation} \to \text{Recession}) \\
	\P(\text{Recession} \to \text{Inflation}) & \P(\text{Recession} \to \text{Recession})
\end{bmatrix}
=
\begin{bmatrix}
	0.85 & 0.15\\
	0.2 & 0.8
\end{bmatrix}
\]
with unique stationary distribution $\bpi = (4/7, \, 3/7)$,
as well as the emission distributions $\nu_1,\,\nu_2$ reported in~\Cref{tab:simu-emission-prob}.
\begin{table}[t!]
	\centering
	\caption{Emission probabilities $\nu_h$ by latent state $h \in [2]$:
for readability, the $20$ values are presented through $5 \times 4$ tables,
with rows corresponding to Risk score levels and columns, to Revenue levels. \\[-.5cm] }
	\label{tab:simu-emission-prob}
	\begin{tabular}{lcr}
    \begin{minipage}{0.48\textwidth}
		\centering
		\[
		\begin{array}{ccccc}
			\multicolumn{5}{c}{\text{\emph{Inflation} ($h=1$)}}\\
			\hline
			\text{Risk} \, \backslash \, \text{Revenue}
			& 1 & 2 & 3 & 4 \\
			\hline
			1 & 0.0407 & 0.1242 & 0.0956 & 0.0665 \\
			2 & 0.0048 & 0.0760 & 0.0628 & 0.0661 \\
			3 & 0.0697 & 0.0109 & 0.0664 & 0.1634 \\
			4 & 0.0010 & 0.0058 & 0.0025 & 0.0967 \\
			5 & 0.0050 & 0.0310 & 0.0107 & 0.0002 \\
            \hline
		\end{array}
		\]
	\end{minipage} & \ \ &
	\begin{minipage}{0.48\textwidth}
		\centering
		\[
		\begin{array}{ccccc}
			\multicolumn{5}{c}{\text{\emph{Recession} ($h=2$)}}\\
			\hline
			\text{Risk} \, \backslash \, \text{Revenue}
			& 1 & 2 & 3 & 4 \\
			\hline
			1 & 0.0093 & 0.0373 & 0.0202 & 0.0089 \\
			2 & 0.0420 & 0.0320 & 0.0098 & 0.0361 \\
			3 & 0.0855 & 0.0341 & 0.0265 & 0.0070 \\
			4 & 0.1101 & 0.0636 & 0.0913 & 0.0228 \\
			5 & 0.0986 & 0.0480 & 0.1423 & 0.0746 \\
            \hline
		\end{array}
		\]
	\end{minipage}
    \end{tabular}
\end{table}

Given the expression for $\bM$, the latent state is persistent: remaining in the same latent state is substantially more likely than switching
to the other state.
The emission probabilities reflect the intended economic interpretation: the inflation state is associated more strongly with lower risk and higher revenue profiles, whereas the recession state places more probability mass on higher risk and lower revenue profiles.

We draw the initial latent state $h_1 \sim \bpi$, and then, successively, for all $t \geq 1$,
draw the reduced context $\bx'_t$ from the emission distribution $\nu_{h_t}$ and
draw the next latent state $h_{t+1}$ according to the distribution
read in the row $h_t$ of $\bM$.

\paragraph{Reward model.}
We recall that the action space $\cA$ contains three actions:
$\cA = \{\text{No action},\ \text{Call},\ \text{Email} \}$.
At round $t \geq 1$,
given a latent state $h_t$, a context $\bx_t$, and an action $a$, the realized reward is generated
according to \Cref{eq:reward-our}, in the specific form
\[
r_t(a) = \underbrace{\bphi(a, \bx_t)^{\transp} \btheta^{\star}_{h_t}}_{= \mu_{h_t}(a, \bx_t)} + \eta_t(a)\,,
\qquad \mbox{where} \qquad \eta_t(a) \sim \mathcal{N}(0, 0.2)\,.
\]
The linear expected part $\mu_{h_t}(a, \bx_t)$ of the reward is defined through coefficients
listed in \Cref{table:reward-inflation-coef,table:reward-recession-coef}, as follows:
for all $h \in [2]$, action $a \in \cA$, and context $\bx$,
\begin{align*}
	\mu_h(a, \bx) = & \ \ \beta^{h}_{0}
	+ \beta^{h}_{\text{age}}\bigl(\text{Age}\bigr)
	+ \beta^{h}_{\text{edu}}\bigl(\text{Education}\bigr)
	+ \beta^{h}_{\text{mar}}\bigl(\text{Marital Status}\bigr)
	+ \beta^{h}_{\text{rev}}\bigl(\text{Revenue}\bigr)
	+ \beta^{h}_{\text{risk}}\bigl(\text{Risk score}\bigr)\\
	& + \Ind{a=\text{Call}}\beta^{h}_{\text{call}}
	+ \Ind{a=\text{Email}}\beta^{h}_{\text{email}} \\
	& + \Ind{a=\text{Call}}\beta^{h}_{\text{call} \times \text{risk}}\bigl(\text{Risk score}\bigr)
	+ \Ind{a=\text{Email}}\beta^{h}_{\text{email} \times \text{risk}}\bigl(\text{Risk score}\bigr) \\
	& + \Ind{a=\text{Call}}\beta^{h}_{\text{call} \times \text{rev}}\bigl(\text{Revenue}\bigr)
	+ \Ind{a=\text{Email}}\beta^{h}_{\text{email} \times \text{rev}}\bigl(\text{Revenue}\bigr)\,.
\end{align*}
For a context $\bx = (i_1,i_2,i_3,i_4,i_5)$,
coefficients marked like $\beta^{h}_{\text{rev}}\bigl(\text{Revenue}\bigr)$
and $\beta^{h}_{\text{risk}}\bigl(\text{Risk Score}\bigr)$
refer to $\beta^{h}_{\text{rev}}(i_4)$
and $\beta^{h}_{\text{risk}}(i_5)$, respectively.

The coefficients are listed in \Cref{table:reward-inflation-coef} and \Cref{table:reward-recession-coef}.
In words, the tables specify, for each latent state, an intercept, additive effects for the context variables,
additive effects for the active marketing actions, and action-specific additive interactions with Risk score and Revenue.
For the passive action (No action), only the intercept and additive context effects remain.

\begin{table}[p]
	\caption{Coefficients for the linear reward model under inflation $h=1$.}
	\label{table:reward-inflation-coef}
	\centering
	\begin{tabular}{lccccc}
		\toprule
		Intercept                              &  \multicolumn{5}{c}{$0.05$} \\
		\midrule
		Context variables                  &  \multicolumn{5}{c}{Coefficients for each level} \\
		&  Level 1  &  Level 2   &  Level 3     &  Level 4     &  Level 5 \\
		\cline{2-6} \\
		\var{Risk score}                 &  $0.067$       &  $0.05$       &   $0.033$  &  $0.017$   &  $0$ \\
		\var{Revenue}       &  $0$         &  $0.017$         &  $0.025$    &  $0.033$  &  \\
		\var{Age}            &  $0$        &  $0.008$       & $0.017$   &  $0.008$  &   $0$ \\
		\var{Education}                 &  $0$         & $0$        & $0.017$    & $0.033$   &   \\
		\var{Marital status}           &  $0$        &  $0$      &    $0.033$  &  & \\
		\midrule
		Action Variables                   &  \multicolumn{5}{c}{Single coefficient} \\
		\var{Call}     &  \multicolumn{5}{c}{$-0.2$} \\
		\var{Email}     &  \multicolumn{5}{c}{$0.025$} \\
		\midrule
		Action $\times$ Risk Score                  &  \multicolumn{5}{c}{Coefficients for each level} \\
		& Risk Score: 1  &  Risk Score: 2   &  Risk Score: 3     &  Risk Score: 4    &  Risk Score: 5 \\
		\cline{2-6} \\
		\var{Call}       &  $-0.2$         &  $-0.1$         &  $0$    &  $0.3$  &  $0.4$ \\
		\var{Email}      &  $0.05$       &  $0.025$       &   $0$  &  $-0.2$   &  $-0.25$ \\
		\midrule
		Action $\times$ Revenue                  &  \multicolumn{5}{c}{Coefficients for each level} \\
		& Revenue: 1  & Revenue: 2   &  Revenue: 3     &  Revenue: 4\\
		\cline{2-6} \\
		\var{Call}       &  $0.2$         &  $0$         &  $0$    &  $-0.2$ \\
		\var{Email}      &  $-0.1$       &  $0$       &   $0$  &  $0.025$ \\
		\bottomrule
	\end{tabular}
\end{table}

\begin{table}[p]
	\caption{Coefficients for the linear reward model under recession $h=2$.}
	\label{table:reward-recession-coef}
	\centering
	\begin{tabular}{lccccc}
		\toprule
		Intercept                              &  \multicolumn{5}{c}{$0.017$} \\
		\midrule
		Context variables                  &  \multicolumn{5}{c}{Coefficients for each level} \\
		&  Level 1  &  Level 2   &  Level 3     &  Level 4     &  Level 5 \\
		\cline{2-6} \\
		\var{Risk score}                 &  $0.067$       &  $0.05$       &   $0.033$  &  $0.017$   &  $0$ \\
		\var{Revenue}       &  $0$         &  $0.017$         &  $0.025$    &  $0.033$  &  \\
		\var{Age}            &  $0$        &  $0.008$       & $0.017$   &  $0.008$  &   $0$ \\
		\var{Education}                 &  $0$         & $0$        & $0.017$    & $0.033$   &   \\
		\var{Marital status}           &  $0$        &  $0$      &    $0.033$  &  & \\
		\midrule
		Action Variables                   &  \multicolumn{5}{c}{Single coefficient} \\
		\var{Call}     &  \multicolumn{5}{c}{$0.15$} \\
		\var{Email}     &  \multicolumn{5}{c}{$-0.1$} \\
		\midrule
		Action $\times$ Risk Score                  &  \multicolumn{5}{c}{Coefficients for each level} \\
		& Risk Score: 1  &  Risk Score: 2   &  Risk Score: 3     &  Risk Score: 4    &  Risk Score: 5 \\
		\cline{2-6} \\
		\var{Call}       &  $0.25$         &  $0.2$         &  $0$    &  $-0.1$ & $-0.3$\\
		\var{Email}      &  $0.15$       &  $0.1$       &   $0$  &  $0.1$  & $0.05$\\
		\midrule
		Action $\times$ Revenue                  &  \multicolumn{5}{c}{Coefficients for each level} \\
		& Revenue: 1  & Revenue: 2   &  Revenue: 3     &  Revenue: 4\\
		\cline{2-6} \\
		\var{Call}       &  $-0.2$         &   $0$         &  $0$    &  $0.1$ \\
		\var{Email}      &  $0.1$       &   $0$     &   $0$  &  $0.05$ \\
		\bottomrule
	\end{tabular}
\end{table}

\subsection{Algorithms and Hyper-parameters Thereof}
We consider $T = 50,000$ and use the first $250$ rounds as a warm start for all algorithms considered,
during which actions are drawn uniformly at random.
We compare the belief-based LinUCB strategies (Box~A, and its special case Box~B) with a baseline given by plain LinUCB (introduced by \citealp{abbasi2011improved} and restated next in Box~D).

\subsubsection{Staged LinUCB Strategies on Estimated Beliefs}
\label{sec:grids}

These strategies correspond to Box~A and Box~B.

Three hyper-parameters need to be
set: the stage lengths $\ell \geq 1$ (with $\ell=1$ corresponding to per-round
updates as in Box~B), the regularization parameter $\lambda >0$,
and the form of the confidence bonuses $\epsilon_{t,a}$.
We take as belief-estimation subroutine~$\cB$ the spectral method described in Box~C of \Cref{sec:spectral} to estimate HMM parameters, together with the Bayes' updates rules of \Cref{eq:Bayes-est-1,eq:Bayes-est-2,eq:Bayes-est-3}.

\paragraph{Stage length $\ell \geq 1$.}
The theory suggests a reference stage length of order $T^{3/4}$. Since $T=50,000$ is relatively small from the viewpoint of the asymptotic analysis, we rather report results for $\ell \in \{1,\,15,\, 37,\, 224\}$, which are integer-rounded values of $T^{0}, T^{1/4}, T^{1/3}$, and $T^{1/2}$.

\paragraph{Regularization parameter $\lambda > 0$.}
We tune $\lambda$ over the logarithmic grid $\bigl\{10^k : k \in \{{-9},{-3},{-2},{-1},0,1,2\}\bigr\}$. The range of this base--$10$ grid is chosen to cover several orders of magnitude and to ensure that the best-performing value of $\lambda$ is not attained at one endpoint of the grid for the strategies implemented. Importantly, the marginal effect of further decreasing $\lambda$ below $10^{-3}$ becomes negligible. Thus, we include $\lambda = 10^{-9}$ as a representative value at this order of magnitude to approximate the case of a very weak regularization while avoiding the numerical instability suffered when inverting the Gram matrix for even smaller values of $\lambda$.

\paragraph{Confidence bonuses $\epsilon_{t,a}$: two forms.}
We considered two forms of confidence bonuses in this article: one in \Cref{eq:epsilonta-mostcomplex} to deal with rewards
stemming from the most complex reward model~\eqref{eq:reward-our}, where rewards depend directly on the latent states,
and one in Appendix~\ref{appendix:simplified-analysis}
suited to the simpler reward model~\eqref{eq:nelson-model-orig}.
We will respectively refer to these two forms as the ``complex form'' and the ``simplified form''.

\paragraph{Confidence bonuses $\epsilon_{t,a}$: ``complex form''.}
We recall the expression stated in \Cref{theorem:total-regret-bound}:
\begin{multline*}
\eps_{t,a} = \Ubelief(t, \delta/2)
+ f_t \biggl\Arrowvert G_{(s_t -1) \ell}^{-1} \Bigl(\hbb_t \otimes \bphi(a, \bx_t)\Bigr) \! \biggr\Arrowvert_2\,,
\qquad \mbox{where} \\ f_t = \lambda \sqrt{H} \, C_{\btheta^{\star}}
		+ 4 \sqrt{ \frac{s_T(s_t -1) (1+s_t \gamma) \ell}{\delta(1 - \gamma)} } \ +
		\sqrt{\frac{4 s_T}{\delta} C_{\eta} (s_t-1) \ell } + \frac{2 (s_t-1) \gamma}{1 - \gamma}
		+ \!\!\! \sum_{\tau=1}^{(s_t-1)\ell} \!\! \Ubelief(\tau,\delta/2)\,.
\end{multline*}
We actually omit the $\sqrt{s_T/\delta}$ terms coming from the union bounds mentioned
at the beginning of Appendix~\ref{appendix:proof-total-regret}, so that the dominant contribution in
the formula above for
$\eps_{t, a}$ is
\[
	4 \sqrt{ \frac{(s_t -1) (1+s_t \gamma) \ell}{\delta(1 - \gamma)} } \, \biggl\Arrowvert G_{(s_t -1) \ell}^{-1} \Bigl(\hbb_t \otimes \bphi(a, \bx_t)\Bigr) \! \biggr\Arrowvert_2\,.
\]
In the simulation, we therefore use confidence bonuses $\epsilon^{\cplx}_{t,a}$ of the ``complex form''
\begin{equation}
	\label{eq:simplified-boxa}
\epsilon^{\cplx}_{t,a} = C \, s_t \sqrt{\ell} \, \biggl\Arrowvert G_{(s_t -1) \ell}^{-1} \Bigl(\hbb_t \otimes \bphi(a, \bx_t)\Bigr) \biggr\Arrowvert_2\,,
\end{equation}
where the multiplicative exploration constant $C$ controls the exploration level and is tuned over the logarithmic grid
$\bigl\{5 \times 10^k : k \in \{{-6},{-5},{-4},{-3},{-2},{-1}\}\bigr\}$.
This grid is chosen to cover a wide range of exploration strengths and to ensure that the best-performing value of $C$ lies in the interior of the grid for all the strategies implemented.

\paragraph{Confidence bonuses $\epsilon_{t,a}$: ``simplified form''.}
These confidence bonuses are derived, for the value $\ell=1$ (and only for this value),
from~\Cref{lm:est-simple} together with some crude boundings.
Their dominant term is of the original form
given by the left-hand side below (looking at the proof),
even though we rather stated an upper bound thereof in \Cref{lm:est-simple}
(based on $G_{t-1} \succeq \lambda \id{dH}$),
given by the right-hand side below:
\[
	\biggl\Arrowvert G_{t-1}^{-1} \Bigl(\hbb_t \otimes \bphi(a, \bx_t)\Bigr) \Bigr\Arrowvert_2 \,\, \sum_{\tau=1}^{t-1} \bigl\Arrowvert \bb_\tau - \hbb_\tau \bigr\Arrowvert_1
	\leq \Bigl\Arrowvert \hbb_t \otimes \bphi(a, \bx_t) \Bigr\Arrowvert_{G_{t-1}^{-1}} \,
\frac{1}{\sqrt{\lambda}}\sum_{\tau=1}^{t-1} \Arrowvert \bb_{\tau} - \hbb_{\tau} \Arrowvert_1 \,.
\]
The right-hand side provides a simpler and more readable expression to derive the
confidence bonuses $\epsilon_{t,a}$ in \Cref{theorem:total-regret-bound-simplified}
but for a fairer comparison with the ``complex form'' of confidence bonuses~\eqref{eq:simplified-boxa},
and due to their similar expressions, we prefer resorting to the tighter left-hand side above.
Replacing the cumulative belief-estimation error by its order $\sqrt{t}$,
we thus consider, in the case $\ell = 1$ only, confidence bonuses proportional to
\[
\sqrt{t}\,\biggl\Arrowvert G_{t-1}^{-1} \Bigl(\hbb_t \otimes \bphi(a, \bx_t)\Bigr) \Bigr\Arrowvert_2\,.
\]
The analysis in Appendix~\ref{appendix:simplified-analysis} was only performed for the case $\ell=1$ of
no stages, but we extend it in the simulations to staged updates, by considering
\[
	\epsilon^{\simpl}_{t,a} = C \, \sqrt{s_t \ell} \,\biggl\Arrowvert G_{(s_t-1)\ell}^{-1} \Bigl(\hbb_t \otimes \bphi(a, \bx_t)\Bigr) \biggr\Arrowvert_2
\]
as the confidence bonuses of the ``simplified form''. The multiplicative exploration constant $C$ is tuned over the same grid as above.

\subsubsection{Plain LinUCB Strategy by \citet{abbasi2011improved}}
As a baseline, we consider the LinUCB strategy by \citet{abbasi2011improved} in its standard form, see Box~D. This baseline ignores the latent-state
dynamics altogether and therefore does not exploit either the HMM structure or the belief estimates. In particular, the rewards in the simulation are still generated by the true latent-state-dependent model, but the plain version of LinUCB treats them as if they arose from a standard linear contextual bandit model based only on the observed context and action. Thus, the comparison of
the LinUCB strategies exploiting estimated beliefs to this plain version of LinUCB
indicates whether latent-state-aware models bring a practical benefit.

We recall in Box~D the plain LinUCB strategy of \citet{abbasi2011improved},
as slightly adapted by \citet[Appendix~E]{CBwK-LP-2022}
to take care of the existence of a transfer function $\varphi$ taking into account the action;
in particular, the confidence bonuses used therein are of the form
\[
C \ln(t) \bigl\Arrowvert\bphi(a, \bx_t) \bigr\Arrowvert_{V_{t-1}^{-1}}\,,
\]
where the matrices $V_{t-1}$ are defined in Box~D.
Plain LinUCB relies on per-round updates. We use the same grids of exploration constants $C$ and regularization parameters $\lambda$ as for
the staged LinUCB on estimated beliefs; see Appendix~\ref{sec:grids}.
	
\begin{figure}[t]
	\begin{nbox}[title={\small Box D: Plain LinUCB strategy (Abbasi-Yadkori et al., 2011)}]
		\label{nbox:linUCB}
		\textbf{Known parameters:} finite action set $\cA$; context set $\cX$; transfer function $\bphi: \cA \times \cX \to \R^d$;
		
		\textbf{Unknown parameters:}
		HMM parameters, given by a transition matrix $\bM = (M_{h,h'})_{(h,h') \in [H]}$
		and emission distributions $(\nu_h)_{h \in [H]}$ over $\cX$;
		reward parameters $\btheta^{\star}_{h} \in \R^d$, for $h \in [H]$ \smallskip
		
		\textbf{Inputs:}
		regularization parameter $\lambda > 0$; closed-form
		expression for the confidence bonuses $\eps_{t,a}$ \smallskip
		
		\textbf{Initialization:} the learner sets $\hbtheta_0 = (1/\lambda) \, \bone \in \R^{d}$ \medskip

		\textbf{For} rounds $t \geq 1$, \textbf{the learner:} \smallskip
		\begin{enumerate}
			\item Observes the context $\bx_{t}$, drawn independently by the environment from $\nu_{h_{t}}$; \smallskip
			\item Computes the estimated mean rewards \quad $\displaystyle{\hr_{t}(a) = \bphi(a, \bx_t)^{\transp} \hbtheta_{t-1}}$ \quad for all $a \in \cA$; \smallskip
			\item Picks an action $\displaystyle{a_{t} \in \argmax_{a \in \cA} \bigl\{ \hat{r}_{t}(a) + \eps_{t, a}} \bigr\}$\,; \smallskip
			\item Obtains and observes the reward \quad $\displaystyle{r'_t(a_t) = \sum_{h \in [H]} \bb_{t}(h) \, \bphi(a_t, \bx_t)^{\transp} \btheta^\star_{h} + \eta_t(a_t)}$\,;
			\item Computes \quad $\displaystyle{\hbtheta_t  = V_t^{-1}
			\sum_{\tau=1}^{t} \bphi(a_\tau, \bx_\tau) r_\tau(a_\tau)}$ \quad where \quad $\displaystyle{V_t = \sum_{\tau=1}^{t} \bphi(a_\tau, \bx_\tau)
			\bphi(a_\tau, \bx_\tau)^{\transp} + \lambda \id{d}}$\,.
		\end{enumerate}
		\textbf{end}
	\end{nbox}
\end{figure}

\subsection{Performance Reported: Empirical Averages of Pseudo-Regrets}

\paragraph{Disclaimer.}
We ran $N = 100$ independent simulation, using random seeds $1951, \, \dots, \, 2050$. Because $N = 100$ is relatively small, the results below should be interpreted as illustrative only. This choice also reflects our computational budget: the simulations were run on a modest CPU-only setup, and the goal here is to visualize the practical behavior of the algorithms rather than to provide an extensive empirical benchmark.

\paragraph{Additional indexations by runs.}
For run $i \in [N]$, let $h^{(i)}_t$, $\bx^{(i)}_t$, $\bb^{(i)}_t$, and $a^{(i)}_t$ denote, respectively,
the realized latent state, the realized context, the true belief,
and the action selected by the policy at round $t$. In particular,
\[
\forall h \in [H], \qquad
	\bb^{(i)}_t(h) = \P\bigl(h^{(i)}_t = h \,\big|\, \bx^{(i)}_{1:t}\bigr)\,.
\]
In the simulation, these true beliefs are computed via the Bayes' update rule of \Cref{eq:Bayes-est-1,eq:Bayes-est-2,eq:Bayes-est-3},
performed with the true HMM parameters and the realized contexts $\bx^{(i)}_{1:t}$. They are used only for the evaluation
of the strategies, not by the strategies themselves.

\paragraph{Empirical average of pseudo-regrets.}
We report pseudo-regrets rather than cumulative rewards for three reasons.
First, as shown next in Appendix~\ref{sec:outcome-simulation}, the performance of the variants of the staged LinUCB strategy on estimated
beliefs is often close to each other, while pseudo-regret makes their difference easier to visualize. Second, since the results are averaged over only $N = 100$ runs, realized rewards would include additional Gaussian noise. Third, the consideration of cumulative pseudo-regrets directly indicates whether an algorithm exhibits sublinear or approximately linear regret over time.

That being said, we thus report pseudo-regrets.
The pseudo-regret of run $i$ up to round $t$ is defined, with the notation above
and given the definition of \Cref{eq:def-regret}, by
\[
	R^{(i)}_T = \sum_{\tau=1}^{t} \max_{a \in \cA} \sum_{h \in [H]} \bb^{(i)}_\tau(h) \, \bphi\bigl(a, \bx^{(i)}_t\bigr)^{\transp} \btheta^{\star}_{h} -
	\sum_{\tau=1}^{t} \sum_{h \in [H]} \bb^{(i)}_t(h) \, \bphi\bigl(a^{(i)}_t, \bx^{(i)}_t\bigr)^{\transp} \btheta^{\star}_{h} \, .
\]
In the simulations, we report the empirical averages
\[
	\bar{R}_t = \frac{1}{N} \sum_{i=1}^{N} R^{(i)}_t
\]
over time, together with bands equal to $\pm 2$ times the standard errors
of the series $\bigl( R^{(i)}_t \bigr)_{i \in [N]}$.

\subsection{Outcomes of Simulations}
\label{sec:outcome-simulation}

\paragraph{Overview of the performance by strategies.}
\Cref{fig:overall-performance} compares the pseudo-regret of the baseline strategy, \emph{Plain LinUCB}
to the one of the strategies introduced in this article:
the staged LinUCB strategy on estimated beliefs with complex-form confidence bonuses, abbreviated as \emph{LinUCB-Belief-Complex}
on the pictures and tables, and of
the staged LinUCB strategy on estimated beliefs with simplified-form confidence bonuses, abbreviated as \emph{LinUCB-Belief-Simplified}.
We do so for $\ell \in \{1, 15, 37, 224\}$, using for
each value of $\ell$ the best exploration constants $C$ and regularization parameters $\lambda$ selected in hindsight from the grids considered.

The first observation is that the baseline \emph{Plain LinUCB}, which does not leverage the latent-state structure, exhibits approximately linear pseudo-regret, even with the best $C$ and $\lambda$ in hindsight. By contrast, both \emph{LinUCB-Belief-Complex} and \emph{LinUCB-Belief-Simplified} achieve clearly sublinear pseudo-regret for all values of $\ell$ considered: this highlights the importance of exploiting the latent-state dynamics in the algorithm design.

The second set of observations is that \emph{LinUCB-Belief-Complex} performs generally better than \emph{LinUCB-Belief-Simplified}. Moreover, for a fixed strategy, the value of $\ell$ is not too influential. That being said, for \emph{LinUCB-Belief-Complex}, the staged variants generally perform slightly (but not significantly) better than the per-round version $\ell=1$.

\begin{figure}[!ht]
	\centering
	\includegraphics[width=\textwidth]{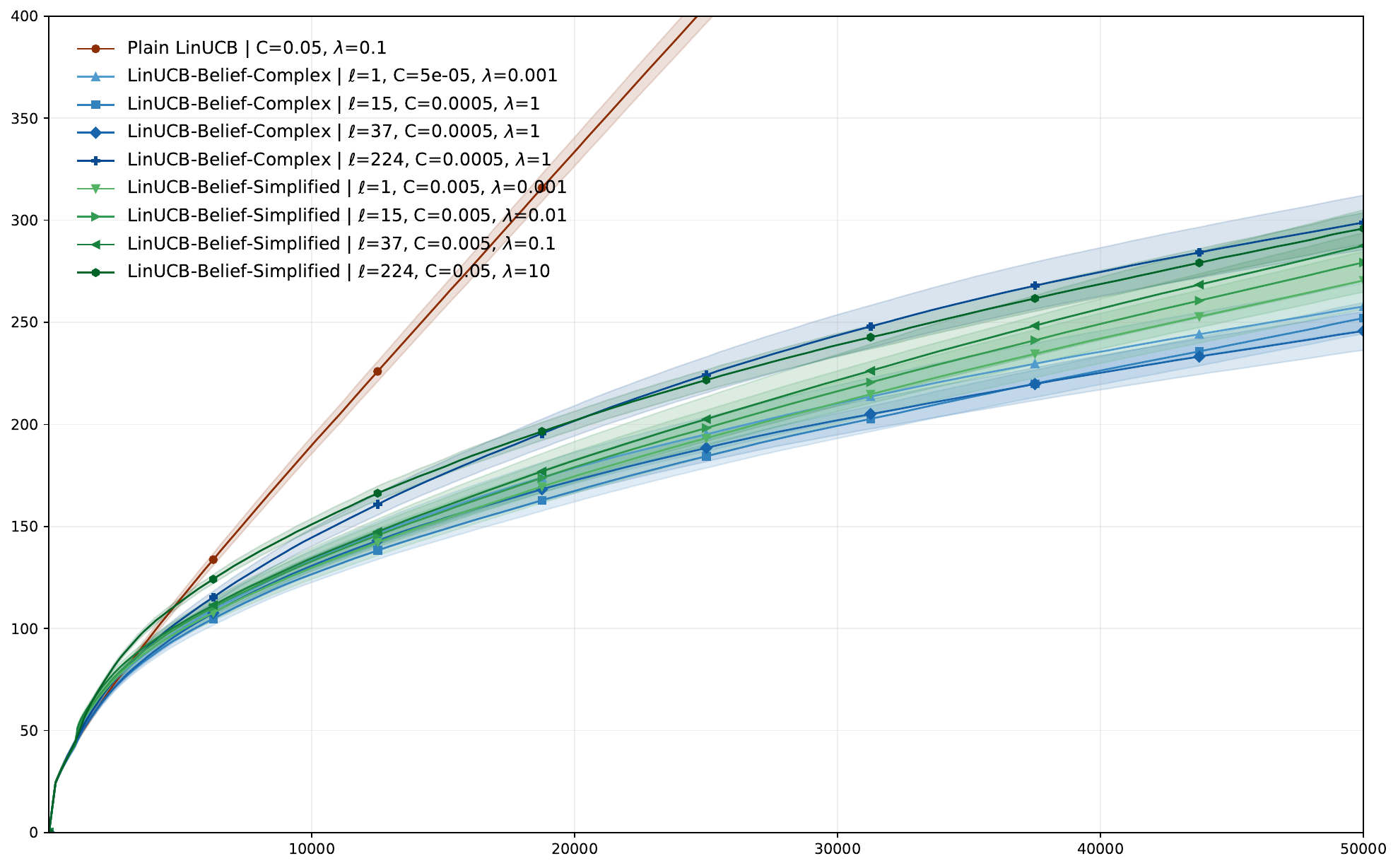}
	\caption{\label{fig:overall-performance} Pseudo-regrets averaged over $100$ runs. Solid lines correspond to averages and shaded areas to $\pm 2$ standard errors. The figure compares \emph{Plain LinUCB}, \emph{LinUCB-Belief-Complex}, and \emph{LinUCB-Belief-Simplified}, for $\ell \in \{1, 15, 37, 224\}$, using the best exploration constants $C$ and regularization parameters $\lambda$ selected in hindsight from the grids considered.}
\end{figure}

\paragraph{Sensitivity to exploration constants.}
\Cref{fig:plain-linucb-by-c,fig:linucb-by-c} compare the pseudo-regrets of \emph{Plain LinUCB}, \emph{LinUCB-Belief-Complex}, and \emph{LinUCB-Belief-Simplified} over grids of $\ell$ and $C$, using, for each pair $(\ell, C)$, the best regularization parameter $\lambda$ selected in hindsight from its grid.
We see that, no matter the stage lengths $\ell$, too large values of $C$ lead to much larger pseudo-regrets, that have a nearly-linear behavior for the initial values of $T$ (and must later exhibit a sublinear behavior). The parameter $C$ is thus critical to tune.

\begin{figure}[t]
	\centering
	\includegraphics[width=0.6\linewidth]{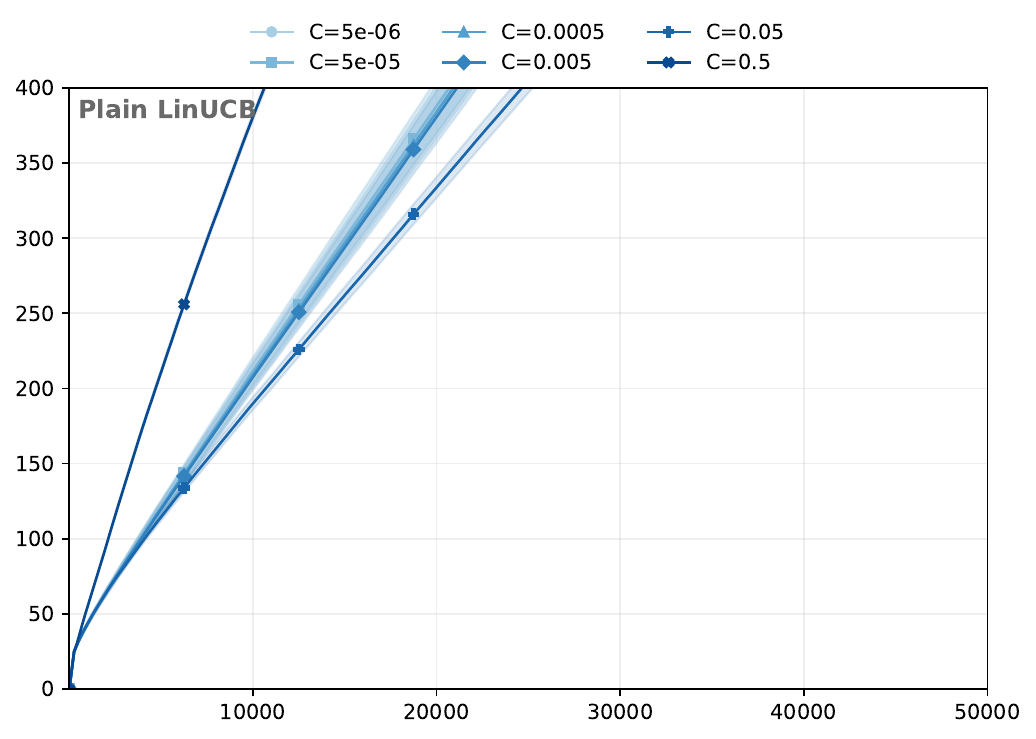}
	\caption{\label{fig:plain-linucb-by-c}
    Pseudo-regrets averaged over $100$ runs. Solid lines correspond to averages and shaded areas to $\pm 2$ standard errors. The figure compares \emph{Plain LinUCB} for different values of $C$, using, for each $C$, the best regularization parameter $\lambda$ selected in hindsight from its grid.}
\end{figure}

\begin{figure}[p]
	\centering
	
	\includegraphics[width=0.85\linewidth]{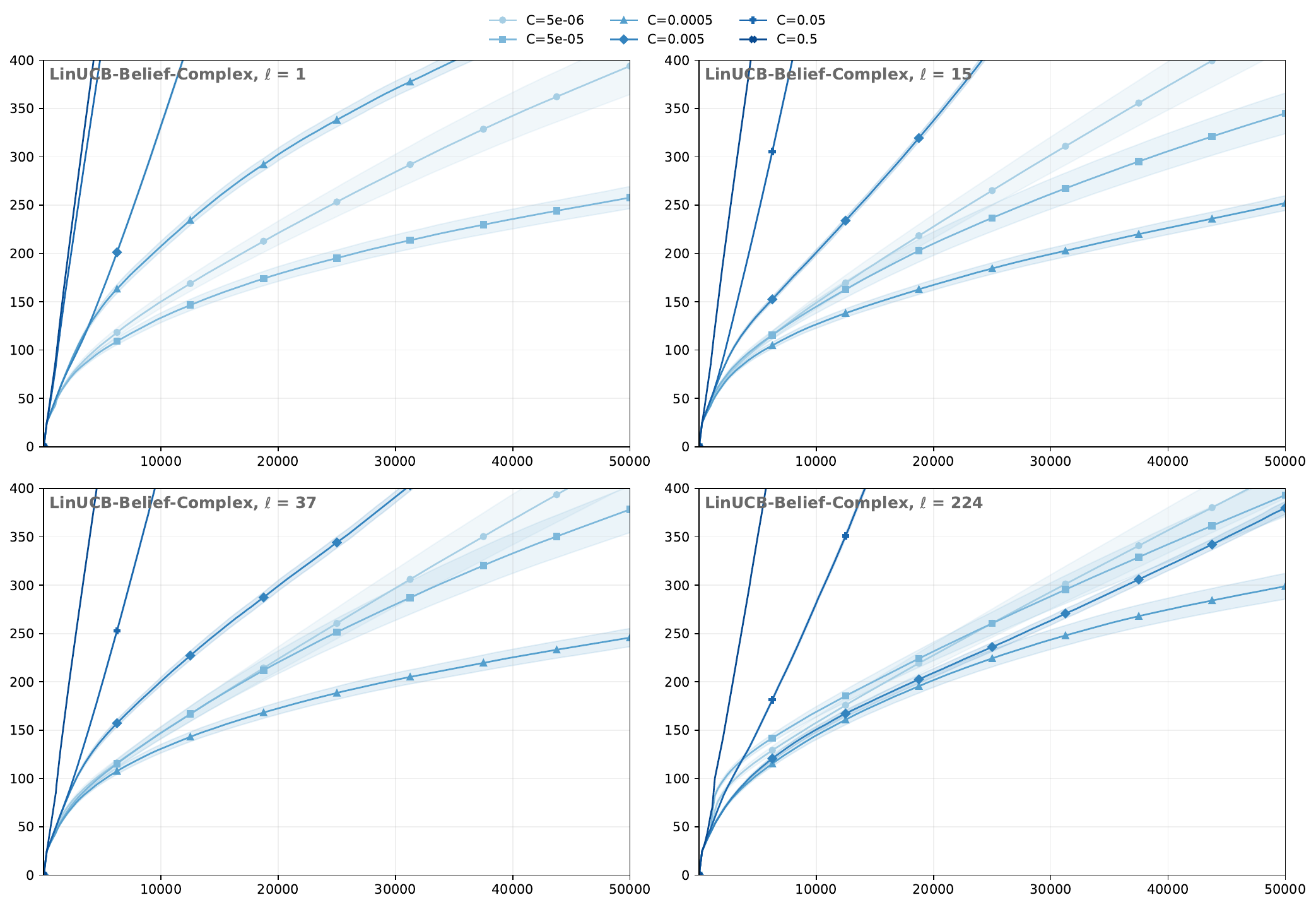}
	\caption*{(a) LinUCB-Belief-Complex}
	
	\includegraphics[width=0.85\linewidth]{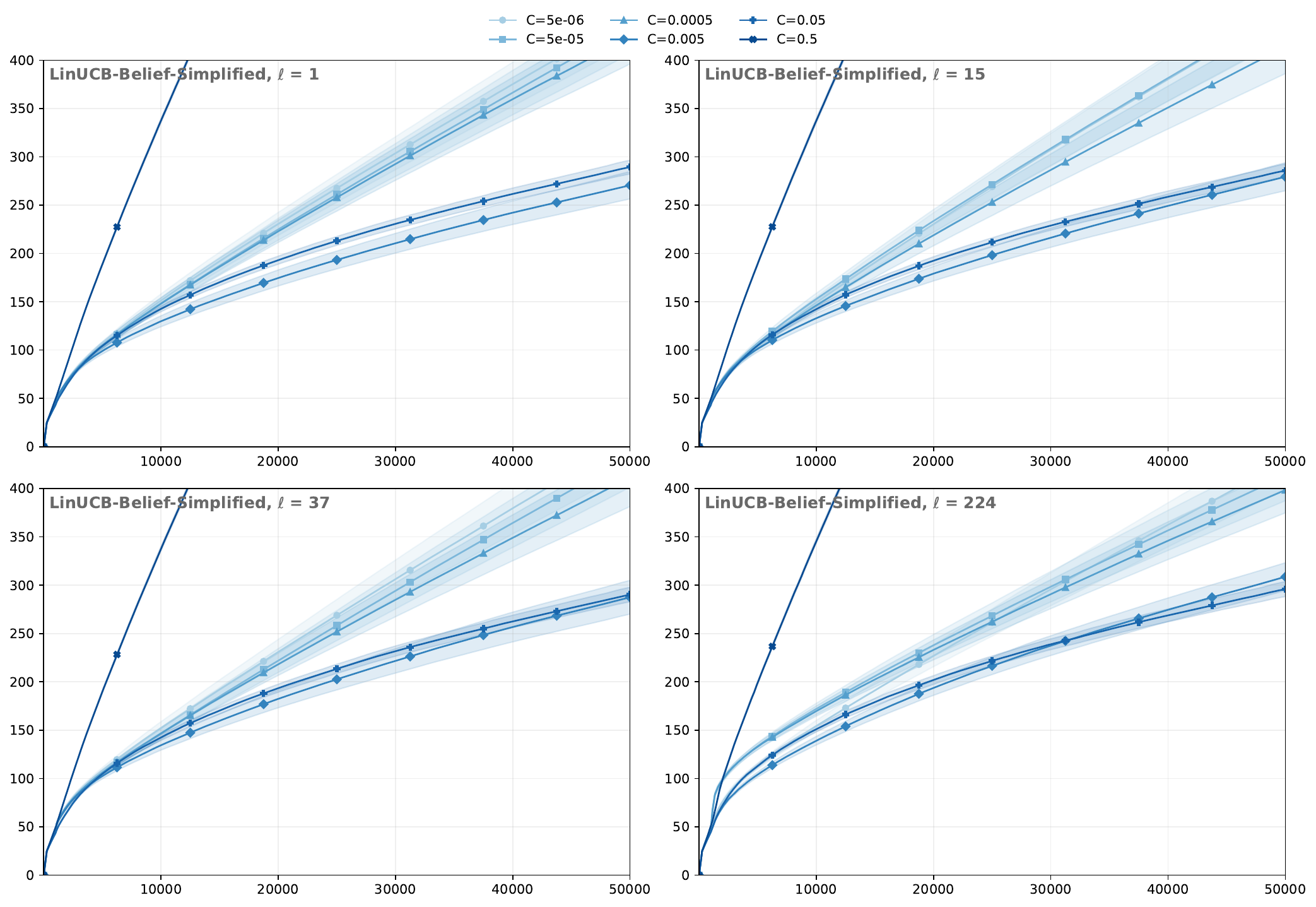}
	\caption*{(b) LinUCB-Belief-Simplified}
	
	\caption{Pseudo-regrets averaged over $100$ runs. Solid lines correspond to averages and shaded areas to $\pm 2$ standard errors. The figures compare
    LinUCB-Belief-Complex (top graphs) and LinUCB-Belief-Simplified (bottom graphs)
    for different values of $\ell$ and $C$, using, for each pair $(\ell,C)$, the best regularization parameter $\lambda$ selected in hindsight from its grid.}
	\label{fig:linucb-by-c}
\end{figure}

\clearpage
\begin{table}[t]
	\caption{\label{tab:plain-linucb-table-regret}
    Pseudo-regrets of \emph{Plain LinUCB} averaged over $100$ runs for each pair $(\lambda, C)$. The table reports averages (with $\pm 2$ standard errors in parentheses).}
	\centering
	\small
	\begin{tabular}{lcccccc}
		\toprule
		\multicolumn{7}{c}{\emph{Plain LinUCB}} \\
		\toprule
		$\lambda$ / $C$ & 5e-06 & 5e-05 & 0.0005 & 0.005 & 0.05 & 0.5 \\
		\midrule
		1e-09 & 917 (45) & 915 (44) & 911 (43) & 911 (44) & 760 (15) & 1458 (6) \\
		0.001 & 914 (44) & 916 (45) & 909 (43) & 910 (44) & 758 (14) & 1455 (6) \\
		0.01 & 921 (44) & 917 (42) & 915 (43) & 916 (44) & 765 (16) & 1456 (6) \\
		0.1 & 931 (44) & 927 (44) & 918 (43) & 899 (43) & \textbf{754 (16)} & 1454 (6) \\
		1 & 930 (41) & 933 (40) & 930 (42) & 909 (39) & 757 (15) & 1453 (6) \\
		10 & 915 (33) & 916 (33) & 920 (32) & 916 (31) & 788 (17) & 1445 (6) \\
		100 & 1224 (52) & 1221 (53) & 1211 (53) & 1179 (51) & 946 (13) & 1403 (6) \\
		\bottomrule
		\addlinespace[0.8em]
	\end{tabular}
\end{table}

\begin{table}[t]
	\caption{\label{tab:linucb-complex-table-regret}
    Pseudo-regrets of \emph{LinUCB-Belief-Complex} averaged over $100$ runs for each triplet $(\ell,\,\lambda, \, C)$. The table reports averages (with
    $\pm 2$~standard errors in parentheses).}
    \centering
	\begin{tabular}{@{}c l cccccc@{}}
		\toprule
		\multicolumn{7}{c}{\emph{LinUCB-Belief-Complex}} \\
		\toprule
		 $\ell$ & $\lambda$ / $C$ & 5e-6 & 5e-5 & 5e-4 & 5e-3 & 5e-2 & 5e-1 \\
		\midrule
		\makebox[1.2em][c]{%
			\smash{\raisebox{-3.0\normalbaselineskip}{\rotatebox{90}{$\ell=1$}}}%
		}
		& 1e-09 & 394 (29) & 259 (13) & 521 (7) & 2180 (8) & 4000 (6) & 4614 (6) \\
		& 0.001 & 416 (28) & \textbf{258 (11)} & 521 (7) & 2181 (8) & 3999 (6) & 4614 (5) \\
		& 0.01 & 432 (34) & 259 (11) & 520 (8) & 2182 (8) & 4001 (6) & 4615 (6) \\
		& 0.1 & 417 (30) & 262 (13) & 519 (7) & 2179 (8) & 4000 (6) & 4615 (6) \\
		& 1 & 490 (31) & 262 (12) & 514 (7) & 2174 (8) & 3999 (6) & 4616 (5) \\
		& 10 & 678 (26) & 457 (18) & 494 (6) & 2147 (8) & 4004 (6) & 4635 (5) \\
		& 100 & 658 (13) & 650 (11) & 488 (9) & 1965 (7) & 4001 (5) & 4709 (6) \\
		\addlinespace[0.8em]
		\midrule
		\makebox[1.2em][c]{%
			\smash{\raisebox{-3.0\normalbaselineskip}{\rotatebox{90}{$\ell=15$}}}%
		}
		& 1e-09 & 443 (29) & 359 (21) & 261 (9) & 1038 (8) & 3041 (7) & 4377 (5) \\
		& 0.001 & 448 (27) & 354 (22) & 266 (8) & 1037 (8) & 3039 (7) & 4377 (5) \\
		& 0.01 & 443 (31) & 345 (21) & 264 (9) & 1035 (8) & 3039 (7) & 4377 (5) \\
		& 0.1 & 473 (33) & 348 (21) & 256 (8) & 1035 (8) & 3039 (7) & 4377 (5) \\
		& 1 & 564 (31) & 387 (21) & \textbf{252 (8)} & 1028 (8) & 3036 (7) & 4378 (5) \\
		& 10 & 678 (24) & 644 (22) & 319 (8) & 997 (8) & 3020 (7) & 4391 (5) \\
		& 100 & 657 (13) & 655 (12) & 598 (13) & 842 (7) & 2906 (6) & 4436 (5) \\
		\addlinespace[0.8em]
		\midrule
		\makebox[1.2em][c]{%
			\smash{\raisebox{-3.0\normalbaselineskip}{\rotatebox{90}{$\ell=37$}}}%
		}
		& 1e-09 & 448 (31) & 383 (26) & 255 (11) & 754 (8) & 2639 (8) & 4226 (6) \\
		& 0.001 & 437 (33) & 378 (25) & 253 (11) & 755 (7) & 2639 (7) & 4226 (6) \\
		& 0.01 & 471 (33) & 382 (29) & 255 (11) & 753 (8) & 2639 (7) & 4227 (6) \\
		& 0.1 & 450 (28) & 378 (24) & 250 (9) & 753 (7) & 2639 (7) & 4226 (6) \\
		& 1 & 556 (32) & 433 (26) & \textbf{246 (9)} & 747 (7) & 2634 (7) & 4227 (5) \\
		& 10 & 670 (28) & 653 (23) & 367 (11) & 719 (7) & 2613 (7) & 4237 (5) \\
		& 100 & 659 (12) & 657 (12) & 637 (11) & 610 (8) & 2459 (7) & 4263 (5) \\
		\addlinespace[0.8em]
		\midrule
		\makebox[1.2em][c]{%
			\smash{\raisebox{-3.0\normalbaselineskip}{\rotatebox{90}{$\ell=224$}}}%
		}
		& 1e-09 & 426 (28) & 393 (22) & 323 (15) & 410 (8) & 1842 (9) & 3794 (6) \\
		& 0.001 & 420 (30) & 425 (24) & 338 (14) & 417 (8) & 1840 (9) & 3796 (6) \\
		& 0.01 & 424 (32) & 410 (29) & 337 (14) & 417 (8) & 1840 (8) & 3795 (6) \\
		& 0.1 & 451 (30) & 401 (28) & 305 (16) & 413 (8) & 1839 (9) & 3795 (6) \\
		& 1 & 525 (29) & 474 (27) & \textbf{299 (13)} & 395 (8) & 1832 (8) & 3794 (6) \\
		& 10 & 651 (20) & 655 (21) & 512 (17) & 380 (7) & 1793 (8) & 3789 (6) \\
		& 100 & 660 (12) & 659 (12) & 654 (11) & 490 (11) & 1586 (8) & 3745 (5) \\
		\bottomrule
		\addlinespace[0.8em]
	\end{tabular}
\end{table}

\paragraph{Detailed results for triplets $(\ell,\,C,\,\lambda)$.}
\Cref{tab:plain-linucb-table-regret,tab:linucb-complex-table-regret,tab:linucb-simplified-table-regret} report the detailed pseudo-regret results for each strategy over the full grids of $\ell,\,C,\,\lambda$.
What we wanted to check is that the grids of $C$ and $\lambda$ were sufficiently large in the sense that, for each strategy and each value of stage length~$\ell$, the best-performing pair $(\lambda, C)$ is not achieved at a boundary of the grids.

\begin{table}[t]
	\caption{\label{tab:linucb-simplified-table-regret}
    Pseudo-regrets of \emph{LinUCB-Belief-Simplified} averaged over $100$ runs for each triplet $(\ell,\,\lambda, \, C)$. The table reports averages (with
    $\pm 2$~standard errors in parentheses).}
	\centering
	\begin{tabular}{@{}c l cccccc@{}}
		\toprule
		\multicolumn{7}{c}{\emph{LinUCB-Belief-Simplified}} \\
		\toprule
		$\ell$ & $\lambda$ / $C$ & 5e-6 & 5e-5 & 5e-4 & 5e-3 & 5e-2 & 5e-1 \\
		\midrule
		\makebox[1.2em][c]{%
			\smash{\raisebox{-3.0\normalbaselineskip}{\rotatebox{90}{$\ell=1$}}}%
		}
		& 1e-09 & 451 (30) & 462 (32) & 430 (30) & 279 (14) & 318 (7) & 1508 (9) \\
		& 0.001 & 446 (29) & 455 (29) & 424 (28) & \textbf{271 (14)} & 318 (8) & 1505 (9) \\
		& 0.01 & 455 (30) & 436 (29) & 435 (31) & 284 (13) & 317 (8) & 1504 (9) \\
		& 0.1 & 513 (34) & 509 (37) & 441 (30) & 276 (14) & 316 (8) & 1504 (8) \\
		& 1 & 571 (35) & 568 (32) & 521 (29) & 288 (15) & 311 (8) & 1499 (9) \\
		& 10 & 688 (26) & 688 (25) & 680 (27) & 536 (20) & 289 (7) & 1464 (9) \\
		& 100 & 659 (13) & 659 (13) & 658 (13) & 651 (11) & 430 (12) & 1249 (8) \\
		\addlinespace[0.8em]
		\midrule
		\makebox[1.2em][c]{%
			\smash{\raisebox{-3.0\normalbaselineskip}{\rotatebox{90}{$\ell=15$}}}%
		}
		& 1e-09 & 453 (31) & 459 (30) & 414 (28) & 280 (13) & 320 (8) & 1506 (9) \\
		& 0.001 & 452 (28) & 453 (30) & 418 (30) & 285 (15) & 321 (8) & 1506 (8) \\
		& 0.01 & 483 (29) & 475 (31) & 432 (30) & \textbf{279 (15)} & 319 (8) & 1508 (9) \\
		& 0.1 & 535 (38) & 495 (38) & 443 (33) & 287 (17) & 319 (8) & 1506 (9) \\
		& 1 & 571 (32) & 575 (33) & 514 (28) & 287 (16) & 311 (8) & 1499 (9) \\
		& 10 & 679 (25) & 682 (24) & 673 (24) & 524 (22) & 286 (7) & 1464 (8) \\
		& 100 & 657 (12) & 657 (12) & 657 (12) & 651 (10) & 426 (12) & 1247 (8) \\
		\addlinespace[0.8em]
		\midrule
		\makebox[1.2em][c]{%
			\smash{\raisebox{-3.0\normalbaselineskip}{\rotatebox{90}{$\ell=37$}}}%
		}
		& 1e-09 & 452 (34) & 440 (31) & 419 (31) & 291 (17) & 322 (8) & 1509 (9) \\
		& 0.001 & 459 (34) & 433 (32) & 411 (30) & 290 (16) & 325 (8) & 1510 (9) \\
		& 0.01 & 476 (30) & 451 (33) & 419 (32) & 291 (16) & 323 (8) & 1508 (9) \\
		& 0.1 & 506 (30) & 481 (31) & 424 (29) & \textbf{288 (17)} & 320 (8) & 1508 (8) \\
		& 1 & 567 (35) & 560 (33) & 510 (34) & 295 (19) & 311 (7) & 1503 (9) \\
		& 10 & 670 (31) & 669 (25) & 662 (23) & 532 (19) & 290 (7) & 1465 (8) \\
		& 100 & 659 (12) & 659 (12) & 657 (12) & 648 (10) & 428 (13) & 1247 (8) \\
		\addlinespace[0.8em]
		\midrule
		\makebox[1.2em][c]{%
			\smash{\raisebox{-3.0\normalbaselineskip}{\rotatebox{90}{$\ell=224$}}}%
		}
		& 1e-09 & 437 (27) & 413 (26) & 398 (24) & 320 (15) & 343 (9) & 1528 (9) \\
		& 0.001 & 428 (29) & 434 (27) & 417 (22) & 340 (18) & 347 (9) & 1528 (9) \\
		& 0.01 & 455 (29) & 432 (31) & 430 (24) & 339 (17) & 347 (9) & 1529 (9) \\
		& 0.1 & 496 (33) & 456 (35) & 403 (27) & 330 (16) & 346 (9) & 1528 (9) \\
		& 1 & 529 (28) & 524 (30) & 457 (25) & 309 (14) & 337 (9) & 1521 (9) \\
		& 10 & 652 (21) & 649 (21) & 651 (21) & 525 (20) & \textbf{296 (8)} & 1482 (9) \\
		& 100 & 660 (12) & 659 (12) & 660 (12) & 652 (11) & 437 (12) & 1252 (8) \\
		\bottomrule
		\addlinespace[0.8em]
	\end{tabular}
\end{table}

\paragraph{Computational costs; link to the code.}
\emph{LinUCB-Belief-Complex} and \emph{LinUCB-Belief-Simplified} have comparable computational costs in our implementation. For each triplet $(\ell, \lambda, C)$, one series of $N=100$ runs takes approximately $15$ minutes. The full implementation is available at \url{https://github.com/zhenli1989/bandits_latent_states}.

\end{document}